%% file: main.tex
\documentclass{article}

\usepackage[margin=1.2in]{geometry} 
\usepackage{subfiles}
\usepackage{authblk}
\usepackage{footmisc}
\usepackage{natbib}
\usepackage{parskip}

\usepackage[utf8]{inputenc} 
\usepackage[T1]{fontenc}    
\usepackage{hyperref}       
\usepackage{url}            
\usepackage{booktabs}       
\usepackage{makecell}
\usepackage{amsfonts}       
\usepackage{nicefrac}       
\usepackage{microtype}      
\usepackage{xcolor}         
\usepackage{wrapfig}
\usepackage[titletoc, title]{appendix}
\usepackage{titletoc}

\input{math_commands.tex}

\input{includes}

\usepackage{subfiles}
\usepackage{authblk}
\usepackage{footmisc}

\title{Particle-based Generalised Stochastic Optimisation}

\author{Jiechen Jackie Zhang}
\author{O.~Deniz Akyildiz}
\affil{Department of Mathematics, Imperial College London}

\affil[]{{\textcolor{blue}{\footnotesize \texttt{\{jackie.zhang21,  deniz.akyildiz\}@imperial.ac.uk}}}}

\begin{document}

\maketitle

\begin{abstract}
We develop a class of general diffusion-based stochastic particle optimisation methods for loss functions with intractable gradients.  Specifically, we consider problems in which the loss gradient is an integral with respect to a parameter-dependent distribution, a structure that includes training generative models, fine-tuning, and learning latent-variable models.  We introduce mean-field dynamics and its interacting-particle approximations, which contain several existing algorithms as special cases and provide a route to constructing new methods. Under well-posedness and joint contractivity assumptions, we prove exponential convergence and show that the continuous-time particle system admits a non-asymptotic error bound. We illustrate our framework by developing higher-order Langevin and momentum optimisers and evaluating them on maximum marginal-likelihood estimation and energy-based-model training.
\end{abstract}


\section{Introduction}
\subfile{sections/main/introduction}

\section{Background}\label{sec:background}
\subfile{sections/main/background}

\section{Particle-based Stochastic Optimisation}\label{sec:main_system}
\subfile{sections/main/main_system}

\section{Analysis}\label{sec:analysis}
\subfile{sections/main/analysis}

\section{New Algorithms and Numerics}\label{sec:algorithm}
\subfile{sections/main/algorithms}

\section{Conclusion}\label{sec:conclusion}
\subfile{sections/main/conclusion}

\section*{Acknowledgements}
JJZ acknowledges support from EPRSC studentship grant EP/W524323/1. We acknowledge computational resources and support provided by the Department of
Mathematics and the Imperial College Research Computing
Service, DOI: 10.14469/hpc/2232.


\bibliographystyle{apalike}

\bibliography{references}

\newpage
\appendix

\makeatletter
\newcommand{\appendixpropformat}{%
\setcounter{figure}{0}
\setcounter{table}{0}
\setcounter{theorem}{0}
\setcounter{lemma}{0}
\setcounter{proposition}{0}
\setcounter{corollary}{0}
\setcounter{remark}{0}
\setcounter{assumption}{0}
\renewcommand{\thefigure}{A.\arabic{figure}}
\renewcommand{\thetheorem}{A.\arabic{theorem}}
\renewcommand{\thelemma}{A.\arabic{lemma}}
\renewcommand{\theproposition}{A.\arabic{proposition}}
\renewcommand{\thecorollary}{A.\arabic{corollary}}
\renewcommand{\theremark}{A.\arabic{remark}}
\renewcommand{\thetable}{A.\arabic{table}}
\renewcommand{\theassumption}{U\arabic{assumption}}
}
\begin{center}
    {\LARGE \textbf{Appendix}}
\end{center}
\appendixpropformat

We present proofs for the theorems stated in the main text as well as supplementary material pertaining to experiments and implementation details. The appendix is organised as follows:

\startcontents
\printcontents{}{1}{}

\makeatother
\newpage

\section{Proofs of main results}\label{app:proofs}
\subfile{sections/appendix/theory_proofs}

\section{Discussions on verifiable assumptions}\label{app:assumptions}
\subfile{sections/appendix/assumptions}

\section{Experimental Setting}

\subsection{\textsc{\small MMLE} with Higher-Order Langevin Dynamics}\label{app:mmle_details}
\subfile{sections/appendix/mmle_details.tex}

\subsection{Momentum Mean-Field Particle Optimiser}\label{app:mmfpo_details}
\subfile{sections/appendix/momentum_mfpo_appendix}

\end{document}

%% file: math_commands.tex
\usepackage{amsmath,amsfonts,bm}

\newcommand{\md}{\mathrm{d}}
\newcommand{\Ltwonorm}[2][]{\mathopen{#1\lVert}#2\mathclose{#1\rVert}_{\mathrm{L}^2}}

\def\1{\bm{1}}

\DeclareMathAlphabet{\mathsfit}{\encodingdefault}{\sfdefault}{m}{sl}
\SetMathAlphabet{\mathsfit}{bold}{\encodingdefault}{\sfdefault}{bx}{n}

\def\sB{{\mathsf{B}}}

\def\sL{{\mathsf{L}}}
\def\sM{{\mathsf{M}}}

\def\sV{{\mathsf{V}}}

\def\sX{{\mathsf{X}}}



%% file: includes.tex
\usepackage{xcolor}
\definecolor{light-teal}        {RGB}{88,214,199}
\definecolor{cb-blue}       {RGB}{70, 130, 180}
\definecolor{orange}        {RGB}{214,150, 92}
\definecolor{green}         {RGB}{136,196,136}
\usepackage{hyperref}
\usepackage{cleveref}
\hypersetup{colorlinks,linkcolor={cb-blue},citecolor={cb-blue},urlcolor={cb-blue}}
\usepackage{url}
\usepackage{booktabs}
\usepackage{multirow}
\usepackage{multicol}
\usepackage{subcaption}
\usepackage{algorithm}
\usepackage{algpseudocode}
\usepackage{graphicx}
\usepackage{amssymb}
\usepackage{amsmath}
\usepackage{amsthm}
\usepackage{mathtools}
\usepackage{minitoc}
\usepackage{paralist}
\usepackage{cleveref}
\usepackage{wrapfig}
\usepackage{cancel}

\usepackage[textwidth=1.8cm, textsize=scriptsize]{todonotes}

\usepackage[normalem]{ulem}

\newtheorem{theorem}{Theorem}

\newtheorem{proposition}{Proposition}
\newtheorem{corollary}{Corollary}
\newtheorem{lemma}{Lemma}

\newtheoremstyle{assumptionstyle}
  {3pt}
  {3pt}
  {\itshape}
  {}
  {\bfseries}
  {.}
  {0.5em}
  {#2}

\theoremstyle{assumptionstyle}
\newtheorem{assumption}{Assumption}

\renewcommand{\theassumption}{\textbf{A}\arabic{assumption}}

\crefname{assumption}{Assumption}{Assumptions}
\Crefname{assumption}{Assumption}{Assumptions}

\crefformat{assumption}{#2#1#3}
\Crefformat{assumption}{#2#1#3}

\theoremstyle{definition}
\newtheorem{remark}{Remark}
\makeatletter
\let\oldremark\remark
\let\endoldremark\endremark
\renewenvironment{remark}{\oldremark}{\hfill$\diamond$\endoldremark}
\makeatother

\usepackage[acronym]{glossaries-extra}
\setabbreviationstyle[acronym]{long-short}

\newacronym{ebm}{\textsc{\small EBM}}{energy-based model}
\newacronym{lebm}{\textsc{\small LEBM}}{latent energy-based model}
\newacronym{mmle}{\textsc{\small MMLE}}{maximum marginal likelihood estimation}
\newacronym{mmd}{\textsc{\small MMD}}{maximum mean discrepancy}
\newacronym{iplebm}{\textsc{\small EBIPLA}}{Energy-Based Interacting Particle Langevin Algorithm}
\newacronym{mle}{\textsc{\small MLE}}{Maximum Likelihood Estimation}
\newacronym{ula}{\textsc{\small ULA}}{Unadjusted Langevin Algorithm}
\newacronym{ipla}{\textsc{\small IPLA}}{Interacting Particle Langevin Algorithm}
\newacronym{pgd}{\textsc{\small PGD}}{Particle Gradient Descent}
\newacronym{mpgd}{\textsc{\small MPGD}}{Momentum Particle Gradient Descent}
\newacronym{em}{\textsc{\small EM}}{Expectation-Maximisation}
\newacronym{lvm}{\textsc{\small LVM}}{Latent Variable Model}
\newacronym{FID}{{{\textsc{\small FID}}}}{Fréchet Inception Distance}
\newacronym{sm}{\textsc{\small SM}}{Score Matching}
\newacronym{lsm}{\textsc{\small LSM}}{Latent Score Matching}
\newacronym{mfpo}{\textsc{\small MFPO}}{mean-field particle optimiser}
\newacronym{mmfpo}{\textsc{\small M-MFPO}}{momentum mean-field particle optimiser}
\newacronym{impdiff}{\textsc{\small ImpDiff}}{Implicit Diffusion}
\newacronym{soul}{\textsc{\small SOUL}}{Stochastic Optimization via Unadjusted Langevin}
\newacronym{souk}{\textsc{\small SOUK}}{\textsc{\small SOUK}}
\newacronym{cd}{\textsc{\small CD}-$k$}{Contrastive Divergence}
\newacronym{pcd}{\textsc{\small PCD}}{Persistent Contrastive Divergence}
\newacronym{smc}{\textsc{\small SMC}}{Sequential Monte Carlo}
\newacronym{sa}{\textsc{\small SA}}{stochastic approximation}
\newacronym{mypgd}{\textsc{\small MYPGD}}{Moreau-Yosida particle gradient descent}
\newacronym{myholpgd}{\textsc{\small HOL-MYPGD}}{Moreau-Yosida high-order Langevin particle gradient descent}
\newacronym{nag}{\textsc{\small NAG}}{Nesterov's Accelerated Gradient}
\newacronym{hmc}{\textsc{\small HMC}}{Hamiltonian Monte Carlo}
\newacronym{mcmc}{\textsc{\small MCMC}}{Markov chain Monte Carlo}
\newacronym{montecarlo}{\textsc{\small MC}}{Monte Carlo}
\newacronym{kl}{\textsc{\small KL}}{Kullback-Leibler}
\newacronym{kiplmc}{\textsc{\small KIPLMC}}{kinetic interacting particle Langevin Monte Carlo}
\newacronym{uld}{\textsc{\small ULD}}{underdamped Langevin diffusion}
\newacronym{ou_process}{\textsc{\small OU}}{Ornstein-Uhlenbeck}
\newacronym{sde}{\textsc{\small SDE}}{stochastic differential equation}
\newacronym{rmsprop}{\textsc{\small RMSPROP}}{Root Mean Square Propagation}
\newacronym{adam}{\textsc{\small ADAM}}{Adaptive Moment Estimation}
\newacronym{totalvariation}{\textsc{\small TV}}{total-variation}
\newacronym{my}{\textsc{\small MY}}{Moreau-Yosida}
\newacronym{mse}{\textsc{\small MSE}}{mean squared error}
\newacronym{ssim}{\textsc{\small SSIM}}{structural similarity index}

\renewcommand{\bibname}{References}
\renewcommand{\bibsection}{\subsubsection*{\bibname}}
\usepackage{subfiles}

%% file: sections/main/introduction.tex
Many problems in statistics and machine learning can be formulated as the optimisation of a loss $\ell:\Theta\rightarrow\mathbb{R}$ whose gradient takes the form of an expectation with respect to an intractable parameter-dependent distribution $\pi_\theta$, i.e.\ $\nabla_\theta \ell(\theta) = \mathbb{E}_{\pi_\theta}[F(\theta, X)]$. This structure arises in problems such as maximum-likelihood training of \glspl*{ebm} \citep{cd_original,nijkamp2019learning,song2021train,oliva2025uniformintimeconvergenceboundspersistent}, latent \glspl*{ebm} \citep{pang2020learning,marks2025learninglatentenergybasedmodels}, \gls*{mmle} \citep{atchade2017perturbed,de_bortoli_efficient_2021,pmlr-v206-kuntz23a,lim2024momentumparticlemaximumlikelihood,akyildiz2025interacting}, fine tuning \citep{marion2025implicitdiffusionefficientoptimization}, and Bayesian experimental design \citep{iollo2025bayesian}.

These problems are usually solved via generic first-order optimisation schemes where each iteration involves running \gls*{mcmc} to sample from $\pi_\theta$, resulting in a nested optimisation-sampling scheme \citep{younes1999convergence,andrieu2014markovian,atchade2017perturbed,de_bortoli_efficient_2021}. However, this is computationally expensive with delicate convergence properties dependent on various factors such as Markov-chain bias, mixing, and step-size selections. A number of algorithms have been developed in recent literature to avoid separately equilibrating an inner Markov chain, see, e.g. \citet{pmlr-v206-kuntz23a,lim2024momentumparticlemaximumlikelihood,akyildiz2025interacting} for \gls*{mmle} in latent variable models, and \citet{marion2025implicitdiffusionefficientoptimization} for fine tuning of generative models. These methods construct joint sampling-optimisation procedures, typically with Langevin processes and gradient descent based optimisers. However, they are often tied to particular dynamical choices of optimisers and diffusions, which may even be replaced by alternatives in implementation. In fact, practical applications of such methods often use modern optimisers such as Adam \citep{kingma2017adammethodstochasticoptimization}, see, e.g., \citet{pmlr-v206-kuntz23a,marks2025learninglatentenergybasedmodels} or preconditioned Langevin dynamics, see, e.g., \citet{glyn2025statistical,wang2026training,marks2025learninglatentenergybasedmodels}. This raises the question of how to define principled, general purpose stochastic optimisation methods for loss functions with intractable gradients.

In this work, we develop general diffusion-based particle optimisation methods. Specifically:

\textbf{(C1)} We develop a general diffusion-based framework for stochastic optimisation of loss functions, accounting for general optimisation and sampling dynamics. This unifies and generalises several existing methods for general optimisation problems whose gradients involve parameter-dependent expectations and leads to the design of a novel class of algorithms.

\textbf{(C2)} Under this general formulation, we establish exponential contraction of the mean-field system under a joint strong-monotonicity condition and a finite-$N$ error bound for a particle approximation.

\textbf{(C3)} Using the general framework, we develop two novel algorithms, (i) a higher-order Langevin dynamics based \gls*{mmle} algorithm, and (ii) a momentum-based algorithm for training \glspl*{ebm}. We show that these algorithms outperform existing methods in their respective applications. We also discuss how to develop novel algorithms for other applications, by instantiating the general framework with appropriate choices of optimisers and diffusions.

The rest of the paper is organised as follows. In Section~\ref{sec:background}, we provide background on the optimisation problems we consider, and the sampling dynamics we use, together with a discussion of related work. In Section~\ref{sec:main_system}, we introduce our general mean-field system for stochastic optimisation. We then provide theoretical results on the convergence of the mean-field system and finite-particle error bounds in Section~\ref{sec:analysis}. Finally, we show the design of novel algorithms in Section~\ref{sec:algorithm} accompanied with numerical results and experiments.

\subsection{Notation}
We denote by $\mathcal{P}_2(\mathbb{R}^d)$ the set of probability measures on $\mathbb{R}^d$ with finite second moments.
We write $\delta_x$ for the Dirac delta measure at $x$. For a function $f:\mathbb{R}^d\rightarrow\mathbb{R}$, we write $\nabla f$ for the gradient of $f$, and $\nabla^2 f$ for the Hessian of $f$. For a vector $v\in\mathbb{R}^d$, we write $\|v\|$ for the Euclidean norm of $v$. For a random vector $X$, $\Ltwonorm{X}:=(\mathbb{E}\|X\|^2)^{1/2}$. For a matrix $A\in\mathbb{R}^{d\times d}$, we write $\|A\|$ for the spectral norm of $A$. For two measures $\mu,\nu\in\mathcal{P}(\mathbb{R}^d)$, we write $W_2(\mu,\nu)$ for the 2-Wasserstein distance between $\mu$ and $\nu$:
\begin{align*}
    W_2(\mu,\nu) &= \left(\inf_{\gamma\in\Gamma(\mu,\nu)}\int \|x-y\|^2 \gamma(\md x,\md y)\right)^{1/2},
\end{align*}
where $\Gamma(\mu,\nu)$ is the set of couplings of $\mu$ and $\nu$. For a function $f:\mathbb{R}^d\rightarrow\mathbb{R}$, we write $\|f\|_\infty = \sup_{x\in\mathbb{R}^d}|f(x)|$ for the supremum norm of $f$. $L^p(\mu)$ denotes the space of functions where the integral $\int |f(x)|^p \mu(\md x)$ is finite.

We will denote various spaces belonging to optimiser, sampler, and auxiliary states using the letters $\sX, \sV, \Theta, \sM$. In all these cases, we assume that these spaces are Euclidean, i.e., $\sX = \mathbb{R}^{d_x}$, $\sV = \mathbb{R}^{d_v}$, $\Theta = \mathbb{R}^{d_\theta}$, $\sM = \mathbb{R}^{d_m}$. Therefore, their tangent spaces are also Euclidean spaces of the same dimensions. We will keep this setting for clarity and simplicity, and our results can be extended to more general spaces.

%% file: sections/main/background.tex
\subsection{Stochastic Optimisation with Intractable Gradients}
We are interested in the optimisation of an objective function $\ell:\Theta\rightarrow\mathbb{R}$ over a set of parameters $\theta\in\Theta$ where the gradient is of the form
\begin{align}\label{eq:obj_func}
\nabla_\theta \ell(\theta) = \int_\sX F(\theta, x)\pi_\theta(\md x),
\end{align}
where $F: \Theta\times\sX\rightarrow \Theta$ is a measurable function, $F(\theta, \cdot) \in L^1(\pi_\theta)$ for all $\theta\in\Theta$, and $\pi_\theta$ is a parameter-dependent distribution on a space $\sX$. We assume that the distribution $\pi_\theta$ admits a density with respect to the Lebesgue measure on $\sX$:
\begin{equation}
    \pi_\theta(x)= e^{-E(\theta,x)}/C_\theta, \quad C_\theta=\int_{\sX} e^{-E(\theta,x)}\md x,
\end{equation}
where $0 < C_\theta < \infty$ is the unknown normalisation constant and the potential $E:\Theta \times \sX \to \mathbb R$ is known pointwise. This setting encompasses a wide range of applications seen in the literature and we summarise some in the next section.

\subsection{Related Work}\label{sec:related_work}

We compare related methodologies which also solve the stochastic optimisation problem, relating to objectives of the form in \eqref{eq:obj_func}. Classical approaches use the \gls*{sa} framework of \citet{robbins_monro_SA} with Markovian noise, where for each outer optimisation step on the parameters, an inner loop \gls*{mcmc} sampler is used to approximate samples from $\pi_\theta$ and the resulting \gls*{montecarlo} average is used for a first-order update \citep{younes1999convergence,andrieu2014markovian,atchade2017perturbed}. This includes \gls*{soul} and its more general kernel counterparts (\glsunset{souk}\gls*{souk}), where a number of Langevin or more general Markov-kernel steps are used per parameter update \citep{de_bortoli_efficient_2021}. In \gls*{ebm} training, \gls*{cd} runs $k$ short \gls*{mcmc} transitions, typically initialised from data, to approximate the negative-phase expectation \citep{cd_original}; \gls*{pcd} instead keeps a persistent particle cloud and usually applies one or a few transitions per parameter step, reducing the nested-loop cost at the price of non-equilibrium bias \citep{tieleman2008training,oliva2025uniformintimeconvergenceboundspersistent}. This is seen in other more modern single-loop approaches. \gls*{impdiff} couples the parameter dynamics with one sampler update step per optimisation step, so that particles and parameters evolve jointly \citep{marion2025implicitdiffusionefficientoptimization}. When the inner approximation is based on \gls*{smc}, \citet{cuin2026efficientstochasticoptimisationsequential} maintain weighted particles.

The problem of \gls*{mmle} for a latent variable model $p_\theta(x,y)$ fits in the setting of \eqref{eq:obj_func} after applying Fisher's identity to obtain an expectation under the posterior $p_\theta(\cdot\mid y)$. The classical \gls*{em} algorithm alternates an E-step, which forms or approximates the posterior expectation, with an M-step in $\theta$ \citep{em_algorithm_dempster_laird_rubin}. When the E-step is intractable, this leads to Monte Carlo, stochastic approximation, online, variational, and biased-\gls*{mcmc} variants of \gls*{em} \citep{wei1990monte,delyon1999convergence,cappe2009online,gersende_em,pmlr-v238-gruffaz24a,JALA_NEURIPS2025_99d7d3fc}. Another approach 
introduces a particle approximation for the optimisation and avoids separately equilibrating an inner Markov chain. \gls*{pgd} performs a coupled continuous-time descent of the \gls*{em} free energy introduced by \cite{neal1998view} over parameters and a variational posterior approximation, using Langevin dynamics on the particles for current posterior approximations \citep{pmlr-v206-kuntz23a}. \gls*{mpgd} adds momentum to both the parameter and particle dynamics, yielding an underdamped extension of this coupled system \citep{lim2024momentumparticlemaximumlikelihood}. The \gls*{ipla} injects noise into the parameter dynamics and interprets the full parameter-particle system as an interacting Langevin diffusion \citep{akyildiz2025interacting}, with kinetic variants such as \gls*{kiplmc} extending this perspective to underdamped dynamics \citep{oliva2026kineticinteractingparticlelangevin} and adaptations to latent energy-based models \citep{marks2025learninglatentenergybasedmodels}. Whilst these works are mostly developed for \gls*{mmle} or closely related latent-variable problems, they illustrate the same continuous-time particle optimisation principle, replacing a costly inner \gls*{mcmc} loop with a persistent interacting particle system.

Our contribution abstracts and generalises this perspective. Rather than fixing a particular optimiser, sampler, or problem, we formulate a continuous-time mean-field optimisation-sampling system for general objectives with gradients of the form \eqref{eq:obj_func}. This includes existing particle methods such as \gls*{pgd} and \gls*{mpgd} as special cases, while also allowing more general optimiser dynamics and sampler diffusions.

%% file: sections/main/main_system.tex
\subsection{A general mean-field dynamics for stochastic optimisation}
We first describe a fully general framework for stochastic optimisation of loss functions with intractable gradients. Recall that we consider the problem of minimising a loss function $\ell:\Theta\rightarrow\mathbb{R}$ where the gradient takes the form as in \eqref{eq:obj_func}.

In order to define general mean-field dynamics, we will consider a general optimiser state $\vartheta_t \in \Theta \times \sM$. This notation allows the optimiser state to include auxiliary variables, such as momenta -- i.e. it may contain only $\theta_t$, or $(\theta_t,m_t)$, or Adam-style auxiliary variables. Let $\Psi_t:(\Theta \times \sM)\times\Theta \rightarrow \Theta \times \sM$ be a function that gives the update direction for the optimiser state $\vartheta_t$ when it is supplied with a gradient estimate.  Let $\Pi_\theta(\vartheta_t)=\theta_t$ extract the parameter component. Similarly to the optimiser state, we will consider a general sampler state $Z_t \in \sX \times \sV$ with dynamics that depend on the projected parameter state $\theta_t=\Pi_\theta(\vartheta_t)$. Let $\Pi_x(Z_t)=X_t$ extract the $x$-component of the sampler state. We will assume that the dynamics of $Z_t$ are such that, for frozen $\theta$, the marginal distribution of $X_t$ converges to $\pi_\theta$.

Let $\sB$ be a finite-dimensional noise space and let $B_t$ be a Brownian motion on $\sB$. We write $\sL(\sB,\sX\times\sV)$ for the space of
linear maps from $\sB$ to $\sX\times\sV$. We then consider the following coupled mean-field system
\begin{align}
\mathrm{d} \vartheta_t &= \Psi_t(\vartheta_t,G_{\nu_t}(\vartheta_t)) \mathrm{d} t, \label{eq:general_mf_system_opt}\\
\mathrm{d} Z_t &= b_{\Pi_\theta(\vartheta_t)}(Z_t) \mathrm{d}t  +\sigma_{\Pi_\theta(\vartheta_t)}(Z_t)\mathrm{d}B_t, \label{eq:general_mf_system_x}
\end{align}
where the drift $b_\theta:\sX\times\sV\to\sX\times\sV$ and diffusion coefficient
$\sigma_\theta:\sX\times\sV\to\sL(\sB,\sX\times\sV)$ are chosen such that the system \eqref{eq:general_mf_system_x} has invariant distribution $\nu_\theta$ with $\sX$-marginal $\pi_\theta$ and $\nu_t = \mathrm{Law}(Z_t)$ with the function $G_{\nu_t}(\vartheta_t)$ defined as
\begin{align*}
G_{\nu_t}(\vartheta_t)&:=\int_{\sX\times\sV} F(\Pi_\theta(\vartheta_t),\Pi_x(z))\nu_t(\mathrm{d} z).
\end{align*}
Next, we unpack the implications and possibilities provided by the generality of \eqref{eq:general_mf_system_opt}--\eqref{eq:general_mf_system_x}. A general $\Psi_t$ allows many optimisers whose optima are fixed points of the optimiser dynamics, including gradient descent and continuous-time analogues of \textsc{RMSProp}, \textsc{ADAM}, and related adaptive methods \citep{da2020general,barakat_convergence_2021, heredia2026modelingadagradrmspropadam}. Similarly, in \eqref{eq:general_mf_system_x}, one may choose a frozen sampler on the augmented space $\sX\times\sV$ whose invariant law $\nu_\theta$ has $\sX$-marginal $\pi_\theta$. From \citet{Ma_SGMCMC}, we describe one sufficient Euclidean construction. Recall $\sX=\mathbb{R}^{d_x}$, $\sV=\mathbb{R}^{d_v}$, write $z=(x,v)$, and consider functions $H:\Theta\times\sX\times\sV\to\mathbb{R}$ and $g:\sX\times\sV\to\mathbb{R}$ such that
\begin{equation*}
    H(\theta,z) = E(\theta,x) + g(x, v).
\end{equation*}
where $E$ is such that $\pi_\theta(\md x)=C_\theta^{-1}\exp(-E(\theta,x))\,\md x$. Assume also
$C_g:=\int_{\sV}\exp(-g(x,v))\,\md v < \infty$ and independent of $x$. Then the measure $\nu_\theta \propto \exp(-H(\theta,x,v))$ on $\sX\times\sV$ has $\sX$-marginal $\pi_\theta$. Next, consider the diffusion in \eqref{eq:general_mf_system_x} with drift and diffusion coefficients
\begin{align}\label{eq:generalised_langevin_drift_diffusion}
    &b_\theta(z)
    =
    -[\mathrm{D}_\theta(z)+\mathrm{Q}_\theta(z)]\nabla_z H(\theta,z)
    +\xi_\theta(z),\\
    &\sigma_\theta(z)\sigma_\theta(z)^\top
    =
    2\mathrm{D}_\theta(z).
\end{align}
where $\mathrm{D}_\theta(z)\in\mathbb{R}^{(d_x+d_v)\times(d_x+d_v)}$ is symmetric positive semidefinite, $\mathrm{Q}_\theta(z)\in\mathbb{R}^{(d_x+d_v)\times(d_x+d_v)}$ is skew-symmetric, and both are continuously differentiable in $z$. The correction term is defined coordinatewise by
\begin{align}
    \xi_{\theta,i}(z)
    :=
    \sum_{j=1}^{d_x+d_v}
    \frac{\partial}{\partial z_j}
    \bigl(\mathrm{D}_{\theta,ij}(z)+\mathrm{Q}_{\theta,ij}(z)\bigr),
\end{align}
where $i=1,\ldots,d_x+d_v$. All derivatives here are with respect to the sampler variable $z$, with $\theta$ fixed. Assume additionally that $e^{-H(\theta,\cdot)}$ is integrable, that the coefficients and $H$ are sufficiently regular for the stationary Fokker--Planck equation, and that the relevant boundary/no-flux or decay conditions justify integration by parts. Here, the generalised Langevin dynamics with drift and diffusion given in \eqref{eq:generalised_langevin_drift_diffusion} have stationary distribution $\nu_\theta \propto e^{-H(\theta,z)}$ provided that the following identity is satisfied \citep{Ma_SGMCMC,yao2025accelerating}
\begin{equation*}
    \sum_{i,j=1}^{d_x+d_v}\frac{\partial^2}{\partial z_i\partial z_j}\left(\mathrm{Q}_{\theta,ij}(z)e^{-H(\theta,z)}\right)= 0.
\end{equation*}
Under these conditions, the displayed identity holds automatically because \(\mathrm{Q}_\theta\) is sufficiently smooth and skew-symmetric and mixed derivatives commute, thus $\nu_\theta$ is invariant for the frozen sampler, and the preceding marginal condition gives $(\Pi_x)_\#\nu_\theta=\pi_\theta$. This construction can be generalised to arbitrary manifolds using the construction provided in \citet{barp2021unifying}.

Next, we provide a basic result that demonstrates the importance of our framework for a wide class of optimisers and samplers. We start with a well-posedness assumption for the mean-field system \eqref{eq:general_mf_system_opt}--\eqref{eq:general_mf_system_x}.
\begin{assumption}[Well-posedness]\label{ass:strong_sol_exist_unique}
For every $\vartheta_0\in\Theta\times\sM$ and every random variable $Z_0$, independent of the driving Brownian motion, with law $\nu_0\in\mathcal P_2(\sX\times\sV)$, the mean-field system \eqref{eq:general_mf_system_opt}--\eqref{eq:general_mf_system_x} admits a pathwise unique strong solution on every finite time interval. Moreover, $\nu_t:=\mathrm{Law}(Z_t)\in\mathcal P_2(\sX\times\sV)$ for all finite $t\ge0$. Furthermore, assume one of
\begin{enumerate}[(i)]
    \item For any $\vartheta\in\Theta\times\sM,z\in\sX\times\sV$, the diffusion coefficient $\sigma_{\Pi_{\theta}(\vartheta)}(z)$ is constant.
    \item For any $\vartheta\in\Theta\times\sM,z\in\sX\times\sV$, the following linear growth condition on the drift and diffusion 
\begin{equation*}
    \| b_{\Pi_\theta(\vartheta)}(z)\|+\|\sigma_{\Pi_\theta(\vartheta)}(z)\|\le K(1+\|\theta\|+\|z\|).
\end{equation*}
\end{enumerate}
\end{assumption}
Standard Lipschitz and linear-growth assumptions on the combined drift and diffusion coefficients provide sufficient conditions for
Assumption~\ref{ass:strong_sol_exist_unique}.
\begin{assumption}\label{ass:existence_minimizer} There exists an optimiser state \(\vartheta_\star\in\Theta\times\sM\), with \(\Pi_\theta(\vartheta_\star)=\theta_\star\), such that
\[
\nabla_\theta\ell(\theta_\star)=0
\qquad\text{and}\qquad
\Psi_t(\vartheta_\star,0)=0
\quad\text{for all }t\ge0.
\]
Moreover, the frozen sampler at \(\theta_\star\) has an invariant measure \(\nu_{\theta_\star}\in\mathcal P_2(\sX\times\sV)\) whose \(\sX\)-marginal is \(\pi_{\theta_\star}\), i.e.
\[
(\Pi_x)_\#\nu_{\theta_\star}=\pi_{\theta_\star}.
\]
\end{assumption}
This assumption simply states that the optimiser map $\Psi_t$ has a fixed point at $\vartheta_\star$ which is generally satisfied by well-defined optimisation methods. Next we show that the mean-field system \eqref{eq:general_mf_system_opt}--\eqref{eq:general_mf_system_x} has a stationary law at the optimiser--sampler pair $(\vartheta_\star,\nu_{\theta_\star})$.
\begin{theorem}\label{thm:invariant_measure} Let \Cref{ass:strong_sol_exist_unique} and \Cref{ass:existence_minimizer} hold. Let
\[
\mu_\star:=\delta_{\vartheta_\star}\otimes\nu_{\theta_\star}.
\]
Then \(\mu_\star\) is a stationary law for the coupled mean-field system \eqref{eq:general_mf_system_opt}--\eqref{eq:general_mf_system_x}: if the system is initialized from \(\mu_\star\), its law remains \(\mu_\star\) for all \(t\ge0\).
\end{theorem}
See Appendix~\ref{app:thm_invariant_measure} for the proof. If \(\Psi_t\equiv\Psi\) is time-independent, then the coupled mean-field dynamics are time-homogeneous. In this case \(\mu_\star\) is an invariant measure in the usual semigroup sense. This formulation covers various algorithms in the literature as special cases and brings great flexibility to the design of diffusion-based stochastic optimisers. We highlight some concrete examples from the literature below.

\subsubsection{Existing literature as special cases}
\glsreset{mmle}
\glsreset{pgd}

\textbf{PGD.} Consider the \gls*{mmle} problem given a model $p_\theta(x, y)$ with fixed data $y$. As a minimisation, the loss function is $\ell(\theta) = - \log p_\theta(y)$ and the gradient takes the form \eqref{eq:obj_func} with $F(\theta,x)=-\nabla_\theta \log p_\theta(x,y)$ and $\pi_\theta(x) = p_\theta(x \mid y)$. \citet{pmlr-v206-kuntz23a} proposed \gls*{pgd} to solve this problem, where the mean-field limit can be written as
\begin{align*}
\md \theta_t &=
\left[
\int \nabla_\theta \log p_{\theta_t}(x,y)\,
\nu_t(\md x)
\right] \md t,\\
    \md X_t
&=
\nabla_x\log p_\theta(X_t,y)\, \md t
+
\sqrt{2}\,\md B_t.
\end{align*}
This is a special case of \eqref{eq:general_mf_system_opt}--\eqref{eq:general_mf_system_x} with $\Psi_t(\vartheta_t,G_{\nu_t})=-G_{\nu_t}$, $\vartheta_t=\theta_t$, $Z_t=X_t$, $\mathrm{D}_\theta(x)=I_{d_x}$, $\mathrm{Q}_\theta(x)=0$, and $g(x,v)=0$, so that the drift and diffusion coefficients are $b_\theta(x)=\nabla_x\log p_\theta(x,y)$, and $\sigma_\theta(x)=\sqrt{2}I_{d_x}$. The sampler is an overdamped Langevin diffusion targeting $\pi_\theta(x)=p_\theta(x\mid y)$. From Theorem~\ref{thm:invariant_measure}, we see $\delta_{\theta_\star}\otimes\pi_{\theta_\star}$ is an invariant measure for any stationary point $\theta_\star$ of $\ell$; hence for any minimiser, since $\ell\in C^1$ on $\Theta$.

\textbf{Implicit diffusion.} This is also the mean-field limit of the implicit diffusion framework \citep[Section~4.1]{marion2025implicitdiffusionefficientoptimization} which reads as
\begin{align*}
\md \theta_t &= - \varepsilon_t \int F(\theta_t, x) \nu_t(\md x) \md t,\\
\md X_t &= - \nabla_x E(\theta_t, X_t) \md t + \sqrt{2} \md B_t,
\end{align*}
for the optimisation of loss functions where $\nabla \ell(\theta) = \mathbb{E}_{\pi_\theta}[F(\theta, x)]$. For any deterministic schedule \(t\mapsto\varepsilon_t\) satisfying the regularity needed for well-posedness, this is a special case of \eqref{eq:general_mf_system_opt}--\eqref{eq:general_mf_system_x} with \(\Psi_t(\theta,G_{\nu_t})=-\varepsilon_t G_{\nu_t}\), \(Z_t=X_t\), \(b_\theta=-\nabla_xE(\theta,\cdot)\), and \(\sigma_\theta=\sqrt2I_{d_x}\). The choice \(\varepsilon_t\equiv1\) gives the time-homogeneous optimiser used in the \gls*{pgd} example.

\textbf{MPGD.} Accelerated methods were also proposed for the problem of \gls*{mmle} which introduce auxiliary momentum variables to extend \gls*{pgd}, which we term \gls*{mpgd} \citep{lim2024momentumparticlemaximumlikelihood}. This is again a special case of our framework with the loss choices identical to the previous case and with the auxiliary variables $m\in \sM=\mathbb{R}^{d_\theta}$ and $v\in \sV=\mathbb{R}^{d_v}$ such that the optimiser state becomes $\vartheta_t=(\theta_t,m_t)$ and the sampler state is $Z_t=(X_t,V_t)$. When \(\Psi_t\) is the time-homogeneous heavy-ball ODE, take \(d_v=d_x\), $g(x,v):=\tfrac{1}{2M_x}\|v\|^2$,
\[
\begin{aligned}
&\mathrm D=
\begin{pmatrix}
\mathbf 0&\mathbf 0\\
\mathbf 0&\gamma_xI_{d_x}
\end{pmatrix},
\qquad
\mathrm Q=
\begin{pmatrix}
\mathbf 0&-I_{d_x}\\
I_{d_x}&\mathbf 0
\end{pmatrix}.
\end{aligned}
\]
Then
\[
\begin{aligned}
b_\theta(x,v)
&=
\left(
\tfrac{v}{M_x},
-\nabla_xE(\theta,x)-\tfrac{\gamma_x}{M_x}v
\right),\\
\sigma_\theta\sigma_\theta^\top
&=
\operatorname{diag}(0,2\gamma_xI_{d_x}),
\end{aligned}
\]
so the sampler is precisely the mass-\(M_x\) underdamped Langevin diffusion used by \gls*{mpgd} \citep{lim2024momentumparticlemaximumlikelihood}. Theorem~\ref{thm:invariant_measure} therefore shows that
\[
\delta_{\theta_\star}\otimes\delta_0\otimes\pi_{\theta_\star}
\otimes\mathcal N(0,M_xI_{d_x})
\]
is stationary whenever \(\nabla_\theta\ell(\theta_\star)=0\).

\subsection{Space and time discretisations}

To obtain an SDE that is amenable to numerical implementation, we approximate the mean-field law by an empirical measure of \(N\) particles driven by mutually independent Brownian motions. Under i.i.d.\ initialisation, the particles interact through their common optimiser state and are therefore exchangeable rather than independent. Their empirical gradient \(G_{\nu_t^N}(\vartheta_t^N)\) is used as an approximation of the corresponding mean-field gradient. This gives the following finite-dimensional interacting SDE
\begin{align}
\mathrm{d} \vartheta^N_t &= \Psi_t(\vartheta^N_t,G_{\nu_t^N}(\vartheta_t^N) ) \mathrm{d} t,\label{eq:general_particle_system_opt}\\
\mathrm{d} Z^i_t &= b_{\Pi_\theta(\vartheta^N_t)}(Z^i_t) \mathrm{d}t  +\sigma_{\Pi_\theta(\vartheta^N_t)}(Z^i_t)\mathrm{d}B^i_t,\label{eq:general_particle_system_x}
\end{align}
where $\{B^i_t\}_{i=1}^N$ are mutually independent Brownian motions on $\sB$ and the empirical measure $\nu_t^N$ is defined as $\nu_t^N =({1}/{N}) \sum_{i=1}^N \delta_{Z^i_t}$ and
\begin{align}
G_{\nu_t^N}(\vartheta_t^N)&:=\frac{1}{N}\sum_{i=1}^N F(\Pi_{\theta}(\vartheta^N_t),\Pi_{x}(Z^i_t)). \nonumber
\end{align}
Certain choices of $\Psi_t$ and the drift and diffusion coefficients in \eqref{eq:general_particle_system_x}, together with suitable time-discretisation (e.g. Euler-Maruyama) schemes, give rise to practical algorithms. We leave the design of such algorithms to Section~\ref{sec:algorithm} where we develop applications. We next consider theoretical analysis of the general mean-field system and its interacting-particle approximation.

%% file: sections/main/analysis.tex
We now study stability and finite-particle error for the general mean-field system
\eqref{eq:general_mf_system_opt}--\eqref{eq:general_mf_system_x}. In order to do this, given two solutions $(\vartheta_t,\nu_t),(\tilde{\vartheta}_t,\tilde{\nu}_t)$, we define the distance metric \citep{pgd_theory}
\begin{equation*}
    \mathbf{d}((\vartheta_t,\nu_t),(\tilde{\vartheta}_t,\tilde{\nu}_t)):=\sqrt{\|\vartheta_t-\tilde\vartheta_t\|^2
    +
    W_2^2(\nu_t,\tilde\nu_t)},
\end{equation*}
and introduce the main assumptions required for our convergence results.

\begin{assumption}[Twisted joint strong monotonicity]\label{ass:twisted_joint_strong_monotonicity}
Suppose that there exist symmetric positive definite matrices $D_\vartheta$ and $D_z$, and a constant $\rho>0$, such that, for every $t\ge0$, every $\vartheta,\tilde\vartheta\in\Theta\times\sM$, every $\nu,\tilde\nu\in\mathcal P_2(\sX\times\sV)$, and every coupling $\Lambda\in\Gamma(\nu,\tilde\nu)$, the following inequality holds:
\begin{align}
&\left\langle \vartheta-\tilde\vartheta,\,
    D_\vartheta
    \left[
    \Psi_t(\vartheta,G_\nu(\vartheta))
    -
    \Psi_t(\tilde\vartheta,G_{\tilde\nu}(\tilde\vartheta))
    \right]
    \right\rangle \notag\\
&\quad+
\int
\Big[
    \left\langle z-\tilde z,\,
    D_z
    \left[
    b_{\Pi_\theta(\vartheta)}(z)
    -
    b_{\Pi_\theta(\tilde\vartheta)}(\tilde z)
    \right]
    \right\rangle \notag\\
&\quad+
    \frac{1}{2}\left\|
    D_z^{1/2}
    \left[
    \sigma_{\Pi_\theta(\vartheta)}(z)
    -
    \sigma_{\Pi_\theta(\tilde\vartheta)}(\tilde z)
    \right]
    \right\|_{\mathrm{HS}}^2
\Big]\Lambda(\mathrm dz,\mathrm d\tilde z) \notag\\
&\le
-\rho\left(
    \|\vartheta-\tilde\vartheta\|_{D_\vartheta}^2
    +
    \int\|z-\tilde z\|_{D_z}^2\,
    \Lambda(\mathrm dz,\mathrm d\tilde z)
\right).
\label{eq:joint_strong_monotonicity_twisted}
\end{align}
\end{assumption}
\ref{ass:twisted_joint_strong_monotonicity} is the contractivity condition for the coupled optimiser--sampler system in the twisted metrics \(\|\cdot\|^2_{D_\vartheta}\) and \(\|\cdot\|^2_{D_z}\). When \(D_\vartheta=I_{d_\theta+d_m}\) and \(D_z=I_{d_x+d_v}\), it reduces to the usual joint dissipativity condition in the product Euclidean metric.

\begin{remark}
\ref{ass:twisted_joint_strong_monotonicity} is a genuinely joint
contractivity condition: dissipativity of the exact-gradient optimiser and
of each frozen sampler does not by itself suffice;
Appendix~\ref{app:cross_assumptions} provides an explicit counterexample.  Appendix~\ref{app:separated_assumptions} gives a componentwise
small-gain criterion that combines separate dissipativity and Lipschitz bounds
for the optimiser response, gradient observation, sampler drift, and diffusion
through a positive-definite cross-coupling matrix.
Appendix~\ref{app:cor_projected_strong_convexity} instead gives a structured
criterion based on strong monotonicity of a projected field
\(\Phi=(F,\nabla_xE)\), or joint strong convexity of an underlying potential
when this field is conservative, together with sufficient dissipation of the
auxiliary variables. Finally, Appendix~\ref{app:mpgd_result} gives a concrete
hypocoercive construction for a heavy-ball optimiser coupled to an
underdamped sampler.
\end{remark}

\begin{theorem}\label{thm:twisted_joint_strong_monotonicity}
Suppose \ref{ass:strong_sol_exist_unique} and \ref{ass:twisted_joint_strong_monotonicity} hold with positive definite matrices \(D_\vartheta,D_z\) and constant \(\rho>0\). Let \((\vartheta_t,\nu_t)\) and \((\tilde\vartheta_t,\tilde\nu_t)\) be any two solutions of \eqref{eq:general_mf_system_opt}--\eqref{eq:general_mf_system_x} with initial sampler laws in \(\mathcal P_2(\sX\times\sV)\). Define
\[
\kappa
:=
\frac{\lambda_{\max}(\mathrm{diag}(D_\vartheta,D_z))}
{\lambda_{\min}(\mathrm{diag}(D_\vartheta,D_z))}.
\]
Then, for every \(t\ge0\),
\begin{equation*}
\mathbf d((\vartheta_t,\nu_t),(\tilde\vartheta_t,\tilde\nu_t))
\le
\sqrt{\kappa}\,e^{-\rho t}
\mathbf d((\vartheta_0,\nu_0),(\tilde\vartheta_0,\tilde\nu_0)). 
\end{equation*}
If, in addition, \ref{ass:existence_minimizer} holds, define
\(\mu_t:=\delta_{\vartheta_t}\otimes\nu_t\) and
\(\mu_\star:=\delta_{\vartheta_\star}\otimes\nu_{\theta_\star}\).
Then, for every initial product law
\(\mu_0=\delta_{\vartheta_0}\otimes\nu_0\), with
\(\nu_0\in\mathcal P_2(\sX\times\sV)\),
\[
W_2(\mu_t,\mu_\star)
\le
\sqrt{\kappa}\,e^{-\rho t}W_2(\mu_0,\mu_\star).
\]
Consequently, \(\mu_\star\) is the unique finite-second-moment stationary law among laws of the form \(\delta_\vartheta\otimes\nu\). For time-dependent \(\Psi_t\), stationary means that initializing from the law leaves the law unchanged for every \(t\ge0\).
\end{theorem}
The proof can be found in Appendix~\ref{app:thm_joint_strong_monotonicity}. We now consider bounding the error between the particle approximation and the mean-field system. For this, we follow \citet{pgd_theory} and introduce an auxiliary $N$-copy mean field system. Let \(\bar\vartheta_0\) be deterministic, and let
\(\bar Z_0^1,\ldots,\bar Z_0^N\) be i.i.d.\ with common law
\(\bar\nu_0\in\mathcal P_2(\sX\times\sV)\), independently of the mutually independent Brownian motions $(B_t^1,\ldots,B_t^N)_{t\ge0}$. Defining the system
\begin{align}
    \mathrm d\bar\vartheta_t
    &=
    \Psi_t(\bar\vartheta_t,G_{\bar\nu_t}(\bar\vartheta_t))\,\mathrm dt,\qquad\qquad\qquad \label{eq:n-copy-1}\\
    \mathrm d\bar Z_t^i
    &=
    b_{\Pi_\theta(\bar\vartheta_t)}(\bar Z_t^i)\,\mathrm dt
    +
    \sigma_{\Pi_\theta(\bar\vartheta_t)}(\bar Z_t^i)\,\mathrm dB_t^i,
    \label{eq:n-copy-2}
\end{align}
then \(\bar\vartheta_t\) is deterministic and
\(\bar Z_t^1,\ldots,\bar Z_t^N\) are i.i.d.\ with common law \(\bar\nu_t\). We define the empirical measure $\bar\nu_t^{[N]}$ of the $N$-copy system as $\bar\nu_t^{[N]}=({1}/{N})\sum_{i=1}^N\delta_{\bar Z_t^i}$. We now state further assumptions that are required for our parameter error bound.
\begin{assumption}[Lipschitz gradient response]\label{ass:psi_F_lipschitz}
Assume that $\Psi_t$ is uniformly Lipschitz in its gradient argument: for all
$t\ge0$, all $\vartheta\in\Theta\times\sM$,
$g,\tilde g\in\mathbb R^{d_\theta}$,
\[
    \|\Psi_t(\vartheta,g)-\Psi_t(\vartheta,\tilde g)\|
    \le
    L_\Psi\|g-\tilde g\|.
\]
\end{assumption}

\begin{assumption}[Uniform gradient fluctuation]\label{ass:mf_bounded_second_moment}
Assume the mean-field copy system defined in \eqref{eq:n-copy-1}--\eqref{eq:n-copy-2} has uniformly bounded variance of the stochastic gradient:
\[
    \sup_{t\ge0}
    \mathbb E
    \left[
    \left\|
    F(\Pi_\theta(\bar\vartheta_t),\Pi_x(\bar Z_t))
    -
    G_{\bar\nu_t}(\bar\vartheta_t)
    \right\|^2
    \right]
    \le
    V_F
    <
    \infty.
\]
\end{assumption}
{\color{red}\ref{ass:mf_bounded_second_moment}}
is a uniform variance bound for the stochastic-gradient integrand.  It follows,
for example, if \(F(\theta,\cdot)\) is globally Lipschitz with a Lipschitz
constant uniform in \(\theta\) and
\(\sup_{t\geq0}\mathbb E\|\Pi_x(\bar Z_t)\|^2<\infty\).  More generally, a
linear-growth bound uniform along the optimiser path, together with the corresponding uniform second-moment bounds, is sufficient.

\begin{theorem}[Finite-particle interaction error]
\label{thm:finite_particle_error}
Suppose \ref{ass:strong_sol_exist_unique}, \ref{ass:twisted_joint_strong_monotonicity}--\ref{ass:mf_bounded_second_moment} hold.
Let $(\vartheta_t^N,Z_t^1,\ldots,Z_t^N)$ solve the particle system
\eqref{eq:general_particle_system_opt}--\eqref{eq:general_particle_system_x}
with the same Brownian motions and the synchronous initial conditions
\(\vartheta_0^N=\bar\vartheta_0\) and \(Z_0^i=\bar Z_0^i\).
Let \((\bar\vartheta_t,\bar Z_t^1,\ldots,\bar Z_t^N)\) solve the \(N\)-copy mean-field system \eqref{eq:n-copy-1}--\eqref{eq:n-copy-2}. Then, for every
$t\ge0$,
\[
    \Ltwonorm[\big]{\mathbf{d}\big((\vartheta^N_t,\nu_t^N),(\bar\vartheta_t,\bar\nu_t^{[N]})\big)}
    \le
    \tfrac{L_\Psi}{\rho}
    \sqrt{\tfrac{\kappa V_F}{N}}
    \left(1-e^{-\rho t}\right),
\]
where $\kappa$ is the condition number of
$\mathrm{diag}(D_\vartheta,D_z)$.
\end{theorem}
The proof is given in Appendix~\ref{app:proof_finite_particle_error} and proceeds as a generalisation of Lemma 13 in \cite{pgd_theory}. Using Theorems \ref{thm:twisted_joint_strong_monotonicity} and \ref{thm:finite_particle_error}, we may provide a full parameter error bound of the particle system \eqref{eq:general_particle_system_opt}--\eqref{eq:general_particle_system_x}.

\begin{theorem}[Parameter error bound of particle system]\label{thm:parameter_error_full}
    Suppose \ref{ass:strong_sol_exist_unique}--\ref{ass:mf_bounded_second_moment} hold. Assume \(\vartheta_0^N\) is deterministic and \(Z_0^1,\ldots,Z_0^N\) are i.i.d.\ with common law \(\nu_0\), independently of the mutually independent Brownian motions. Then the $\mathrm{L}^2$ parameter error satisfies
\begin{equation*}
\Ltwonorm{\theta^N_t-\theta_\star}\le \sqrt{\kappa}\,e^{-\rho t}
\mathbf d_0+\tfrac{L_\Psi}{\rho}
\sqrt{\tfrac{\kappa V_F}{N}}
\left(1-e^{-\rho t}\right),
\end{equation*}
with
\(\mathbf d_0:=\mathbf d((\vartheta^N_0,\nu_0),(\vartheta_\star,\nu_{\theta_\star}))\), the initial distance from the stationary mean-field solution. Consequently,
\begin{align*}
\lim_{N\to\infty}\limsup_{t\to\infty}
\Ltwonorm{\theta_t^N-\theta_\star}=0.
\end{align*}
\end{theorem}
The proof is given in Appendix~\ref{app:cor_parameter_error_full} and follows from Theorems~\ref{thm:twisted_joint_strong_monotonicity}--\ref{thm:finite_particle_error}. Theorem~\ref{thm:parameter_error_full} concerns the continuous-time particle
system. To describe at a high level how time discretisation would enter the
bound, let \(h>0\), let \(t_k:=kh\), and denote by
\(\widehat\theta_k^{N,h}\) the parameter iterate produced by a numerical
discretisation of the particle dynamics. Supposing the discretisation can
be coupled to the continuous-time system so that, for some error function
\(\varepsilon_{\mathrm{disc}}\) such that $\Ltwonorm{\widehat\theta_k^{N,h}-\theta_{t_k}^N}
\leq
\varepsilon_{\mathrm{disc}}(h,t_k,N),
$
where \(\varepsilon_{\mathrm{disc}}(h,t_k,N)\to0\) as \(h\to0\) -- for example,
on each finite horizon \(T\), one may have
\(\sup_{t_k\leq T}\varepsilon_{\mathrm{disc}}(h,t_k,N)
\leq C_T h^\alpha\) for some \(\alpha>0\).  The triangle inequality then gives the end-to-end estimate
\begin{align*}
\Ltwonorm{\widehat\theta_k^{N,h}-\theta_\star}
\leq
\Ltwonorm{\theta_{t_k}^N-\theta_\star}
+
\varepsilon_{\mathrm{disc}}(h,t_k,N),
\end{align*}
where the first term is bounded by Theorem~\ref{thm:parameter_error_full}.
Establishing such sampler- and optimiser-specific discretisation rates,
especially uniformly over an infinite time horizon, requires a separate
stability and numerical-analysis argument.  We leave this to future
work, since the primary aim of the paper is to demonstrate the
framework and establish its theoretical foundations.
\glsunset{mypgd}
\glsunset{myholpgd}
\begin{figure*}[t!]
\centering
\includegraphics[width=1.0\textwidth]{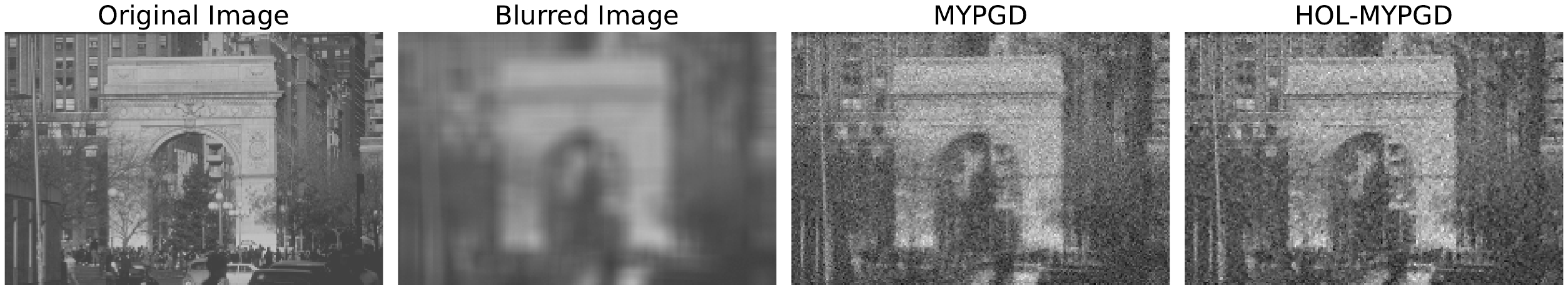}
\caption{Ground truth (left), blurred image (middle-left), and comparisons to single particle reconstructions produced by methods \gls*{mypgd} (middle-right) and \gls*{myholpgd} (right) for setup A.}
\label{fig:deblurring_example}
\end{figure*}

%% file: sections/main/algorithms.tex
In this section, we provide two examples to demonstrate the generality of our framework. We first tackle a \gls*{mmle} problem by developing a higher-order proximal Langevin method, building on the proximal interacting particle Langevin methods of \citet{encinar2025proximal}. Next, we provide a momentum optimiser generalising \gls*{mpgd} \citep{lim2024momentumparticlemaximumlikelihood} for general purpose optimisation and apply this method to \gls*{ebm} training.
\subsection{\textsc{\small MMLE} with Higher-Order Langevin Dynamics}\label{sec:mmle_image_deblurring}
\subfile{mmle_mfpo_image_debluring}

\subsection{A Momentum Particle Optimiser}\label{sec:momentum_mfpo}
\subfile{momentum_mfpo}

%% file: sections/main/mmle_mfpo_image_debluring.tex
\glsreset{mypgd}
\glsreset{myholpgd}
To demonstrate applicability in a higher-dimensional setting, we consider a Bayesian image deblurring problem, following \cite{encinar2025proximal}. Here we are given an information-destroying blurring operator, $B$, acting on an image and the goal is to recover a latent clean image, $x\in\mathbb{R}^{d_x}$, from a noisy blurred observation $y=Bx^\star+\epsilon$ and ground-truth image $x^\star\in\mathbb{R}^{d_x}$. Applying the anisotropic \gls*{totalvariation} penalty, $p_\theta(x)=C^{-1}_\theta \exp(-e^{\theta}\mathrm{TV}(x))$ we have the posterior,
\begin{equation*}
    \pi_\theta(x)\propto \exp\left(-\frac{1}{2\sigma^2}\|y-Bx\|_2^2-e^\theta\mathrm{TV}(x)\right),
\end{equation*}
with the non-differentiable total variation $\mathrm{TV}(x):=\|\nabla_d x\|_1$ and $\nabla_d$ the two-dimensional discrete gradient operator. We additionally assume $C_\theta<\infty$ such that the prior is well-defined. Framed as the \gls*{mmle} of $p_\theta(y)$, we may write the loss function as in \eqref{eq:obj_func} with,
\begin{equation*}
    \nabla_\theta\ell(\theta)=\mathbb{E}_{X\sim\pi_\theta}\left[e^\theta\mathrm{TV}(X)-d_x\right].
\end{equation*}
Following \cite{encinar2025proximal} we use proximal techniques and \gls*{my} approximations to deal with the non-differentiability of the total variation, which we discuss in Appendix~\ref{app:mmle_details}. We compare the baseline \gls*{mypgd} with a higher-order Langevin scheme following the sampler dynamics introduced in \cite{high_order_langevin_JMLR}. Using gradient descent on the parameters we have the particle SDE:
\begin{equation}\label{eq:high_order_langevin_scheme}
    \begin{aligned}
        \md \theta_t&=-\frac{1}{N}\sum_{i=1}^NF(\theta_t,X^i_t)\md t,\\
        \md X^i_t&=U^i_t\md t,\\
        \md U^i_t&=-\nabla_xE^\lambda(\theta_t,X^i_t)\md t+\gamma V_t^i\md t,\\
        \md V^i_t&=-\gamma U^i_t\md t-\alpha V^i_t\md t+\sqrt{2\alpha}\md B^i_t,
    \end{aligned}
\end{equation}
where we note that we have $F(\theta,x)=\nabla_\theta E(\theta,x)$. We replace the non-differentiable $x$-gradient with $E^\lambda$, its \gls*{my} approximation, to obtain a hybrid proximal scheme, which we refer to as \gls*{myholpgd}. Both methods share the same optimiser dynamics but differ in the choice of Langevin scheme for the sampler dynamics.
\begin{figure}[t]
\centering
\includegraphics[width=0.7\textwidth]{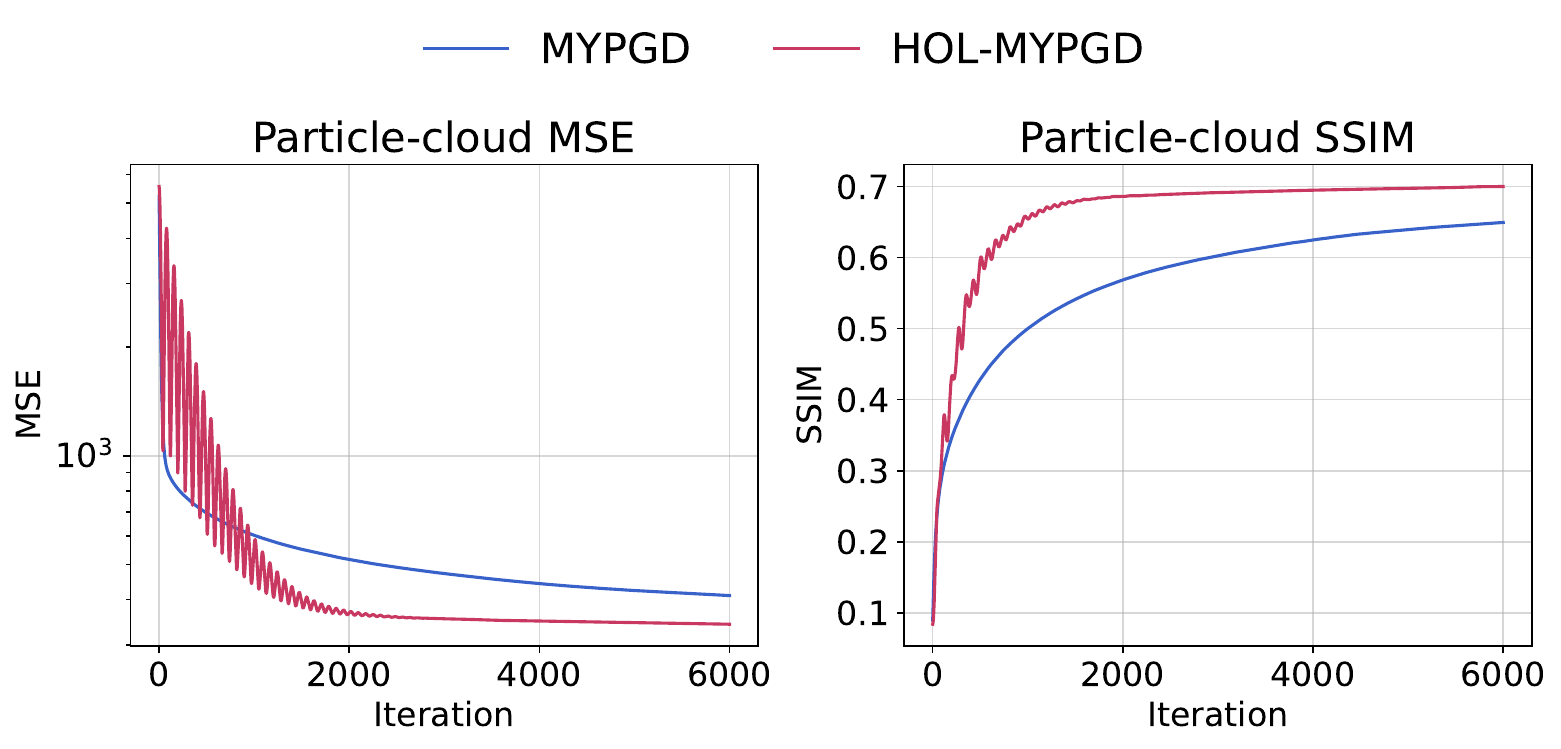} 
\caption{Evolution of MSE (left) and SSIM (right) of the particle cloud mean across iterations, comparing methods \gls*{mypgd} and \gls*{myholpgd} for setup A.}
\label{fig:deblurring_example_graphs}
\end{figure}
We consider setups A and B with different blur and noise levels, providing full experimental details in Appendix~\ref{app:mmle_details}. Comparing MSE and SSIM between methods for setup A in Figure~\ref{fig:deblurring_example_graphs}, we find the high-order Langevin scheme converges faster and to more optimal values; other examples for setup B are given in the appendix. We also provide image reconstructions from particles in Figure~\ref{fig:deblurring_example} where we see \gls*{myholpgd} yields sharper reconstructions.

\begin{figure*}[ht]
    \centering
    \includegraphics[width=1.0\textwidth]{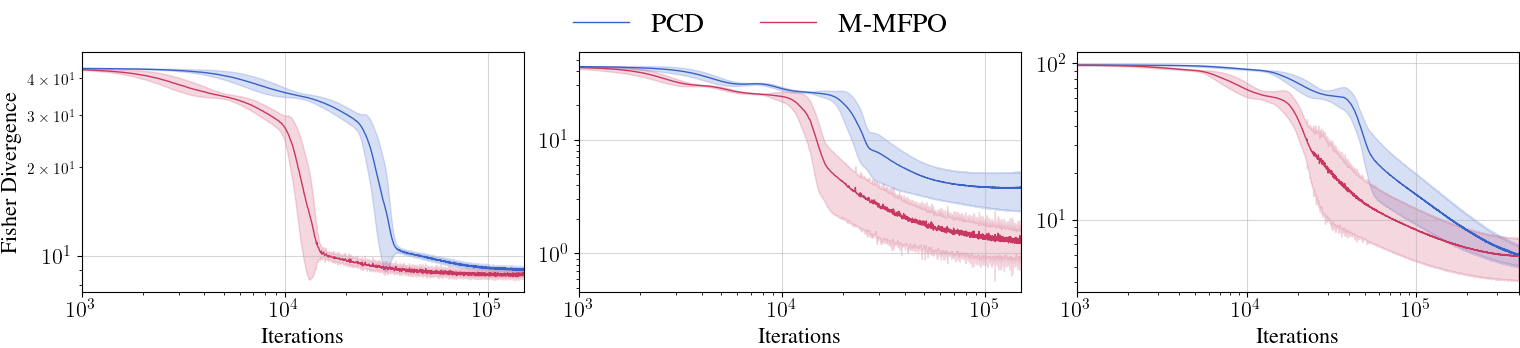}
    \caption{Comparison of Fisher divergence over iterations for the \textit{rings} (left), \textit{beads} (middle), and \textit{lattice} (right) datasets for \gls*{pcd} and \gls*{mmfpo} methods with $N=10,000$ particles. Curves show the mean across 25 independent runs, with shaded bands indicating ±1 standard deviations.}
    \label{fig:training_comparisons}
\end{figure*}

%% file: sections/main/momentum_mfpo.tex
We consider a particular choice of optimiser and sampler to model our system \eqref{eq:general_mf_system_opt}--\eqref{eq:general_mf_system_x}, which may be viewed as a generalisation of \gls*{mpgd}  \citep{lim2024momentumparticlemaximumlikelihood}. Here, both the optimiser and sampler dynamics involve momentum-acceleration, in the spirit of Polyak's Heavy Ball and \gls*{nag} \citep{Nesterov1983AMF}. With $m\in\sM=\mathbb{R}^{d_\theta}$, we model parameter dynamics with
\begin{equation} \label{eq:accelerated_optimiser}
    \begin{aligned}
        \md\theta_t &=m_t\md t,\\
        \md m_t &= -\int F (\theta_t, x)\pi_{\theta_t}(\md x)\md t-\gamma m_t\md t,
    \end{aligned}
\end{equation}
we see that $(\theta,m)=(\theta_\star,0)$ is a fixed point of \eqref{eq:accelerated_optimiser} with $\theta_\star$ such that $\nabla_\theta\ell(\theta_\star)=0$. Using also \gls*{uld} sampling dynamics with  friction coefficient $\gamma>0$, we have an invariant measure with $\pi_\theta$ marginal for frozen $\theta$. We present the \gls*{mmfpo} framework
\begin{equation}\label{eq:alg_underdamped}
    \begin{aligned}
        \md\theta_t &=m_t\md t,\\
        \md m_t &= -\frac{1}{N}\sum_{i=1}^N{F}\left(\theta_t, X^i_t\right)\md t-\gamma m_t\md t
        ,\\
        \md X^i_t&=V_t^i \md t,\\
        \md V^i_t&=-\nabla_x E(\theta_t, X^i_t)\md t-\gamma V_t^i\md t+\sqrt{2\gamma}\md B^i_t.
    \end{aligned}
\end{equation}
We propose a splitting scheme discretisation, inspired by the \textit{OBABO} scheme of \citet{obabo_ref_leimkuhler_matthews} and refer to Appendix~\ref{app:mmfpo_details} for more details.

\subfile{mmfpo_experiments}

%% file: sections/main/mmfpo_experiments.tex
\subsubsection{Training EBMs}\label{sec:ebm_training_2d}

Next we consider the problem of training \glspl*{ebm} on a collection of synthetic two-dimensional datasets given in Appendix \ref{app:training_ebm_2d}. For each dataset, we parametrise the model energy with a neural network, which we train by minimising the negative log-likelihood of the data. Assume the model is normalisable of the form,
\begin{equation*}
    p_\theta(x)=Z_\theta^{-1}e^{-E(\theta,x)},\qquad Z_\theta=\int_{\mathbb{R}^{d_x}}e^{-E(\theta,x)}\md x,
\end{equation*}
where $E(\theta,x)$ is the energy function parametrised by $\theta\in\mathbb{R}^{d_\theta}$ and $Z_\theta$ is the normalisation constant. Given a dataset of $M$ observations denoted $\{y^m\}_{m=1}^M$, the gradient of the negative log-likelihood is given,
\begin{equation*}
    \nabla_\theta\ell(\theta) = \mathbb{E}_{\pi_\theta}\left[-\nabla_\theta E(\theta,X)+\frac{1}{M}\sum_{m=1}^M\nabla_\theta E(\theta, y^m)\right].
\end{equation*}
Here, we can see the \gls*{pcd} algorithm is a special case of the \gls*{impdiff} framework. We tune each method separately to showcase a fair comparison using separately selected hyperparameters. We evaluate performance with respect to the number of outer iterations, which is equivalent to fixing the number of gradient back-propagations using our discretisation, reporting the Fisher Divergence in Figure \ref{fig:training_comparisons}. Even with separate tuning, there is a clear acceleration of \gls*{mmfpo} in the logarithmic scale. We also report the learnt probability densities in Figure \ref{fig:models} and see both methods produce models of visually similar qualities.

%% file: sections/main/conclusion.tex
We introduced a general diffusion-based framework for stochastic optimisation
when the loss gradient is an expectation under a parameter-dependent
distribution. The framework couples an optimiser with a continuously evolving
sampler and accommodates a wide range of dynamics within a common mean-field formulation. Under
well-posedness, existence of the stationary pair, and twisted joint strong
monotonicity, the mean-field system contracts exponentially to that stationary
solution. For the interacting particle system, we
obtained an explicit \(N^{-1/2}\) interaction-error bound and a parameter convergence bound.

The experiments illustrate the practical flexibility of the framework. In
Bayesian image deblurring, the higher-order Langevin variant improves the
reported MSE and SSIM and produces sharper reconstructions than the proximal
overdamped baseline. In the two-dimensional energy-based-model
experiments, the momentum method reaches lower Fisher divergence more rapidly
under the matched (gradient-evaluation) budget, while retaining
comparable visual sample quality.

\subsection{Limitations and Future Work}

The theory relies on global contractivity and uniform-moment assumptions, which
can be restrictive for the non-convex objectives encountered in practice. Moreover, the guarantees concern the continuous-time systems:
establishing finite-step error and stability for the proposed splitting schemes
requires a separate discretisation analysis. The time-dependent contraction
result is conditional on verifying its monotonicity rate for each schedule. Future work should therefore develop local or non-convex guarantees, analyse discretisation and adaptive preconditioning.

%% file: sections/appendix/theory_proofs.tex
\subsection{Proof of Theorem \ref{thm:invariant_measure}}\label{app:thm_invariant_measure}
Set \(\vartheta_0=\vartheta_\star\) and \(Z_0\sim\nu_{\theta_\star}\). We construct a solution to \eqref{eq:general_mf_system_opt}--\eqref{eq:general_mf_system_x} with initial law \(\delta_{\vartheta_\star}\otimes\nu_{\theta_\star}\) and show that this law is preserved.

Set \(\vartheta_t\equiv\vartheta_\star\). Since \(\Pi_\theta(\vartheta_\star)=\theta_\star\), let \(Z_t\) solve the frozen sampler equation
\[
    \mathrm dZ_t
    =
    b_{\theta_\star}(Z_t)\,\mathrm dt
    +
    \sigma_{\theta_\star}(Z_t)\,\mathrm dB_t,
    \qquad
    Z_0\sim\nu_{\theta_\star}.
\]
By invariance of \(\nu_{\theta_\star}\), \(\mathrm{Law}(Z_t)=\nu_{\theta_\star}\) for every \(t\ge0\). Hence
\begin{align*}
G_{\nu_t}(\vartheta_t)
&=
\int_{\sX\times\sV}
F(\Pi_\theta(\vartheta_t),\Pi_x(z))\,\nu_t(\mathrm dz)\\
&=
\int_{\sX\times\sV}
F(\theta_\star,\Pi_x(z))\,\nu_{\theta_\star}(\mathrm dz)\\
&=
\int_{\sX}F(\theta_\star,x)\,\pi_{\theta_\star}(\mathrm dx)
=
\nabla_\theta\ell(\theta_\star)
=0,
\end{align*}
where we used \((\Pi_x)_\#\nu_{\theta_\star}=\pi_{\theta_\star}\). Assumption~\Cref{ass:existence_minimizer} gives \(\Psi_t(\vartheta_\star,0)=0\), and therefore \(\mathrm d\vartheta_t=0\). Thus \((\vartheta_t,Z_t)\) solves the coupled system and has law \(\delta_{\vartheta_\star}\otimes\nu_{\theta_\star}\) for all \(t\ge0\). By pathwise uniqueness in Assumption~\ref{ass:strong_sol_exist_unique}, this constructed solution is the unique solution with the prescribed initial data. This proves that \(\mu_\star\) is invariant.

\subsection{Proof of Theorem \ref{thm:twisted_joint_strong_monotonicity}}\label{app:thm_joint_strong_monotonicity}

The proof follows a synchronous coupling argument with twisted metrics $D_\vartheta,D_z$, similar in style to Proposition 8 in \citet{theorem2literature_bolley_guillin_malrieu}.

Let \((\vartheta_t,\nu_t)\) and \((\tilde\vartheta_t,\tilde\nu_t)\) be two solutions. Choose an optimal coupling \(\Lambda_0\in\Gamma(\nu_0,\tilde\nu_0)\), take \((Z_0,\tilde Z_0)\sim\Lambda_0\), and drive the two sampler equations with the same Brownian motion:
\begin{align*}
    \mathrm dZ_t
    &=
    b_{\Pi_\theta(\vartheta_t)}(Z_t)\,\mathrm dt
    +
    \sigma_{\Pi_\theta(\vartheta_t)}(Z_t)\,\mathrm dB_t,\\
    \mathrm d\tilde Z_t
    &=
    b_{\Pi_\theta(\tilde\vartheta_t)}(\tilde Z_t)\,\mathrm dt
    +
    \sigma_{\Pi_\theta(\tilde\vartheta_t)}(\tilde Z_t)\,\mathrm dB_t.
\end{align*}
By uniqueness, the marginal laws of \(Z_t\) and \(\tilde Z_t\) are \(\nu_t\) and \(\tilde\nu_t\). Hence \(\Lambda_t:=\mathrm{Law}(Z_t,\tilde Z_t)\) is a coupling of \(\nu_t\) and \(\tilde\nu_t\).

Define
\[
    D_t
    :=
    \|\vartheta_t-\tilde\vartheta_t\|_{D_\vartheta}^2
    +
    \mathbb E\|Z_t-\tilde Z_t\|_{D_z}^2
\]
and let
\(\lambda_\vartheta^-,\lambda_\vartheta^+\) be the smallest and largest
eigenvalues of \(D_\vartheta\), and let \(\lambda_z^-,\lambda_z^+\) be the
corresponding eigenvalues of \({D_z}\). For the optimiser states,
\begin{align*}
    \frac{\mathrm d}{\mathrm dt}
    \|\vartheta_t-\tilde\vartheta_t\|_{D_\vartheta}^2
    =
    2\left\langle \vartheta_t-\tilde\vartheta_t,\,D_{\vartheta}\big(
    \Psi_t(\vartheta_t,G_{\nu_t}(\vartheta_t))
    -
    \Psi_t(\tilde\vartheta_t,G_{\tilde\nu_t}(\tilde\vartheta_t))\big)
    \right\rangle .
\end{align*}
For the sampler variables, let $Q_t:=\|Z_t-\tilde Z_t\|^2_{D_z}$ and introduce the notation,
\begin{equation*}
    \Delta Z_t=Z_t-\tilde Z_t,\qquad\Delta b_t=b_{\Pi_\theta(\vartheta_t)}(Z_t)
    -
    b_{\Pi_\theta(\tilde\vartheta_t)}(\tilde Z_t),\qquad \Delta\sigma_t=\sigma_{\Pi_\theta(\vartheta_t)}(Z_t)
    -
    \sigma_{\Pi_\theta(\tilde\vartheta_t)}(\tilde Z_t).
\end{equation*}
Then It\"o's formula under the synchronous coupling gives
\begin{align*}
    Q_t=Q_0+\int_0^t2\left\langle \Delta Z_s,D_z\Delta b_s
    \right\rangle + \| D^{1/2}_z\Delta \sigma_s \|^2_{\mathrm{HS}}\md s +M_t,\qquad M_t:= 2\int_0^t\left\langle D_z\Delta Z_s,\Delta \sigma_s\md B_s\right\rangle.
\end{align*}
If we assume under \ref{ass:strong_sol_exist_unique} that the diffusion coefficient is constant, $M_t=0$ and we can proceed to take expectations. Otherwise for any $t>0$ we pick a $T>t$ and show $M$ is a true martingale on $[0,T]$ and hence has zero expectation under the linear growth assumption. We proceed by bounding the expected square root of the quadratic variation under the linear-growth assumption in \ref{ass:strong_sol_exist_unique}. Indeed, the quadratic variation at time $T$ is given
\begin{align*}
    \langle M\rangle_T&=4\int_0^T\left\|(\Delta\sigma_s)^TD_z\Delta Z_s\right\|^2\md s \\
    &\le 4\int_0^T \left\|D^{1/2}_z\Delta\sigma_s\right\|^2_{\mathrm{HS}}\underbrace{\left\|D^{1/2}\Delta Z_s\right\|^2}_{=Q_s}\md s \\
    &\le 4T\left(\sup_{s\le T}\left\|D^{1/2}_z\Delta\sigma_s\right\|^2_{\mathrm{HS}}\right)\left(\sup_{s\le T} Q_s\right)
\end{align*}
hence taking expected square roots and applying Cauchy-Schwarz,
\begin{align*}
    \mathbb{E}\sqrt{\langle M\rangle_T}&\le 2\sqrt{T}\,\mathbb{E}\left[\left(\sup_{s\le T}\left\|D^{1/2}_z\Delta\sigma_s\right\|^2_{\mathrm{HS}}\right)^{1/2}\left(\sup_{s\le T} Q_s\right)^{1/2}\right]\\
    &\le 2\sqrt{T}\sqrt{\mathbb{E}\left[\sup_{s\le T}\left\|D^{1/2}_z\Delta\sigma_s\right\|^2_{\mathrm{HS}}\right]}\sqrt{\mathbb{E}\left[\sup_{s\le T} Q_s\right]}.
\end{align*}
Now for finite $T$, we may define $r_T:=\sup_{t\le T}\max\{\|\theta_t\|,\|\tilde\theta_t\|\}<\infty$ since $\theta_t,\tilde\theta_t$ are deterministic, continuous and bounded on $[0,T]$. Then on $[0,T]$ we have the modified linear-growth assumption,
\begin{equation}\label{eq:finite_horizon_linear_growth}
    \|b_{\Pi_{\theta}(\vartheta_s)}(Z_s)\|+\|\sigma_{\Pi_{\theta}(\vartheta_s)}(Z_s)\|\le K_T(1+\|Z_s\|),\qquad K_T:=K(1+r_T),
\end{equation}
which we may apply to bound
\begin{equation*}
    \left\|D^{1/2}_z\Delta\sigma_s\right\|^2_{\mathrm{HS}}\le 2\lambda^+_z\left(\|\sigma_{\Pi_\theta(\vartheta_s)}(Z_s)\|^2_{HS}+\|\sigma_{\Pi_\theta(\tilde\vartheta_s)}(\tilde Z_s)\|^2_{HS}\right)\le 4\lambda^+_z K_T^2\left(2+\|Z_s\|^2+\|\tilde Z_s\|^2\right),
\end{equation*}
and we also have
\begin{equation*}
    Q_s\le \lambda^+_z\|Z_s-\tilde Z_s\|^2\le 2\lambda_z^+\left(\|Z_s\|^2+\|\tilde Z_s\|^2\right).
\end{equation*}
The modified linear-growth assumption \eqref{eq:finite_horizon_linear_growth} also implies the expectations of the suprema of $\|Z_s\|^2$ and $\|\tilde Z_s\|^2$ are bounded (see for example Theorem 14.5.2 in \citet{Pascucci2024}). This gives the desired bound $\mathbb{E}\sqrt{\langle M\rangle_T}<\infty$ and so, by the Burkholder–Davis–Gundy inequality, $M_t$ is a true martingale on $[0,T]$. We can now take expectations then differentiate with respect to $t$ and obtain
\begin{align*}
    \frac{\mathrm d}{\mathrm dt}
    \mathbb E\|\Delta Z_t\|_{D_z}^2
    &=
    \mathbb E\Big[
    2\left\langle \Delta Z_t,\,
   D_z\Delta b_t
    \right\rangle+
    \left\|D_z^{1/2}\Delta\sigma_t\right\|_{\mathrm{HS}}^2
    \Big].
\end{align*}
Adding the two identities and applying \eqref{eq:joint_strong_monotonicity_twisted} with \(\Lambda=\Lambda_t\) yields
\[
    \frac{\mathrm d}{\mathrm dt}D_t
    \le
    -2\rho D_t\qquad\Rightarrow\qquad D_t\le e^{-2\rho t}D_0
\]
by Gronwall's lemma. Since \(\Lambda_0\) is optimal,
\begin{align*}
    D_0
    &=
    \|\vartheta_0-\tilde\vartheta_0\|_{D_\vartheta}^2
    + \mathbb E\|Z_0-\tilde Z_0\|_{D_z}^2\\
    &\le \max(\lambda_\vartheta^+,\lambda_z^+)\big(\|\vartheta_0-\tilde\vartheta_0\|^2
    + \mathbb E\|Z_0-\tilde Z_0\|^2\big)
    \\
    &= \max(\lambda_\vartheta^+,\lambda_z^+)\big(\|\vartheta_0-\tilde\vartheta_0\|^2
    +W_2^2(\nu_0,\tilde\nu_0)\big),
\end{align*}
and for every \(t\ge0\),
\begin{align*}
    \|\vartheta_t-\tilde\vartheta_t\|^2 +W_2^2(\nu_t,\tilde\nu_t)
    &\le
    \|\vartheta_t-\tilde\vartheta_t\|^2+\mathbb E\|Z_t-\tilde Z_t\|^2\\
    &\le\frac{1}{\min(\lambda^-_\vartheta,\lambda^-_z)}\big(\|\vartheta_t-\tilde\vartheta_t\|^2_{D_\vartheta}+\mathbb E\|Z_t-\tilde Z_t\|_{D_z}^2\big)
\end{align*}
we arrive at,
\begin{align*}
    \mathbf{d}((\vartheta_t,\nu_t),(\tilde{\vartheta}_t,\tilde{\nu}_t))^2
    \le
    \frac{\max(\lambda_\vartheta^+,\lambda_z^+)}{\min(\lambda^-_\vartheta,\lambda^-_z)}e^{-2\rho t}
    \mathbf{d}((\vartheta_0,\nu_0),(\tilde{\vartheta}_0,\tilde{\nu}_0))^2,
\end{align*}
proving the contraction result.

Now assume in addition that \ref{ass:existence_minimizer} holds.
By Theorem~\ref{thm:invariant_measure},
\((\vartheta_\star,\nu_{\theta_\star})\) is a stationary solution. Moreover,
for deterministic optimiser states,
\[
\begin{aligned}
W_2^2\!\left(
\delta_\vartheta\otimes\nu,\,
\delta_{\tilde\vartheta}\otimes\tilde\nu
\right)
&=
\|\vartheta-\tilde\vartheta\|^2
+
W_2^2(\nu,\tilde\nu).
\end{aligned}
\]
Applying the preceding contraction with
\((\tilde\vartheta_t,\tilde\nu_t)
=(\vartheta_\star,\nu_{\theta_\star})\) proves
\[
W_2(\mu_t,\mu_\star)
\leq
\sqrt\kappa\,e^{-\rho t}W_2(\mu_0,\mu_\star).
\]
Finally, let \(\delta_\vartheta\otimes\nu\) be any other stationary law of
the stated product form with finite second moment. Applying this estimate
between that stationary law and \(\mu_\star\) gives, for every \(t\geq0\),
\[
W_2(\delta_\vartheta\otimes\nu,\mu_\star)
\leq
\sqrt\kappa\,e^{-\rho t}
W_2(\delta_\vartheta\otimes\nu,\mu_\star).
\]
Letting \(t\to\infty\) shows that the left-hand side is zero. Hence
\(\delta_\vartheta\otimes\nu=\mu_\star\), proving uniqueness within the class asserted in the theorem.

\subsection{Proof of Theorem \ref{thm:finite_particle_error}}\label{app:proof_finite_particle_error}
We follow the synchronous coupling route using an $N$-copy mean-field system as in \citet{pgd_theory}. Introducing the notation
\[
\begin{aligned}
    \Delta\vartheta_t&:=\vartheta_t^N-\bar\vartheta_t,
    &
    \Delta Z_t^i&:=Z_t^i-\bar Z_t^i,\\
    \xi_t^N
    &:=
    \|\Delta\vartheta_t\|^2
    +
    \frac1N\sum_{i=1}^N\|\Delta Z_t^i\|^2,
    &
    \widehat\xi_t^N
    &:=
    \|\Delta\vartheta_t\|_{D_\vartheta}^2
    +
    \frac1N\sum_{i=1}^N\|\Delta Z_t^i\|_{{D_z}}^2.
\end{aligned}
\]
We also write $\theta^N_t=\Pi_\theta(\vartheta^N_t)$ and $\bar\theta_t=\Pi_\theta(\bar\vartheta_t)$. By It\"o's formula and the synchronous coupling,
\begin{align*}
    \md \widehat\xi_t^N
&= 
2\left\langle
    \Delta\vartheta_t,
    D_\vartheta
    \left[
    \Psi_t(\vartheta_t^N,G_{\nu_t^N}(\vartheta_t^N))
    -
    \Psi_t(\bar\vartheta_t,G_{\bar\nu_t}(\bar\vartheta_t))
    \right]
\right\rangle \md t\\
&\qquad+
\frac1N\sum_{i=1}^N
\Big\{
2\left\langle
    \Delta Z_t^i,
    {D_z}
    \left[
    b_{\theta^N_t}(Z_t^i)
    -
    b_{\bar\theta_t}(\bar Z_t^i)
    \right]
\right\rangle\\
&\qquad\qquad\qquad+
\left\|
    {D_z}^{1/2}
    \left[
    \sigma_{\theta^N_t}(Z_t^i)
    -
    \sigma_{\bar\theta_t}(\bar Z_t^i)
    \right]
\right\|_{\mathrm{HS}}^2
\Big\} \md t\\
&\qquad + \md M^N_t,\\
\md M^N_t&:=\frac{2}{N}\sum_{i=1}^N\left\langle D_z\Delta Z^i_t,\left[
    \sigma_{\theta^N_t}(Z_t^i)
    -
    \sigma_{\bar\theta_t}(\bar Z_t^i)
    \right]\md B^i_t\right\rangle,\qquad M^N_0=0.
\end{align*}

Note that since solutions exist, the process $M^N_t$ is a local martingale. Now add and subtract
\(D_\vartheta\Psi_t(\bar\vartheta_t,
G_{\bar\nu_t^{[N]}}(\bar\vartheta_t))\). Since
Assumption~\Cref{ass:twisted_joint_strong_monotonicity} holds for every choice of states, measures and couplings, we may apply it to the empirical
coupling
\[
    \Lambda_t^N
    :=
    \frac1N\sum_{i=1}^N
    \delta_{(Z_t^i,\bar Z_t^i)}
\]
giving
\begin{align*}
    \md \hat\xi_t^N&\le -2\rho\hat\xi_t^N\md t \\
    &\qquad +2\left\langle\Delta\vartheta_t,D_\vartheta
    \left[
    \Psi_t(\bar\vartheta_t,G_{\bar\nu_t^{[N]}}(\bar\vartheta_t))
    -
    \Psi_t(\bar\vartheta_t,G_{\bar\nu_t}(\bar\vartheta_t))\right]\right\rangle \md t + \md M^N_t
\end{align*}
almost surely. By Cauchy--Schwarz and Assumption~\Cref{ass:psi_F_lipschitz},
\begin{align*}
&\left|
\left\langle
    \Delta\vartheta_t,
    D_\vartheta
    \left[
    \Psi_t(\bar\vartheta_t,G_{\bar\nu_t^{[N]}}(\bar\vartheta_t))
    -
    \Psi_t(\bar\vartheta_t,G_{\bar\nu_t}(\bar\vartheta_t))
    \right]
\right\rangle
\right|\\
&\qquad\le
    \|\Delta\vartheta_t\|_{D_{\vartheta}}
    \left\|
    \Psi_t(\bar\vartheta_t,G_{\bar\nu_t^{[N]}}(\bar\vartheta_t))
    -
    \Psi_t(\bar\vartheta_t,G_{\bar\nu_t}(\bar\vartheta_t))
    \right\|_{D_{\vartheta}}
    \\
    &\qquad\le
    \sqrt{\lambda_\vartheta^+}\sqrt{\widehat\xi_t^N}
    \left\|
    \Psi_t(\bar\vartheta_t,G_{\bar\nu_t^{[N]}}(\bar\vartheta_t))
    -
    \Psi_t(\bar\vartheta_t,G_{\bar\nu_t}(\bar\vartheta_t))
    \right\|\\
&\qquad\le
    L_\Psi\sqrt{\lambda_\vartheta^+}\,
    \sqrt{\widehat\xi_t^N}
    \left\|
    G_{\bar\nu_t^{[N]}}(\bar\vartheta_t)
    -
    G_{\bar\nu_t}(\bar\vartheta_t)
    \right\|,
\end{align*}


where \(\lambda_\vartheta^+\) is the largest eigenvalue of \(D_\vartheta\).
Since \(\bar\vartheta_0\) is deterministic, \(\bar\vartheta_t\) is
deterministic. Moreover, because the sampler initial states are i.i.d.\ and
independent of the mutually independent Brownian motions and each copy shares a deterministic optimiser path, the random vectors
\[
    Y_i
    :=
    F(\Pi_\theta(\bar\vartheta_t),\Pi_x(\bar Z_t^i))
    -
    G_{\bar\nu_t}(\bar\vartheta_t),
    \qquad i=1,\ldots,N,
\]
are independent and centred. Therefore, by
Assumption~\ref{ass:mf_bounded_second_moment},
\begin{align*}
    \mathbb E
    \left[
    \left\|
    G_{\bar\nu_t^{[N]}}(\bar\vartheta_t)
    -
    G_{\bar\nu_t}(\bar\vartheta_t)
    \right\|^2
    \right]
    &=
    \mathbb E\left\|\frac1N\sum_{i=1}^NY_i\right\|^2=
    \frac1{N^2}\sum_{i=1}^N\mathbb E\|Y_i\|^2
    \le
    \frac{V_F}{N}.
\end{align*}
Set
\[
    c:=L_\Psi\sqrt{\lambda_\vartheta^+},
    \qquad
    g_t:=\left\|
    G_{\bar\nu_t^{[N]}}(\bar\vartheta_t)
    -
    G_{\bar\nu_t}(\bar\vartheta_t)
    \right\|,
\]
so that \(\mathbb E[g_t^2]\le V_F/N\) for every \(t\ge0\) by the last display. Combining the almost-sure inequality for \(\md\widehat\xi_t^N\) with the Cauchy--Schwarz bound above gives, almost surely,
\[
    \md\widehat\xi_t^N
    \le
    \left(-2\rho\,\widehat\xi_t^N+2c\sqrt{\widehat\xi_t^N}\,g_t\right)\md t
    +\md M_t^N .
\]
Since \(t\mapsto e^{2\rho t}\) is deterministic and of finite variation, It\"o's product rule and \(\widehat\xi_0^N=0\) (synchronous initialisation) give, almost surely for all \(t\ge0\),
\[
    e^{2\rho t}\widehat\xi_t^N
    \le
    2c\int_0^t e^{2\rho s}\sqrt{\widehat\xi_s^N}\,g_s\,\md s
    +
    \int_0^t e^{2\rho s}\,\md M_s^N .
\]
Since no moment bounds for the particle system are assumed, the local martingale \(M^N\) need not be a true martingale, and we use the standard localisation argument for continuous local martingales, followed by Fatou's lemma to remove the stopping times; see for example \citet{stochastic_gronwal_lemma}. For \(n\ge1\), define the stopping times
\[
    \tau_n
    :=
    \inf\left\{s\ge0:\ \widehat\xi_s^N\ge n\ \text{or}\ \langle M^N\rangle_s\ge n\right\},
    \qquad
    \inf\emptyset:=\infty .
\]
Since \(\widehat\xi^N\) and \(\langle M^N\rangle\) are continuous, adapted and finite-valued, and \(\widehat\xi_0^N=0\), we have \(\tau_n\uparrow\infty\) almost surely as \(n\to\infty\), and \(\widehat\xi^N_{s\wedge\tau_n}\le n\) for all \(s\ge0\). Moreover, for fixed \(t\ge0\), the stopped integral \(\int_0^{t\wedge\tau_n}e^{2\rho s}\,\md M_s^N\) has quadratic variation at most \(e^{4\rho t}n\); hence it is a square-integrable martingale with zero mean. Evaluating the previous inequality at \(t\wedge\tau_n\), taking expectations and using Tonelli's theorem (the integrand is non-negative) gives
\[
    \phi_n(t)
    :=
    \mathbb E\left[e^{2\rho(t\wedge\tau_n)}\,\widehat\xi^N_{t\wedge\tau_n}\right]
    \le
    2c\int_0^t
    \mathbb E\left[\mathbf 1_{\{s<\tau_n\}}\,e^{2\rho s}\sqrt{\widehat\xi_s^N}\,g_s\right]\md s .
\]
Since \(\mathbf 1_{\{s<\tau_n\}}e^{2\rho s}\widehat\xi_s^N\le e^{2\rho(s\wedge\tau_n)}\widehat\xi^N_{s\wedge\tau_n}\), the Cauchy--Schwarz inequality and \(\mathbb E[g_s^2]\le V_F/N\) give, for every \(s\ge0\),
\[
    \mathbb E\left[\mathbf 1_{\{s<\tau_n\}}\,e^{2\rho s}\sqrt{\widehat\xi_s^N}\,g_s\right]
    \le
    \left(\mathbb E\left[\mathbf 1_{\{s<\tau_n\}}\,e^{2\rho s}\,\widehat\xi_s^N\right]\right)^{1/2}
    \left(\mathbb E\left[e^{2\rho s}g_s^2\right]\right)^{1/2}
    \le
    e^{\rho s}\sqrt{\frac{V_F}{N}}\,\sqrt{\phi_n(s)} .
\]
Hence \(\phi_n\) is continuous (by dominated convergence), satisfies \(0\le\phi_n(r)\le e^{2\rho r}n\), and, with \(a(s):=2c\sqrt{V_F/N}\,e^{\rho s}\),
\[
    \phi_n(r)\le\int_0^r a(s)\sqrt{\phi_n(s)}\,\md s
    \qquad\text{for all } r\ge0 .
\]
We now use the following square-root form of Gr\"onwall's lemma, a special case of the Bihari--LaSalle inequality: if \(\phi:[0,t]\to[0,\infty)\) is bounded and measurable, \(a\ge0\) is integrable on \([0,t]\), and \(\phi(r)\le\int_0^r a(s)\sqrt{\phi(s)}\,\md s\) for all \(r\in[0,t]\), then \(\sqrt{\phi(t)}\le\frac12\int_0^t a(s)\,\md s\). Indeed, \(\Phi(r):=\int_0^r a(s)\sqrt{\phi(s)}\,\md s\) is absolutely continuous and \(\phi\le\Phi\); for every \(\delta>0\), the function \(r\mapsto\sqrt{\Phi(r)+\delta}\) is absolutely continuous with
\[
    \frac{\md}{\md r}\sqrt{\Phi(r)+\delta}
    =
    \frac{a(r)\sqrt{\phi(r)}}{2\sqrt{\Phi(r)+\delta}}
    \le
    \frac{a(r)}{2}
    \qquad\text{for almost every } r\in[0,t],
\]
so that \(\sqrt{\phi(t)}\le\sqrt{\Phi(t)+\delta}\le\sqrt\delta+\frac12\int_0^t a(s)\,\md s\), and we let \(\delta\downarrow0\). Applying this lemma to \(\phi_n\) yields
\[
    \sqrt{\phi_n(t)}
    \le
    c\sqrt{\frac{V_F}{N}}\int_0^t e^{\rho s}\,\md s
    =
    \frac{c}{\rho}\sqrt{\frac{V_F}{N}}\left(e^{\rho t}-1\right).
\]
Since \(\tau_n\uparrow\infty\) almost surely, \(e^{2\rho(t\wedge\tau_n)}\widehat\xi^N_{t\wedge\tau_n}\to e^{2\rho t}\widehat\xi_t^N\) almost surely, and Fatou's lemma gives
\[
    e^{2\rho t}\,\mathbb E\left[\widehat\xi_t^N\right]
    \le
    \liminf_{n\to\infty}\phi_n(t)
    \le
    \frac{c^2}{\rho^2}\,\frac{V_F}{N}\left(e^{\rho t}-1\right)^2 .
\]
Recalling \(c=L_\Psi\sqrt{\lambda_\vartheta^+}\), we arrive at
\[
    \left(\mathbb E\left[\widehat\xi^N_t\right]\right)^{1/2}
    \le
    \frac{L_\Psi}{\rho}\sqrt{\frac{\lambda^+_\vartheta V_F}{N}}\left(1-e^{-\rho t}\right).
\]
%
Finally,
\[
    W_2^2(\nu_t^N,\bar\nu_t^{[N]})
    \le
    \frac1N\sum_{i=1}^N\|Z_t^i-\bar Z_t^i\|^2,
\]
and
\[
    \xi_t^N
    \le
    \frac{1}{\min(\lambda_\vartheta^-,{\lambda_z^-})}
    \widehat\xi_t^N.
\]
Thus
\begin{align*}
    \Ltwonorm[\big]{\mathbf d\big((\vartheta_t^N,\nu_t^N),(\bar\vartheta_t,\bar\nu_t^{[N]})\big)}
    &\mathrel{\le}
    \left(\mathbb E\left[\xi_t^N\right]\right)^{1/2}
    \le
    \frac{L_\Psi}{\rho}
    \sqrt{
        \frac{\lambda_\vartheta^+V_F}
             {N\min(\lambda_\vartheta^-,{\lambda_z^-})}
    }
    (1-e^{-\rho t})\\
    &\le
    \frac{L_\Psi}{\rho}
    \sqrt{\frac{\kappa V_F}{N}}
    (1-e^{-\rho t}),
\end{align*}
which is the desired estimate.

\subsection{Proof of Theorem \ref{thm:parameter_error_full}}\label{app:cor_parameter_error_full}
Construct the \(N\)-copy mean-field
system and the interacting particle system with the coupling from
Theorem~\ref{thm:finite_particle_error}: set
\(\bar\vartheta_0=\vartheta_0^N\), let the mean-field sampler copies have
common initial law \(\nu_0\), and use the same initial sampler states and
Brownian motions in each particle--mean-field pair. Since every
mean-field sampler copy has law \(\bar\nu_t\), the common optimiser path
\(\bar\vartheta_t\) is the optimiser path of the one-copy mean-field system
started from \((\vartheta_0^N,\nu_0)\).

Write \(\bar\theta_t:=\Pi_\theta(\bar\vartheta_t)\), which is deterministic. Minkowski's inequality in \(\mathrm{L}^2(\mathbb P)\)
and the fact that \(\Pi_\theta\) is a Euclidean coordinate projection give
\[
\begin{aligned}
\Ltwonorm{\theta_t^N-\theta_\star}
&\leq
\Ltwonorm{\theta_t^N-\bar\theta_t}
+
\|\bar\theta_t-\theta_\star\|\\
&\leq
\Ltwonorm[\big]{
\mathbf d\big(
(\vartheta_t^N,\nu_t^N),
(\bar\vartheta_t,\bar\nu_t^{[N]})
\big)} +
\mathbf d\left(
(\bar\vartheta_t,\bar\nu_t),
(\vartheta_\star,\nu_{\theta_\star})
\right).
\end{aligned}
\]
By Theorem~\ref{thm:finite_particle_error}, the first term is bounded by
\[
\frac{L_\Psi}{\rho}
\sqrt{\frac{\kappa V_F}{N}}
(1-e^{-\rho t}).
\]
By Theorem~\ref{thm:invariant_measure},
\((\vartheta_\star,\nu_{\theta_\star})\) is a stationary mean-field
solution. Therefore
Theorem~\ref{thm:twisted_joint_strong_monotonicity} gives
\[
\begin{aligned}
&\mathbf d\left(
(\bar\vartheta_t,\bar\nu_t),
(\vartheta_\star,\nu_{\theta_\star})
\right)\leq
\sqrt\kappa\,e^{-\rho t}
\mathbf d\left(
(\bar\vartheta_0,\nu_0),
(\vartheta_\star,\nu_{\theta_\star})
\right)
=
\sqrt\kappa\,e^{-\rho t}\mathbf d_0,
\end{aligned}
\]
where \(\bar\vartheta_0=\vartheta_0^N\). Combining the last three displays
yields
\[
\Ltwonorm{\theta_t^N-\theta_\star}
\leq
\sqrt\kappa\,e^{-\rho t}\mathbf d_0
+
\frac{L_\Psi}{\rho}
\sqrt{\frac{\kappa V_F}{N}}
(1-e^{-\rho t}),
\]
which is the claimed parameter-error estimate.

%% file: sections/appendix/assumptions.tex
This appendix explains the role and verification of
\ref{ass:twisted_joint_strong_monotonicity}.  We first show why separate stability of the
optimiser and frozen sampler is not sufficient.  We then give two
complementary sufficient-condition routes: a general small-gain criterion
based on separate dissipativity and Lipschitz bounds, and a structured
criterion for projected fields with auxiliary variables.  The latter is
illustrated by a concrete heavy-ball/underdamped construction.

\subsection{Why joint cross-coupling control is necessary}
\label{app:cross_assumptions}

Separate dissipativity of the exact-gradient optimiser and of each frozen
sampler does not imply stability of the coupled dynamics.  Consider the
two-dimensional non-augmented system $Z_t=X_t$.  Let $a>0$,
$\omega\in\mathbb R$, and
\begin{align*}
M:=
\begin{pmatrix}
a&-\omega\\
\omega&a
\end{pmatrix}.
\end{align*}
Define
\begin{align*}
\ell(\theta):=\frac12\|\theta\|^2,
\qquad
\pi_\theta:=\mathcal N(\theta,a^{-1}I_2),
\qquad
F(\theta,x):=x,
\end{align*}
and choose the time-independent optimiser map
\begin{align*}
\Psi(\theta,g):=-g,
\qquad
\Psi(\theta,G_\nu(\theta))=-\mathbb E_\nu[X],
\qquad
\mathrm dX_t
=M(\theta_t-X_t)\,\mathrm dt+\sqrt2\,\mathrm dB_t.
\end{align*}
The gradient identity is valid because
\begin{align*}
\int F(\theta,x)\,\pi_\theta(\mathrm dx)
=\mathbb E_{\pi_\theta}[X]
=\theta
=\nabla_\theta\ell(\theta).
\end{align*}
The exact-gradient optimiser is strongly dissipative:
\begin{align*}
\left\langle
\theta-\tilde\theta,\,
\Psi(\theta,\nabla\ell(\theta))
-\Psi(\tilde\theta,\nabla\ell(\tilde\theta))
\right\rangle
=-\|\theta-\tilde\theta\|^2.
\end{align*}
The frozen sampler drift is also strongly dissipative.  Indeed, the
skew-symmetric part of $M$ contributes zero to its quadratic form, and hence
\begin{align*}
\langle x-\tilde x,b_\theta(x)-b_\theta(\tilde x)\rangle
&=-\langle x-\tilde x,M(x-\tilde x)\rangle\\
&=-a\|x-\tilde x\|^2.
\end{align*}
Moreover, $\pi_\theta$ is the invariant law of the frozen sampler.  Its mean
is $\theta$, and its covariance $\Sigma=a^{-1}I_2$ solves
\begin{align*}
M\Sigma+\Sigma M^\top
=a^{-1}(M+M^\top)
=2I_2.
\end{align*}

Nevertheless, with $m_t:=\mathbb E[X_t]$, the coupled first moments obey
\begin{align*}
\dot\theta_t=-m_t,
\qquad
\dot m_t=M(\theta_t-m_t),
\end{align*}
or equivalently
\begin{align*}
\frac{\mathrm d}{\mathrm dt}
\begin{pmatrix}\theta_t\\m_t\end{pmatrix}
=
\begin{pmatrix}
0&-I_2\\
M&-M
\end{pmatrix}
\begin{pmatrix}\theta_t\\m_t\end{pmatrix}.
\end{align*}
If $\mu$ is an eigenvalue of $M$, the corresponding eigenvalues
$\lambda$ of this block system satisfy
\begin{align*}
\lambda^2+\mu\lambda+\mu=0.
\end{align*}
For $a=1/2$ and $\omega=1$, the eigenvalues of $M$ are
$\mu=1/2\pm i$, and the block system has eigenvalues
\begin{align*}
0.1561565506\pm1.4232893066i,
\qquad
-0.6561565506\pm0.4232893066i.
\end{align*}
Two eigenvalues have positive real part, so generic first moments grow
exponentially.  Thus \ref{ass:twisted_joint_strong_monotonicity} must control
the optimiser--sampler cross-coupling;
separate componentwise stability alone cannot provide the required
contraction.

\subsection{A general separated sufficient condition}
\label{app:separated_assumptions}

We first give a general small-gain route to
\ref{ass:twisted_joint_strong_monotonicity}.  Fix symmetric
positive-definite matrices
\begin{align*}
D_\vartheta\in
\mathbb R^{(d_\theta+d_m)\times(d_\theta+d_m)},
\qquad
D_z\in
\mathbb R^{(d_x+d_v)\times(d_x+d_v)}.
\end{align*}
The following assumptions separate the optimiser response, the stochastic
gradient observation map, and the frozen sampler.

\begin{assumption}[Contractive optimiser response]
\label{ass:psi_assumptions}
There exist constants $\rho_\Psi>0$ and $L_\Psi\geq0$ such that, for every
$t\geq0$, every
$\vartheta,\tilde\vartheta\in\Theta\times\sM$, and every
$\nu,\tilde\nu\in\mathcal P_2(\sX\times\sV)$,
\begin{align*}
&\left\langle
\vartheta-\tilde\vartheta,\,
D_\vartheta
\left[
\Psi_t(\vartheta,G_\nu(\vartheta))
-\Psi_t(\tilde\vartheta,G_\nu(\tilde\vartheta))
\right]
\right\rangle
\leq
-\rho_\Psi
\|\vartheta-\tilde\vartheta\|_{D_\vartheta}^2,\\
&\left\|
\Psi_t(\vartheta,G_\nu(\vartheta))
-\Psi_t(\vartheta,G_{\tilde\nu}(\vartheta))
\right\|
\leq
L_\Psi
\left\|G_\nu(\vartheta)-G_{\tilde\nu}(\vartheta)\right\|.
\end{align*}
\end{assumption}

\begin{assumption}[Lipschitz gradient observation]
\label{ass:lipschitz_F}
There exists $L_F\geq0$ such that, for every
$\vartheta\in\Theta\times\sM$ and every
$z,\tilde z\in\sX\times\sV$,
\begin{align*}
\left\|
F(\Pi_\theta(\vartheta),\Pi_x(z))
-F(\Pi_\theta(\vartheta),\Pi_x(\tilde z))
\right\|
\leq
L_F\|z-\tilde z\|.
\end{align*}
\end{assumption}

\begin{assumption}[Contractive frozen sampler and Lipschitz coupling]
\label{ass:SDE_assumptions}
There exist constants $\rho_b>0$ and $L_b,L_\sigma\geq0$ such that, for
every $\vartheta,\tilde\vartheta\in\Theta\times\sM$ and every
$z,\tilde z\in\sX\times\sV$,
\begin{align*}
&\left\langle
z-\tilde z,\,
D_z
\left[
b_{\Pi_\theta(\vartheta)}(z)
-b_{\Pi_\theta(\vartheta)}(\tilde z)
\right]
\right\rangle
\leq
-\rho_b\|z-\tilde z\|_{D_z}^2,\\
&\left\|
b_{\Pi_\theta(\vartheta)}(z)
-b_{\Pi_\theta(\tilde\vartheta)}(z)
\right\|
\leq
L_b\|\vartheta-\tilde\vartheta\|,\\
&\left\|
\sigma_{\Pi_\theta(\vartheta)}(z)
-\sigma_{\Pi_\theta(\tilde\vartheta)}(\tilde z)
\right\|_{\mathrm{HS}}
\leq
L_\sigma
\sqrt{
\|\vartheta-\tilde\vartheta\|^2+\|z-\tilde z\|^2
}.
\end{align*}
\end{assumption}

\begin{corollary}[Separate sufficient conditions for
\ref{ass:twisted_joint_strong_monotonicity}]
\label{cor:separate_monotonicity}
Let
\begin{align*}
\kappa:=
\frac{\lambda_{\max}(\operatorname{diag}(D_\vartheta,D_z))}
{\lambda_{\min}(\operatorname{diag}(D_\vartheta,D_z))}.
\end{align*}
Define
\begin{align*}
H:=
\begin{pmatrix}
\rho_\Psi-\frac12\kappa L_\sigma^2&
-\frac12\sqrt\kappa(L_\Psi L_F+L_b)\\
-\frac12\sqrt\kappa(L_\Psi L_F+L_b)&
\rho_b-\frac12\kappa L_\sigma^2
\end{pmatrix}.
\end{align*}
Under \ref{ass:psi_assumptions},
\ref{ass:lipschitz_F}, and \ref{ass:SDE_assumptions},
\ref{ass:twisted_joint_strong_monotonicity} holds if
\begin{equation}
\label{eq:cor3_indv_cond_joint}
\begin{aligned}
2\rho_\Psi-\kappa L_\sigma^2&>0,\\
2\rho_b-\kappa L_\sigma^2&>0,\\
(2\rho_\Psi-\kappa L_\sigma^2)
(2\rho_b-\kappa L_\sigma^2)
&>
\kappa(L_\Psi L_F+L_b)^2.
\end{aligned}
\end{equation}
More precisely, under \eqref{eq:cor3_indv_cond_joint}, $H$ is positive
definite and \ref{ass:twisted_joint_strong_monotonicity} holds with rate
$\rho=\lambda_{\min}(H)>0$.
A simpler sufficient condition is
\begin{equation}
\label{eq:cor3_indv_cond_separate}
\rho_\Psi
>
\frac{\kappa L_\sigma^2}{2}
+
\frac{\sqrt\kappa}{2}(L_\Psi L_F+L_b),
\qquad
\rho_b
>
\frac{\kappa L_\sigma^2}{2}
+
\frac{\sqrt\kappa}{2}(L_\Psi L_F+L_b).
\end{equation}
Consequently, if \ref{ass:strong_sol_exist_unique} also holds,
Theorem~\ref{thm:twisted_joint_strong_monotonicity} gives exponential
contraction.
\end{corollary}

\begin{proof}
Fix arbitrary $t\geq0$,
$\vartheta,\tilde\vartheta\in\Theta\times\sM$,
$\nu,\tilde\nu\in\mathcal P_2(\sX\times\sV)$, and
$\Lambda\in\Gamma(\nu,\tilde\nu)$.  Define
\begin{align*}
a:=\|\vartheta-\tilde\vartheta\|_{D_\vartheta},
\qquad
b:=
\left(
\int\|z-\tilde z\|_{D_z}^2\,
\Lambda(\mathrm dz,\mathrm d\tilde z)
\right)^{1/2}.
\end{align*}
Also write
\begin{align*}
\lambda_-:=
\lambda_{\min}(\operatorname{diag}(D_\vartheta,D_z)),
\qquad
\lambda_+:=
\lambda_{\max}(\operatorname{diag}(D_\vartheta,D_z)),
\end{align*}
so that $\kappa=\lambda_+/\lambda_-$.

For the optimiser term, add and subtract
$\Psi_t(\tilde\vartheta,G_\nu(\tilde\vartheta))$.
\ref{ass:psi_assumptions} bounds the same-measure term by
$-\rho_\Psi a^2$.  For the remaining term, Cauchy--Schwarz in the
$D_\vartheta$-metric gives
\begin{align*}
&\left|
\left\langle
\vartheta-\tilde\vartheta,\,
D_\vartheta
\left[
\Psi_t(\tilde\vartheta,G_\nu(\tilde\vartheta))
-\Psi_t(\tilde\vartheta,G_{\tilde\nu}(\tilde\vartheta))
\right]
\right\rangle
\right|\leq
a\sqrt{\lambda_+}\,L_\Psi
\|G_\nu(\tilde\vartheta)-G_{\tilde\nu}(\tilde\vartheta)\|.
\end{align*}
The coupling identity, \ref{ass:lipschitz_F}, and
Cauchy--Schwarz imply
\begin{align*}
\|G_\nu(\tilde\vartheta)-G_{\tilde\nu}(\tilde\vartheta)\|
&\leq
L_F
\left(
\int\|z-\tilde z\|^2\,
\Lambda(\mathrm dz,\mathrm d\tilde z)
\right)^{1/2}\leq
\frac{L_F}{\sqrt{\lambda_-}}\,b.
\end{align*}
Therefore the full optimiser contribution is at most
\begin{align*}
-\rho_\Psi a^2+\sqrt\kappa L_\Psi L_F\,ab.
\end{align*}

For the sampler drift, split
\begin{align*}
b_{\Pi_\theta(\vartheta)}(z)
-b_{\Pi_\theta(\tilde\vartheta)}(\tilde z)
&=
\left[
b_{\Pi_\theta(\vartheta)}(z)
-b_{\Pi_\theta(\vartheta)}(\tilde z)
\right]+
\left[
b_{\Pi_\theta(\vartheta)}(\tilde z)
-b_{\Pi_\theta(\tilde\vartheta)}(\tilde z)
\right].
\end{align*}
\ref{ass:SDE_assumptions} and the same norm comparisons bound the
integrated drift contribution by
\begin{align*}
-\rho_b b^2+\sqrt\kappa L_b\,ab.
\end{align*}
The diffusion contribution satisfies
\begin{align*}
&\int
\left\|D_z^{1/2}
\left[
\sigma_{\Pi_\theta(\vartheta)}(z)
-\sigma_{\Pi_\theta(\tilde\vartheta)}(\tilde z)
\right]
\right\|_{\mathrm{HS}}^2
\Lambda(\mathrm dz,\mathrm d\tilde z)\leq
\kappa L_\sigma^2(a^2+b^2).
\end{align*}
Including the factor $1/2$ in
\eqref{eq:joint_strong_monotonicity_twisted}, its left-hand side is therefore
at most
\begin{align*}
-
\begin{pmatrix}a\\b\end{pmatrix}^{\!\top}
H
\begin{pmatrix}a\\b\end{pmatrix}.
\end{align*}
The three inequalities in \eqref{eq:cor3_indv_cond_joint} are precisely the
positive-diagonal and positive-determinant conditions for $H$.  It is
therefore positive definite.  Since its smallest eigenvalue is
$\rho=\lambda_{\min}(H)>0$, the preceding display is at most
$-\rho(a^2+b^2)$, which is
\ref{ass:twisted_joint_strong_monotonicity}.  Under
\eqref{eq:cor3_indv_cond_separate}, each diagonal entry is strictly larger
than the absolute value of the off-diagonal entry, so strict diagonal
dominance gives the simpler sufficient condition.
\end{proof}

\subsection{A sufficient condition based on the projected field}
\label{app:cor_projected_strong_convexity}

For systems with position-like and auxiliary variables, the following route
uses strong monotonicity of the projected $(\theta,x)$-field and separate
dissipation of the auxiliary coordinates.

\begin{assumption}[Strongly monotone projected field]
\label{ass:auxiliary_separation}
Assume
\begin{align*}
\Theta=\mathbb R^{d_\theta},
\qquad
\sM=\mathbb R^{d_m},
\qquad
\sX=\mathbb R^{d_x},
\qquad
\sV=\mathbb R^{d_v}.
\end{align*}
Write
$\vartheta=(\theta,m)$, $z=(x,v)$, and let
$\Pi_\theta(\vartheta)=\theta$, $\Pi_x(z)=x$.
Assume that the potential-dependent force in the sampler is the score
$\nabla_x\log\pi_\theta(x)=-\nabla_xE(\theta,x)$.  Define the projected
dissipative field by using the negative of this score in its $x$-component:
\begin{align*}
\Phi(\theta,x):=
\begin{pmatrix}
F(\theta,x)\\
\nabla_xE(\theta,x)
\end{pmatrix}.
\end{align*}
Assume there exists $\rho_\Phi>0$ such that, for every
$\theta,\tilde\theta\in\mathbb R^{d_\theta}$ and every
$x,\tilde x\in\mathbb R^{d_x}$,
\begin{align*}
\left\langle
(\theta,x)-(\tilde\theta,\tilde x),\,
\Phi(\theta,x)-\Phi(\tilde\theta,\tilde x)
\right\rangle
\geq
\rho_\Phi
\left(
\|\theta-\tilde\theta\|^2+\|x-\tilde x\|^2
\right).
\end{align*}
\end{assumption}

\begin{assumption}[Dissipation of auxiliary variables]
\label{ass:auxiliary_dissipation}
Work under the decomposition in
\ref{ass:auxiliary_separation}.  There exist constants
$\beta>0$, $A_\theta,A_x\geq0$, and symmetric positive-definite matrices
\begin{align*}
D_\vartheta&\in
\mathbb R^{(d_\theta+d_m)\times(d_\theta+d_m)},&
D_z&\in
\mathbb R^{(d_x+d_v)\times(d_x+d_v)},\\
D_m&\in\mathbb R^{d_m\times d_m},&
D_v&\in\mathbb R^{d_v\times d_v},
\end{align*}
such that the following holds.  For every $t\geq0$, every
$\vartheta=(\theta,m)$, $\tilde\vartheta=(\tilde\theta,\tilde m)$, every
$\nu,\tilde\nu\in\mathcal P_2(\sX\times\sV)$, and every
$\Lambda\in\Gamma(\nu,\tilde\nu)$, define
\begin{align*}
R_\Lambda
:=
\|\theta-\tilde\theta\|^2
+
\int\|x-\tilde x\|^2\,
\Lambda(\mathrm dz,\mathrm d\tilde z).
\end{align*}
Then
\begin{equation}
\label{eq:optimiser_auxiliary_dissipation}
\begin{aligned}
&\left\langle
\vartheta-\tilde\vartheta,\,
D_\vartheta
\left[
\Psi_t(\vartheta,G_\nu(\vartheta))
-\Psi_t(\tilde\vartheta,G_{\tilde\nu}(\tilde\vartheta))
\right]
\right\rangle\leq
-\beta
\left\langle
\theta-\tilde\theta,\,
G_\nu(\vartheta)-G_{\tilde\nu}(\tilde\vartheta)
\right\rangle
-\|m-\tilde m\|_{D_m}^2
+A_\theta R_\Lambda,
\end{aligned}
\end{equation}
and
\begin{equation}
\label{eq:sampler_auxiliary_dissipation}
\begin{aligned}
\int
\left\langle
z-\tilde z,\,
D_z
\left[
b_{\Pi_\theta(\vartheta)}(z)
-b_{\Pi_\theta(\tilde\vartheta)}(\tilde z)
\right]
\right\rangle
\Lambda(\mathrm dz,\mathrm d\tilde z) \leq &
-\beta
\int
\left\langle
x-\tilde x,\,
\nabla_xE(\theta,x)-\nabla_xE(\tilde\theta,\tilde x)
\right\rangle
\Lambda(\mathrm dz,\mathrm d\tilde z)\\
&
-\int\|v-\tilde v\|_{D_v}^2\,
\Lambda(\mathrm dz,\mathrm d\tilde z)
+A_x R_\Lambda.
\end{aligned}
\end{equation}
\end{assumption}

\ref{ass:auxiliary_separation} controls the joint projected field
in $(\theta,x)$, whereas
\ref{ass:auxiliary_dissipation} requires the auxiliary coordinates
$(m,v)$ to dissipate strongly enough to absorb the cross terms generated by
the twisted metrics.  In the conservative case
$\Phi=\nabla_{(\theta,x)}\varphi$, joint
$\rho_\Phi$-strong convexity of $\varphi$ implies the required strong
monotonicity.  For \gls*{mmle}, this applies when $F=\nabla_\theta E$, so
that $\varphi=E$ is the negative joint log-density up to an additive
constant independent of $(\theta,x)$;
see \citet{pmlr-v206-kuntz23a,akyildiz2025interacting}.

\begin{corollary}[Projected-field sufficient condition for
\ref{ass:twisted_joint_strong_monotonicity}]
\label{cor:projected_strong_convexity}
Suppose \ref{ass:auxiliary_separation} and
\ref{ass:auxiliary_dissipation} hold, and suppose that
$\sigma_\theta(z)$ is independent of $(\theta,z)$.  If
\begin{align*}
c:=\beta\rho_\Phi-A_\theta-A_x>0,
\end{align*}
define
\begin{align*}
K_\vartheta:=
\begin{pmatrix}
cI_{d_\theta}&0\\
0&D_m
\end{pmatrix},
\qquad
K_z:=
\begin{pmatrix}
cI_{d_x}&0\\
0&D_v
\end{pmatrix},
\end{align*}
and
\begin{align*}
\rho:=
\min\left\{
\lambda_{\min}\!\left(
D_\vartheta^{-1/2}K_\vartheta D_\vartheta^{-1/2}
\right),
\lambda_{\min}\!\left(
D_z^{-1/2}K_zD_z^{-1/2}
\right)
\right\}.
\end{align*}
Then $\rho>0$ and
\ref{ass:twisted_joint_strong_monotonicity} holds with the
matrices $D_\vartheta,D_z$ and rate $\rho$.  Consequently, if
\ref{ass:strong_sol_exist_unique} also holds,
Theorem~\ref{thm:twisted_joint_strong_monotonicity} gives exponential
contraction.
\end{corollary}

\begin{proof}
Fix arbitrary $t\geq0$,
$\vartheta,\tilde\vartheta\in\Theta\times\sM$,
$\nu,\tilde\nu\in\mathcal P_2(\sX\times\sV)$, and
$\Lambda\in\Gamma(\nu,\tilde\nu)$.
Write
\begin{align*}
\Delta\vartheta&:=\vartheta-\tilde\vartheta,&
\Delta\Psi_t
&:=
\Psi_t(\vartheta,G_\nu(\vartheta))
-\Psi_t(\tilde\vartheta,G_{\tilde\nu}(\tilde\vartheta)),\\
\Delta z&:=z-\tilde z,&
\Delta b
&:=
b_{\Pi_\theta(\vartheta)}(z)
-b_{\Pi_\theta(\tilde\vartheta)}(\tilde z).
\end{align*}
Because the diffusion coefficient is independent of $(\theta,z)$, the
diffusion-difference term in
\eqref{eq:joint_strong_monotonicity_twisted} is zero.  Adding
\eqref{eq:optimiser_auxiliary_dissipation} and
\eqref{eq:sampler_auxiliary_dissipation} gives
\begin{align*}
\langle\Delta\vartheta,D_\vartheta\Delta\Psi_t\rangle
+
\int\langle\Delta z,D_z\Delta b\rangle\,
\Lambda(\mathrm dz,\mathrm d\tilde z)\leq &
-\beta
\left\langle
\theta-\tilde\theta,\,
G_\nu(\vartheta)-G_{\tilde\nu}(\tilde\vartheta)
\right\rangle\\
&\qquad
-\beta\int
\left\langle
x-\tilde x,\,
\nabla_xE(\theta,x)-\nabla_xE(\tilde\theta,\tilde x)
\right\rangle
\Lambda(\mathrm dz,\mathrm d\tilde z)\\
&\qquad
-\|m-\tilde m\|_{D_m}^2
-\int\|v-\tilde v\|_{D_v}^2\,
\Lambda(\mathrm dz,\mathrm d\tilde z)
+(A_\theta+A_x)R_\Lambda.
\end{align*}
The marginal identities for $\Lambda$ imply
\begin{align*}
G_\nu(\vartheta)-G_{\tilde\nu}(\tilde\vartheta)
=
\int
\left[
F(\theta,x)-F(\tilde\theta,\tilde x)
\right]
\Lambda(\mathrm dz,\mathrm d\tilde z).
\end{align*}
The two terms multiplied by $-\beta$ therefore combine into
\begin{align*}
-\beta
\int
\left\langle
(\theta,x)-(\tilde\theta,\tilde x),\,
\Phi(\theta,x)-\Phi(\tilde\theta,\tilde x)
\right\rangle
\Lambda(\mathrm dz,\mathrm d\tilde z),
\end{align*}
which is at most $-\beta\rho_\Phi R_\Lambda$.  Hence
\begin{align*}
&\langle\Delta\vartheta,D_\vartheta\Delta\Psi_t\rangle
+
\int\langle\Delta z,D_z\Delta b\rangle\,
\Lambda(\mathrm dz,\mathrm d\tilde z)\\
&\quad\leq
-cR_\Lambda
-\|m-\tilde m\|_{D_m}^2
-\int\|v-\tilde v\|_{D_v}^2\,
\Lambda(\mathrm dz,\mathrm d\tilde z)\\
&\quad=
-\|\Delta\vartheta\|_{K_\vartheta}^2
-\int\|\Delta z\|_{K_z}^2\,
\Lambda(\mathrm dz,\mathrm d\tilde z).
\end{align*}
By the definition of $\rho$,
\begin{align*}
K_\vartheta\succeq\rho D_\vartheta,
\qquad
K_z\succeq\rho D_z.
\end{align*}
Substitution into the preceding display gives exactly
\eqref{eq:joint_strong_monotonicity_twisted}.
\end{proof}

\subsubsection{Application to momentum-enriched dynamics}
\label{app:mpgd_result}

\begin{proposition}[A hypocoercive verification of
\ref{ass:twisted_joint_strong_monotonicity}]
Suppose \ref{ass:auxiliary_separation} holds and that
$\Phi$ is $L_\Phi$-Lipschitz.  Take $d_m=d_\theta$ and $d_v=d_x$, and
consider the unit-mass heavy-ball optimiser and underdamped sampler with a
common friction coefficient $\gamma>0$:
\begin{align*}
\Psi_t(\vartheta,G_\nu(\vartheta))
=
\begin{pmatrix}
m\\
-G_\nu(\vartheta)-\gamma m
\end{pmatrix},
\qquad
b_\theta(x,v)
=
\begin{pmatrix}
v\\
-\nabla_xE(\theta,x)-\gamma v
\end{pmatrix},
\end{align*}
with constant sampler diffusion
\begin{align*}
\sigma_\theta(x,v)
=
\begin{pmatrix}
0_{d_x\times d_x}\\
\sqrt{2\gamma}\,I_{d_x}
\end{pmatrix}.
\end{align*}
If
\begin{align*}
\gamma>\frac{\sqrt2\,L_\Phi}{\sqrt{\rho_\Phi}},
\end{align*}
then there exist symmetric positive-definite matrices
$D_\vartheta,D_z$ and a rate $\rho>0$ for which
\ref{ass:twisted_joint_strong_monotonicity} holds.
\end{proposition}

\begin{proof}
For a coupling $\Lambda\in\Gamma(\nu,\tilde\nu)$, set
\begin{align*}
\Delta\theta:=\theta-\tilde\theta,
\qquad
\Delta m:=m-\tilde m,
\qquad
\Delta G:=G_\nu(\vartheta)-G_{\tilde\nu}(\tilde\vartheta),
\end{align*}
and recall
\begin{align*}
R_\Lambda
=
\|\Delta\theta\|^2
+
\int\|x-\tilde x\|^2\,
\Lambda(\mathrm dz,\mathrm d\tilde z).
\end{align*}
The coupling identity, Jensen's inequality, and Lipschitz continuity of
$\Phi$ give
\begin{align*}
\|\Delta G\|^2
&=
\left\|
\int
\left[F(\theta,x)-F(\tilde\theta,\tilde x)\right]
\Lambda(\mathrm dz,\mathrm d\tilde z)
\right\|^2\leq
L_\Phi^2R_\Lambda.
\end{align*}

The strict friction condition allows us to choose
$0<\varepsilon<1$ such that
\begin{align*}
\frac{2L_\Phi^2}{(1-\varepsilon)\gamma^2}<\rho_\Phi.
\end{align*}
Fix any $\beta>0$, and set
\begin{align*}
\delta:=\frac{2\beta}{\gamma},
\qquad
\eta:=(1-\varepsilon)\beta.
\end{align*}
Choose the hypocoercive matrices
\begin{align*}
D_\vartheta
=
\begin{pmatrix}
\beta\gamma I_{d_\theta}&\beta I_{d_\theta}\\
\beta I_{d_\theta}&\delta I_{d_\theta}
\end{pmatrix},
\qquad
D_z
=
\begin{pmatrix}
\beta\gamma I_{d_x}&\beta I_{d_x}\\
\beta I_{d_x}&\delta I_{d_x}
\end{pmatrix}.
\end{align*}
They are positive definite by the Schur-complement criterion because
$\delta>\beta/\gamma$. Now, direct multiplication gives
\begin{align*}
\langle\Delta\vartheta,D_\vartheta\Delta\Psi_t\rangle
&=
-\beta\langle\Delta\theta,\Delta G\rangle
-(\delta\gamma-\beta)\|\Delta m\|^2
-\delta\langle\Delta m,\Delta G\rangle
\end{align*}
and applying Young's inequality yields
\begin{align*}
-\delta\langle\Delta m,\Delta G\rangle
\leq
\eta\|\Delta m\|^2
+
\frac{\delta^2}{4\eta}\|\Delta G\|^2.
\end{align*}
Since $\delta\gamma-\beta=\beta$, it follows that
\begin{align*}
\langle\Delta\vartheta,D_\vartheta\Delta\Psi_t\rangle
\leq
-\beta\langle\Delta\theta,\Delta G\rangle
-\varepsilon\beta\|\Delta m\|^2
+
\frac{\beta L_\Phi^2}{(1-\varepsilon)\gamma^2}R_\Lambda.
\end{align*}
Thus \eqref{eq:optimiser_auxiliary_dissipation} holds with
\begin{align*}
D_m=\varepsilon\beta I_{d_\theta},
\qquad
A_\theta
=\frac{\beta L_\Phi^2}{(1-\varepsilon)\gamma^2}.
\end{align*}

For the sampler, write
\begin{align*}
\Delta H_x
:=
\nabla_xE(\theta,x)-\nabla_xE(\tilde\theta,\tilde x).
\end{align*}
The same calculation gives, pointwise,
\begin{align*}
\langle\Delta z,D_z\Delta b\rangle
&\leq
-\beta\langle x-\tilde x,\Delta H_x\rangle
-\varepsilon\beta\|v-\tilde v\|^2\\
&\quad+
\frac{\beta}{(1-\varepsilon)\gamma^2}
\|\Delta H_x\|^2.
\end{align*}
Because $\nabla_xE$ is the second component of the
$L_\Phi$-Lipschitz field $\Phi$,
\begin{align*}
\|\Delta H_x\|^2
\leq
L_\Phi^2
\left(
\|\theta-\tilde\theta\|^2+\|x-\tilde x\|^2
\right).
\end{align*}
After integration, \eqref{eq:sampler_auxiliary_dissipation} holds with
\begin{align*}
D_v=\varepsilon\beta I_{d_x},
\qquad
A_x
=\frac{\beta L_\Phi^2}{(1-\varepsilon)\gamma^2}.
\end{align*}
Finally,
\begin{align*}
A_\theta+A_x
=
\frac{2\beta L_\Phi^2}{(1-\varepsilon)\gamma^2}
<
\beta\rho_\Phi.
\end{align*}
Corollary~\ref{cor:projected_strong_convexity} therefore verifies
\ref{ass:twisted_joint_strong_monotonicity} in the displayed
hypocoercive metrics.
\end{proof}

%% file: sections/appendix/mmle_details.tex
Here we discuss the Bayesian image-deblurring experiments introduced in Section~\ref{sec:mmle_image_deblurring}. Given a blurred and noisy observation $y=Bx+\varepsilon$, obtained through an ill-conditioned blurring operator $B$ and $\varepsilon\sim\mathcal{N}(0,\sigma^2I_{d_x})$, the goal is to recover clean latent images $x$. We equip this ill-conditioned inverse problem with a total-variation prior, with strength governed by an unknown scalar parameter $\theta$. The hyper-parameter $\theta$ typically requires manual tuning but instead we estimate an optimal value by framing this as a high-dimensional \gls*{mmle} problem. With the total-variation prior, we have a non-differentiable posterior over the latent image which closely follows the setup in Section 5.3 of \citet{encinar2025proximal} and the image reconstruction problems in \citet{durmus2018efficient,goldman2022gradient}.

Let $x^\star\in\mathbb{R}^{d_x}$ denote the ground-truth image, with dimension $d_x=n_w\times n_h$. We consider  the blurring operator $B\in\mathbb{R}^{d_x\times d_x}$ following \citet{encinar2025proximal}. In particular, $B$ is a uniform local averaging operator taking the value of an interior pixel $x^\star_{ij}$ and replacing with the uniform average of a $P\times P$ neighbourhood of pixels, with truncation and renormalisation at the image boundaries. We equip the latent image $x^\star$ with the anisotropic total-variation log-prior,
\begin{equation*}
    p_\theta(x)=C(\theta)^{-1}\exp\left(-e^{\theta}\mathrm{TV}(x)\right),\qquad C(\theta)=\int_{\mathbb{R}^{d_x}}\exp\left(-e^{\theta}\mathrm{TV}(x)\right)\md x.
\end{equation*}
Positive one-homogeneity of the total-variation function and a change of variables $u=e^\theta x$ give the normalisation constant $C(\theta)\propto e^{-d_x\theta}$ which we assume to be finite for modelling purposes. The parameter $e^\theta>0$ controls the strength of this prior, which typically requires manual tuning \citep{durmus2018efficient,goldman2022gradient}. Note we exponentiate $\theta$ to ensure positivity. Here, $\mathrm{TV}(x):=\|\nabla_dx\|_1$ is the anisotropic total-variation, with $\nabla_d$ the two-dimensional discrete gradient operator, which is non-differentiable, and $\|\cdot\|_1$ is the $\ell_1$ norm. The full posterior is therefore of the form,
\begin{equation*}
    \pi_\theta(x)=p_\theta(x\mid y)\propto\exp\left(-\frac{1}{2\sigma^2}\|y-Bx\|^2_2-e^\theta\mathrm{TV}(x)-\log C(\theta)\right),
\end{equation*}
where $y$ is considered as fixed data. We may derive the gradient of the negative log-marginal likelihood using Fisher's identity,
\begin{align*}
    \nabla_\theta\ell(\theta)
    &=
    -\mathbb{E}_{x\sim \pi_\theta}
    [\nabla_\theta\log p_\theta(x,y)] .
\end{align*}
Since the log-likelihood
\begin{equation*}
    \log p(y\mid x) = -\frac{1}{2\sigma^2}\|y-Bx\|^2_2 + \text{constant}
\end{equation*}
is independent of \(\theta\), we have
\begin{equation*}
    \nabla_\theta\log p_\theta(x,y)= d_x-e^\theta\mathrm{TV}(x),
\end{equation*}
and hence
\begin{equation*}
    \nabla_\theta\ell(\theta)
    =
    \mathbb{E}_{x\sim \pi_\theta}
    [e^\theta\mathrm{TV}(x)]-d_x .
\end{equation*}
Thus the objective function is expressed as in $\eqref{eq:obj_func}$ by an expectation over $\pi_\theta$ with $F(\theta,x)=e^{\theta}\mathrm{TV}(x)-d_x$.

\subsubsection{Proximal Dynamics}

Sampling from $\pi_\theta$ for a fixed $\theta$ is challenging as the total-variation, $\mathrm{TV}(x)$, is non-differentiable. However, following \citet{encinar2025proximal}, particle algorithms applicable to the \gls*{mmle} problem may be extended to non-differentiable $p_\theta(x,y)$. In particular, letting
\begin{equation*}
    E(\theta,x):=\frac{1}{2\sigma^2}\|y-Bx\|^2+e^\theta\mathrm{TV}(x)-d_x\theta
\end{equation*}
such that $\pi_\theta(x)\propto \exp(-E(\theta,x))$, we decompose $E(\theta,x)=g_1(\theta,x)+g_2(\theta,x)$ with $g_1,g_2$ convex lower-bounded functions, $g_1$ differentiable and $g_2$ proper and lower semi-continuous. We now replace $g_2$ with the $\lambda$-\gls*{my} approximation denoted by $g_2^\lambda$ and defined,
\begin{equation*}
\begin{aligned}
    g_2^\lambda(x)&:=\min_{x'\in\mathbb{R}^{d_x}}\left\{g_2(x')+\frac{1}{2\lambda}\|x'-x\|^2\right\}\\
    &=g_2\left(\mathrm{prox}^{\lambda}_{g_2}(x)\right)+\frac{1}{2\lambda}\left\|\mathrm{prox}^{\lambda}_{g_2}(x)-x\right\|^2,
\end{aligned}
\end{equation*}
where $\mathrm{prox}^\lambda_{g_2}(x)$ is the $\lambda$-proximity mapping or proximal operator function of $g_2$ for any $\lambda>0$. We denote the full \gls*{my} approximation with $E^\lambda(\theta,x)=g_1(\theta,x)+g_2^\lambda(\theta,x)$, which is always differentiable and converges pointwise to $E(\theta,x)$ in the limit $\lambda\rightarrow0$. The gradient is given for all $x\in\mathbb{R}^d$ by
\begin{equation*}
    \nabla_xE^\lambda(\theta,x)=\nabla_x g_1(x)+\frac{1}{\lambda}(x-\mathrm{prox}^\lambda_{g_2}(\theta,x)),
\end{equation*}
where we treat $\theta$ as fixed in the minimisation over $\mathbb{R}^{d_x}$. Since $E(\theta,x)$ is continuously differentiable with respect to $\theta$, we may take $\nabla_\theta E$ gradients directly. In our setting we have $g_1=\|y-Bx\|^2_2/2\sigma^2-d_x\theta$ and $g_2=e^\theta\mathrm{TV}(x)$. Now $g_1$ is differentiable, while $g_2$ is proper and lower semi-continuous, allowing us to write the explicit gradients,
\begin{align}
    \nabla_xE^\lambda(\theta,x)&=\frac{1}{\sigma^2}B^T(Bx-y)+\frac{1}{\lambda}\left(x-\mathrm{prox}^\lambda_{e^\theta\mathrm{TV}}(\theta,x)\right),\label{eq:prox_grad_x}\\
    \nabla_\theta E(\theta,x)&=e^\theta\mathrm{TV}(x)-d_x,\label{eq:prox_grad_theta}
\end{align}
with
\begin{equation*}
    \mathrm{prox}^\lambda_{e^\theta\mathrm{TV}}(\theta,x)=\underset{x'\in\mathbb{R}^{d_x}}{\mathrm{argmin}}\left\{e^\theta\mathrm{TV}(x')+\frac{1}{2\lambda}\|x'-x\|^2\right\}.
\end{equation*}
Using these approximations and noting that for \gls*{mmle} we have $F(\theta,x)=\nabla_\theta E(\theta,x)$, we may write the proximal versions of the \gls*{pgd} particle systems with,
\begin{equation*}
    \begin{aligned}
        \md \theta_t&=-\frac{1}{N}\sum_{i=1}^N\nabla_\theta E(\theta_t,X^i_t)\md t,\\
        \md X^i_t&=-\nabla_xE^\lambda(\theta_t,X^i_t)\md t+\sqrt{2}\md B_t^i,
    \end{aligned}
\end{equation*}
which we substitute with the gradient expressions \eqref{eq:prox_grad_x}--\eqref{eq:prox_grad_theta}. We discretise the system with a Lie-Trotter splitting of the parameter and sampling dynamics, running a forward Euler step for a parameter update followed by an Euler-Maruyama step for a sampling update. This is our \gls*{mypgd} baseline as introduced in \citet{encinar2025proximal}. Similarly, we obtain the proximal version of \gls*{myholpgd},
\begin{equation*}\label{eq:high_order_langevin_scheme_proximal}
    \begin{aligned}
        \md \theta_t&=-\frac{1}{N}\sum_{i=1}^N\nabla_\theta E(\theta_t,X^i_t)\md t,\\
        \md X^i_t&=U^i_t\md t,\\
        \md U^i_t&=-\nabla_xE^\lambda(\theta_t,X^i_t)\md t+\gamma V_t^i\md t,\\
        \md V^i_t&=-\gamma U^i_t\md t-\alpha V^i_t\md t+\sqrt{2\alpha}\md B^i_t,
    \end{aligned}
\end{equation*}
where $\alpha,\gamma>0$, and Brownian motions $\{B^i\}^N_{i=1}$ are independent. We discretise the higher-order Langevin parts using the splitting scheme introduced in \citet{high_order_langevin_JMLR} and discuss a full splitting scheme discretisation in the following section. We note that we apply the \gls*{my} approximation only to the sampling dynamics, while the parameter update uses the original gradient $\nabla_\theta E(\theta,x)$. This yields the hybrid dynamics used by \citet{encinar2025proximal}. For any finite $\lambda$, the parameter update is evaluated using particles approximating the smoothed law $\pi_\theta^\lambda\propto e^{-E^\lambda(\theta,\cdot)}$, and thus incurs a small smoothing bias.

\subsubsection{\textsc{\small HOL-MYPGD} Discretisation}

We consider a Lie-Trotter splitting between the parameter updates and the sampling updates like with the \gls*{mypgd} implementation. In particular, at each iteration we sequentially step parameters using a forward Euler scheme followed by a sampling step following the discretisation scheme in \citet{high_order_langevin_JMLR}. This scheme splits the higher-order Langevin dynamics into an exactly solvable linear evolution followed by a quadrature approximation of an integrated non-linear gradient term $\nabla_x E^\lambda$. Introducing the time discretisation $0=t_0<t_1<\cdots<t_K=T$ with $t_k=kh$ and step-size $h=T/K$ we denote by $(\hat X^i_n,\hat U^i_n,\hat V^i_n)$ the discrete time approximations of $(X^i_{t_n},U^i_{t_n},V^i_{t_n})$. Then we have the Gaussian update rule,
\begin{equation*}
    \begin{aligned}
        \hat\theta_{n+1}&=\hat\theta_n-\frac{h}{N}\sum_{i=1}^N\nabla_\theta E(\hat\theta_n,\hat X^i_n),\\
        \hat X^i_{n+1}&=\hat X^i_{n}-\frac{h}{2}\Delta E^\lambda\left(\hat\theta_{n+1},\hat{X}^i_n,\hat U^i_n\right)+\mu_{12}\hat U^i_n+\mu_{13}\hat V^i_n+\xi_n^{i,1}, \\
        \hat U^i_{n+1}&=-\Delta E^\lambda\left(\hat\theta_{n+1},\hat{X}^i_n,\hat U^i_n\right)+\mu_{22}\hat U^i_n+\mu_{23}\hat V^i_n+\xi^{i,2}_n,\\
        \hat V^i_{n+1}&=\mu_{31}\Delta E^\lambda\left(\hat\theta_{n+1},\hat{X}^i_n,\hat U^i_n\right)+\mu_{32}\hat U^i_n+\mu_{33}\hat V^i_n+\xi^{i,3}_n,
    \end{aligned}
\end{equation*}
with $\xi^i_n=(\xi^{i,1}_n,\xi^{i,2}_n,\xi^{i,3}_n)\sim\mathcal{N}(0,\Sigma)$. The vector $\Delta E^\lambda(\theta,x,u)$ is chosen to be an approximation of the integral $\int_0^h\nabla_xE^{\lambda}(\theta,x+ut)dt$. We note that in \citet{high_order_langevin_JMLR}, a Lagrange interpolation is used for a high order approximation with high-order mixing guarantees; however we choose to take the simple endpoint approximation, $\Delta E^\lambda(\theta,x,u)=h\nabla_xE^\lambda(\theta,x)$ for a fair comparison against \gls*{mypgd}. We refer the reader to Section~3 of \citet{high_order_langevin_JMLR} for a more detailed review and derivation of the scheme. The constants $\mu_{12}$--$\mu_{33}$ are defined explicitly as follows,

\begin{equation*}
    \begin{alignedat}{2}
        \mu_{12}&=\left(1+\tfrac{\gamma^2}{\alpha^2}\right)h-\tfrac{\gamma^2}{2\alpha}h^2-\tfrac{\gamma^2}{\alpha^3}\left(1-e^{-\alpha h}\right),\qquad&
    \mu_{13}&=\tfrac{\gamma}{\alpha}h-\tfrac{\gamma}{\alpha^2}\left(1-e^{-\alpha h}\right),\\
    \mu_{22}&=1+\tfrac{\gamma^2}{\alpha^2}\left(1-\alpha h - e^{-\alpha h}\right),&
    \mu_{23}&=\tfrac{\gamma}{\alpha}\left(1-e^{-\alpha h}\right),\\
    \mu_{31}&=\tfrac{\gamma}{\alpha}-\tfrac{\gamma}{\alpha^2h}\left(1-e^{-\alpha h}\right),&
    \mu_{32}&=\tfrac{\gamma^3}{\alpha^2}h+\tfrac{\gamma^3}{\alpha^2}he^{-\alpha h}-\left(\tfrac{2\gamma^3}{\alpha^3}+\tfrac{\gamma}{\alpha}\right)\left(1-e^{-\alpha h}\right),\\
    \mu_{33}&=e^{-\alpha h}\left(1+\tfrac{\gamma^2}{\alpha}h\right)-\tfrac{\gamma^2}{\alpha^2}\left(1-e^{-\alpha h}\right),
    \end{alignedat}
\end{equation*}
and $\Sigma=\begin{pmatrix}
    \sigma_{11}I_{d_x} & \sigma_{12}I_{d_x} & \sigma_{13}I_{d_x} \\
    \sigma_{12}I_{d_x} & \sigma_{22}I_{d_x} & \sigma_{23}I_{d_x} \\
    \sigma_{13}I_{d_x} & \sigma_{23}I_{d_x} & \sigma_{33}I_{d_x}
\end{pmatrix}$ with,

\begin{align*}
    \sigma_{11}&=\tfrac{2\gamma^2}{\alpha^3}h-\tfrac{2\gamma^2}{\alpha^2}h^2+\tfrac{2\gamma^2}{3\alpha}h^3-\tfrac{4\gamma^2}{\alpha^3}he^{-\alpha h}+\tfrac{\gamma^2}{\alpha^4}\left(1-e^{-2\alpha h}\right),\\
    \sigma_{12}&=\tfrac{\gamma^2}{\alpha^3}\left(\alpha h - 1 + e^{-\alpha h}\right)^2,\\
    \sigma_{13}&=-\tfrac{\gamma^3}{\alpha^2}h^2\left(2e^{-\alpha h}+1\right)+\left(\tfrac{2\gamma^3}{\alpha^3}-\tfrac{\gamma^3}{\alpha^3}e^{-2\alpha h}-\tfrac{4\gamma^3}{\alpha^3}e^{-\alpha h}-\tfrac{2\gamma}{\alpha}e^{-\alpha h}\right)h+\left(\tfrac{3\gamma^3}{2\alpha^4}+\tfrac{\gamma}{\alpha^2}\right)\left(1-e^{-2\alpha h}\right),\\
    \sigma_{22}&=\tfrac{2\gamma^2}{\alpha}h-\tfrac{4\gamma^2}{\alpha^2}\left(1-e^{-\alpha h}\right)+\tfrac{\gamma^2}{\alpha^2}\left(1-e^{-2\alpha h}\right),\\
    \sigma_{23}&=\tfrac{\gamma^3}{\alpha^2}\left(e^{-2\alpha h}-2e^{-\alpha h}-2\right)h+\tfrac{3\gamma^3}{2\alpha^3}\left(e^{-2\alpha h}-4e^{-\alpha h}+3\right)+\tfrac{\gamma}{\alpha}\left(1-e^{-\alpha h}\right)^2,\\
    \sigma_{33}&=-\tfrac{\gamma^4}{\alpha^2}h^2e^{-2\alpha h}+\left(-\tfrac{2\gamma^2}{\alpha}e^{-2\alpha h}+\tfrac{\gamma^4}{\alpha^3}\left(-3e^{-2\alpha h}+4e^{-\alpha h}+2\right)\right)h\\
    &\qquad +\tfrac{\gamma^4}{2\alpha^4}\left(-5e^{-2\alpha h}+16 e^{-\alpha h}-11\right)+\tfrac{\gamma^2}{\alpha^2}\left(-3e^{-2\alpha h}+4e^{-\alpha h}-1\right)+1-e^{-2\alpha h}.
\end{align*}

\subsubsection{Image Deblurring Experimental Details}

We present two distinct blurring setups A and B with different blur and noise severity. Setup A is a severe blurring with patch size $P=10$ and added noise $\sigma=1.0$, whereas setup B is a less severe blurring with patch size $P=8$ and $\sigma=0.8$. We apply the problem to a variety of black and white images of varying sizes, which we scale to approximately the same dimension $d_x$.

Across all setups we set $N=20$ particles with independent initial states
\begin{equation*}
    X^i_0\overset{i.i.d}{\sim}\mathcal{N}(50\mathbf{1}_{d_x},30^2I_{d_x}),\qquad i=1,\ldots,N,
\end{equation*}
where each component is drawn from a Gaussian distribution with mean $50$ and standard deviation $30$ and set $\theta_0=-15$. We set the Moreau-envelope parameter $\lambda=0.4$ and run $K=6000$ iterations with a shared step-size $h=0.04$, noting that in implementation we normalise $\theta$ per pixel by the dimension $d_x$ and use the following gradient,
\begin{equation*}
    \nabla_\theta\ell(\theta)
    =
    \mathbb{E}_{x\sim \pi_\theta}
    \left[\frac{e^\theta}{d_x}\mathrm{TV}(x)-1\right].
\end{equation*}
We evaluate the proximal map using a Douglas-Rachford $\mathrm{TV}$ method from the \texttt{proxTV} package in Python. All models were trained and evaluated on a high-performance computing cluster provided by the institution. Training was done on a NVIDIA RTX 6000 Ada Generation GPU with 18176 cores and 48GB of on-GPU memory. Images are taken from the Berkeley Segmentation and USC-SIPI image datasets.

We report the evolution of the $\theta$ estimate, \gls*{mse}, \gls*{ssim} between the averaged particle cloud and ground-truth images. To tune the hyper-parameters $\alpha,\gamma$ for the \gls*{myholpgd} method, we perform a grid search to achieve the best \gls{ssim} after $K$ iterations for setups A and B for an example test image. We report for setup A the best parameters $\alpha=4.0,\gamma=0.6$ and for setup B we take $\alpha=4.2,\gamma=0.7$. We note that the hyper-parameters are not tuned for each image, but rather for a single image in each setup. In addition, we observed little run-to-run variability when changing seeds which we attribute to the particle averaging. We also note that the graphs in Figures~\ref{fig:deblurring_example_graphs}, \ref{fig:setup_e_graphs1}, \ref{fig:setup_e_graphs2}, and \ref{fig:setup_e_graphs3} are evaluated on the averaged particle cloud, whilst Figures~\ref{fig:deblurring_example}, \ref{fig:setup_e_images1}, \ref{fig:setup_e_images2}, and \ref{fig:setup_e_images3} are produced by a single particle.

\begin{figure}[htbp]
    \centering
    \includegraphics[width=\textwidth]{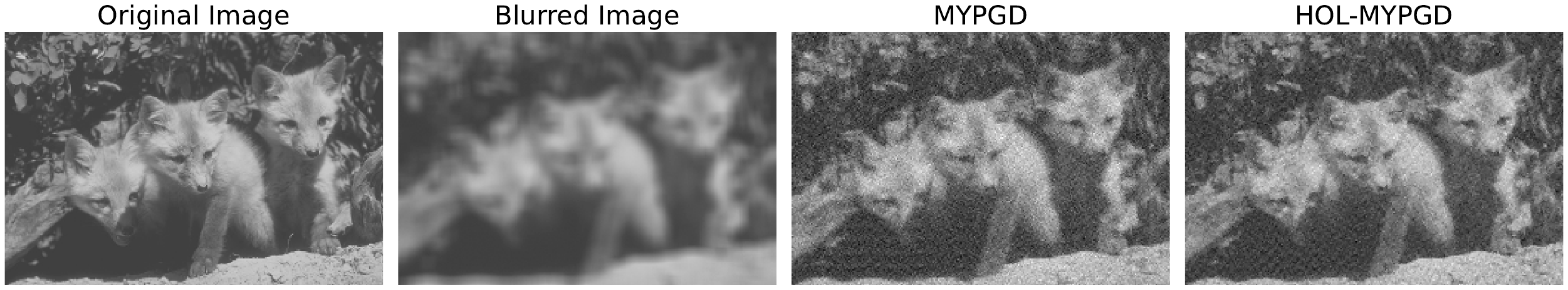}
    \caption{Setup B deblurring image comparison 1.}
    \label{fig:setup_e_images1}
\end{figure}

\begin{figure}[htbp]
    \centering
    \includegraphics[width=\textwidth]{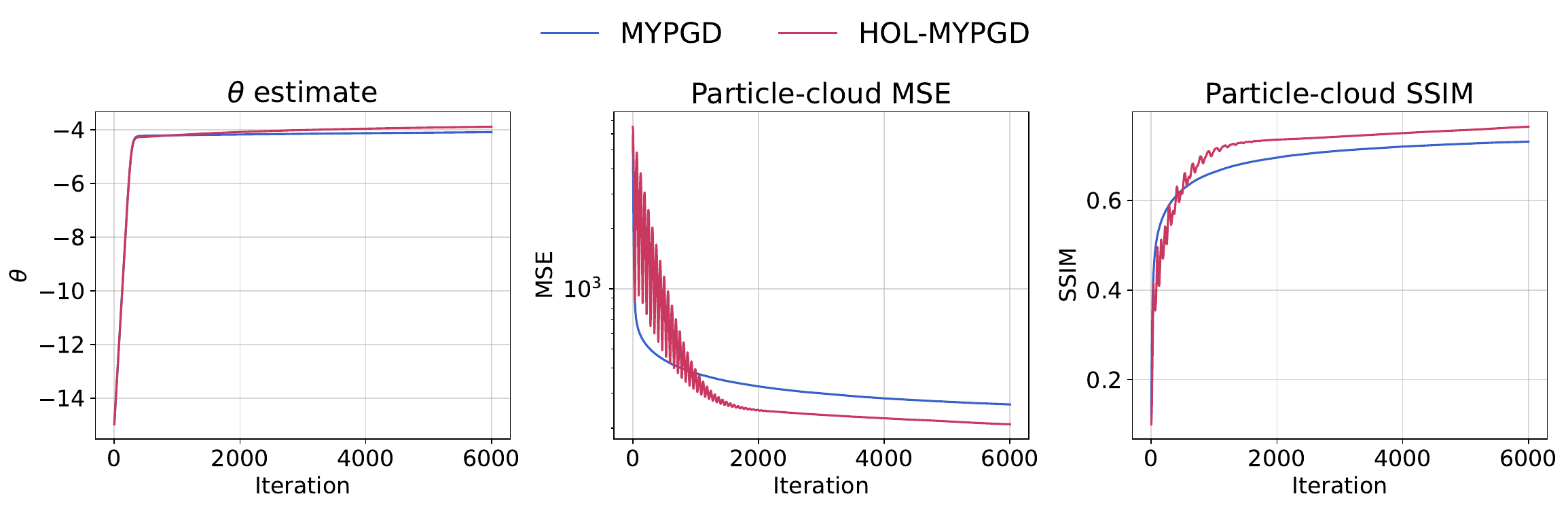}
    \caption{Setup B deblurring convergence comparison 1.}
    \label{fig:setup_e_graphs1}
\end{figure}

\begin{figure}[htbp]
    \centering
    \includegraphics[width=\textwidth]{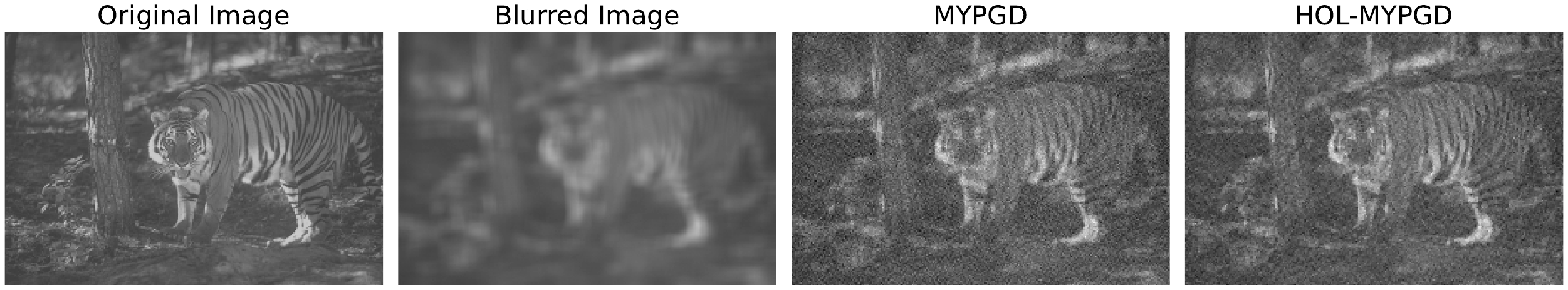}
    \caption{Setup B deblurring image comparison 2.}
    \label{fig:setup_e_images2}
\end{figure}

\begin{figure}[htbp]
    \centering
    \includegraphics[width=\textwidth]{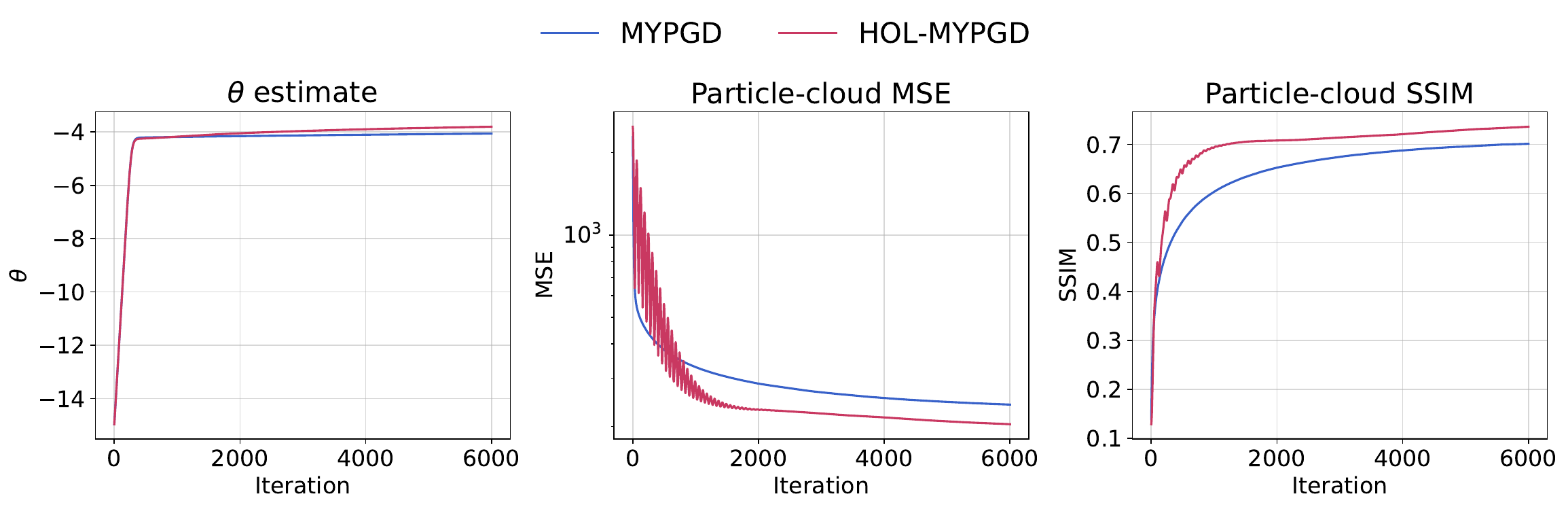}
    \caption{Setup B deblurring convergence comparison 2.}
    \label{fig:setup_e_graphs2}
\end{figure}

\begin{figure}[htbp]
    \centering
    \includegraphics[width=\textwidth]{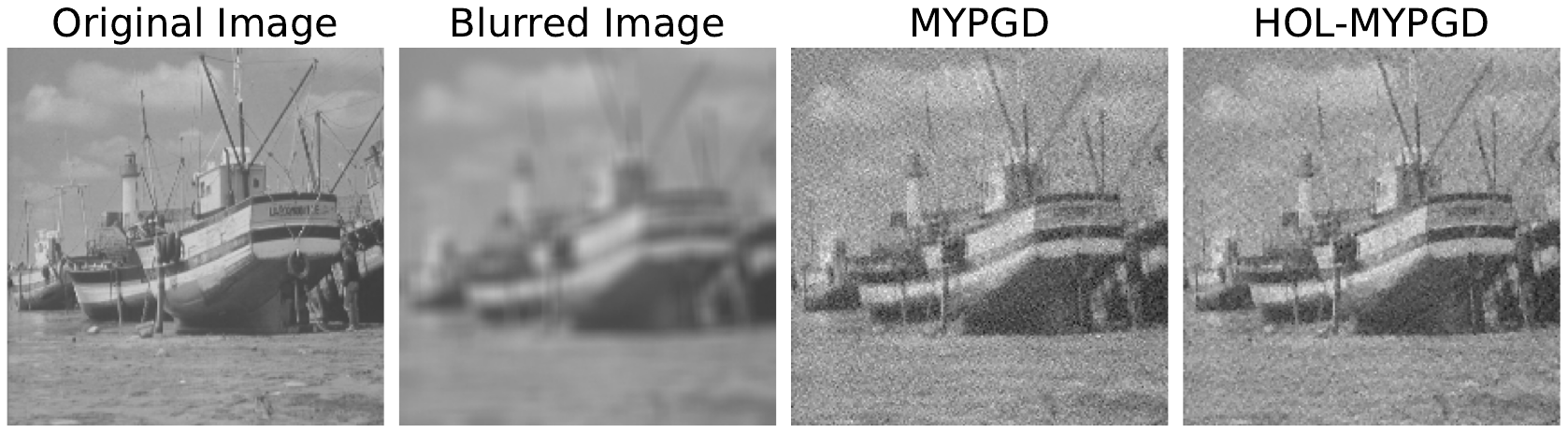}
    \caption{Setup B deblurring image comparison 3.}
    \label{fig:setup_e_images3}
\end{figure}

\begin{figure}[htbp]
    \centering
    \includegraphics[width=\textwidth]{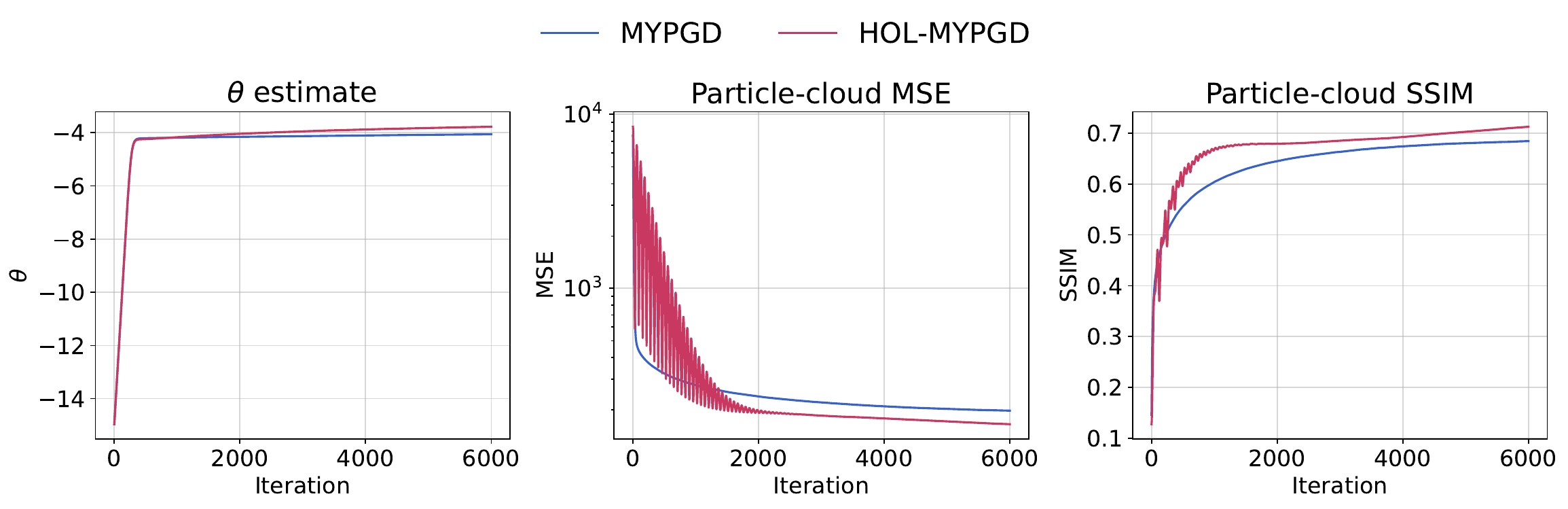}
    \caption{Setup B deblurring convergence comparison 3.}
    \label{fig:setup_e_graphs3}
\end{figure}

%% file: sections/appendix/momentum_mfpo_appendix.tex
Here, we discuss the \gls*{mmfpo} algorithm in Section~\ref{sec:momentum_mfpo}, motivate a discretisation  and present details on the \gls*{ebm} training experiment.

\subsubsection{\textsc{\small M-MFPO} Discretisation}

We begin by constructing a discretisation to provide a concrete algorithm for our framework. A naive Euler-Maruyama discretisation of \gls*{uld} can have poor stability and obscure the practical benefits of momentum; related limitations have been studied for underdamped samplers \citep{em_bad1_ma2019analognesterovaccelerationmcmc}. Indeed, preliminary experimental testing of the Euler-Maruyama discretisation yielded poor stability requiring small step-sizes and hence slow convergence. Instead we propose a splitting scheme inspired by the \textit{OBABO} scheme of \cite{obabo_ref_leimkuhler_matthews} for the discretisation of the underdamped Langevin dynamics. We consider the following splitting of the particle dynamics \eqref{eq:alg_underdamped} generator into,
\begin{equation*}
    \begin{aligned}
    \begin{pmatrix}
        d\theta_t \\ dm_t \\ dX^i_t \\ dV^i_t
    \end{pmatrix} &= \underbrace{\begin{pmatrix}
    m_tdt \\ 0 \\ 0 \\ 0
    \end{pmatrix}}_{\textit{A1}} +
    \underbrace{\begin{pmatrix}
    0 \\ -N^{-1}\sum_{i=1}^N F(\theta_t,X_t^i)dt \\ 0 \\ 0
    \end{pmatrix}}_{\textit{B1}} +
    \underbrace{\begin{pmatrix}
    0 \\ -\gamma  m_tdt \\ 0 \\ 0
    \end{pmatrix}}_{\textit{L}} \\
    &\qquad + \underbrace{\begin{pmatrix}
    0 \\ 0 \\ V^i_tdt \\ 0
    \end{pmatrix}}_{\textit{A2}} +
    \underbrace{\begin{pmatrix}
    0 \\ 0 \\ 0 \\ -\nabla_x E(\theta_t,X_t^i)dt
    \end{pmatrix}}_{\textit{B2}} +
    \underbrace{\begin{pmatrix}
    0 \\ 0 \\ 0 \\ -\gamma V^i_tdt + \sqrt{2\gamma}dB^i_t
    \end{pmatrix}}_{\textit{O}}.
\end{aligned}
\end{equation*}
Here, \textit{L} is linear and \textit{O} is an Ornstein-Uhlenbeck part, for which we know the solution exactly. The remaining parts \textit{A1}, \textit{A2}, \textit{B1} and \textit{B2} are deterministic and may be approximated with a forward Euler step.  The frictionless exact-gradient parameter dynamics described by
\begin{equation*}
    \md\theta_t=m_t \md t,\qquad
    \md m_t=-\nabla_\theta\ell(\theta_t)\md t
\end{equation*}
are Hamiltonian with
\begin{equation*}
    H_\theta(\theta,m)=\frac{1}{2}\Vert m\Vert^2+\ell(\theta),
\end{equation*}
In the mean-field and particle systems, with corresponding \textit{A1} and \textit{B1} parts, the objective gradient is approximated and is not generally $\nabla_\theta\ell(\theta)$ away from sampler equilibrium. For frozen $\theta$, the sampler dynamics, given by the \textit{A2} and \textit{B2} parts, is Hamiltonian for each $(X^i,V^i)$ and a full Hamiltonian is given by
\begin{equation*}
    H_{X,\theta}(X^1,...,X^N,V^1,...,V^N)=\frac{1}{2}\sum_{i=1}^N\Vert V^i\Vert^2+\sum_{i=1}^N E(\theta,X^i_t).
\end{equation*}
We note that the continuous time results in this paper do not establish any discretisation related guarantees, which we leave for future work. 

We now construct a Strang-splitting scheme by combining the exact solutions of the \textit{L} and \textit{O} parts with a forward Euler step for the remaining parts. We introduce the time discretisation $0=t_0<t_1<\cdots<t_K=T$ with $t_k=kh$ and step-size $h=T/K$. Defining $\eta=\exp\left({-\frac{\gamma h}{2}}\right)$ and $\bar{F}\left(\theta,X^{[N]}\right):=\frac{1}{N}\sum_{i=1}^N F(\theta,X^i)$ for clarity, we opt for the following order of updates,
\begin{equation}\label{eq:full_discretisation}
\begin{alignedat}{3}
    \tilde{V}^i_{n} &= \eta\hat{V}^i_n+\sqrt{1-\eta^2}\xi^{X^i,1}_n,
    &\quad \tilde{m}_n &= \eta\hat{m}_n,
    &\quad &(\textit{O,L}) \\
    \tilde{V}^i_{n+\frac{1}{2}} &= \tilde{V}^i_n-\frac{h}{2}\nabla_xE\left(\hat{\theta}_n,\hat{X}^i_n\right),
    &\quad \tilde{m}_{n+\frac{1}{2}} &= \tilde{m}_n-\frac{h}{2} \bar{F}\left(\hat{\theta}_n,\hat{X}^{[N]}_{n}\right),
    &\quad &(\textit{B}) \\
    \hat{X}^i_{n+1} &= \hat{X}^i_n+ h\tilde{V}^i_{n+\frac{1}{2}},
    &\quad \hat{\theta}_{n+1} &= \hat{\theta}_n+h\tilde{m}_{n+\frac{1}{2}},
    &\quad &(\textit{A}) \\
    \tilde{m}_{n+1} &= \tilde{m}_{n+\frac{1}{2}}-\frac{h}{2} \bar{F}\left(\hat{\theta}_{n+1},\hat{X}^{[N]}_{n+1}\right),
    &\quad \tilde{V}^i_{n+1} &= \tilde{V}^i_{n+\frac{1}{2}}-\frac{h}{2}\nabla_xE\left(\hat{\theta}_{n+1},\hat{X}^i_{n+1}\right),
    &\quad &(\textit{B}) \\
    \hat{m}_{n+1} &= \eta\tilde{m}_{n+1},
    &\quad \hat{V}^i_{n+1} &= \eta\tilde{V}^i_{n+1}+\sqrt{1-\eta^2}\xi^{X^i,2}_n,
    &\quad &(\textit{O,L})
\end{alignedat}
\end{equation}
where updates are read in order from left to right, top to bottom. We denote $\hat\theta_n, \hat{m}_n, \hat{X}^i_n, \hat{V}^i_n$ the approximations of $\theta_{t_n}, m_{t_n}, X^i_{t_n}, V^i_{t_n}$ respectively and $\xi^{X^i,1}_n, \xi^{X^i,2}_n$ are independent standard Gaussian random variables. The order of updates is not unique and we may swap the orders of any update within each row without changing the scheme. The linear \textit{L} flow is split symmetrically alongside the \textit{O} flow.

We note that \eqref{eq:full_discretisation} features alternating updates between the parameter and particles. Ignoring parameter updates, the sequential particle updates are exactly the OBABO scheme for the underdamped Langevin dynamics with frozen parameters; similarly, ignoring particle updates, the sequential parameter updates are exactly a Strang splitting scheme of the Heavy Ball dynamics with frozen particles. We note that discretisation may  be viewed this way in special regimes when $\theta$ is close to a fixed point of the optimiser and $X^i$ are close to equilibrium. However in general the coupling between the \textit{A,B} parameter and particle dynamics cannot be reduced to a single Hamiltonian system and separated discretisation steps are needed. We show that both the continuous-time system generated by the \textit{A} and \textit{B} parts and the discrete-time system generated by the \textit{B1,A,B2} parts of the discretisation have volume preserving flows, which motivates our choice of discretisation scheme.

\begin{lemma}[Volume preservation of schemes]\label{lemma:vol_preserve_ode} Consider the ODE obtained by taking the \textit{A, B} parts of \eqref{eq:alg_underdamped}. Assuming the associated flow is globally defined and the vector field in $C^1$ is continuously differentiable, then the flow is volume preserving. In particular, let $\varphi_t$ denote the flow generated by the ODE, then
\begin{equation*}
    \mathrm{det}D\varphi_t(z)=1,
\end{equation*}
where $D$ denotes the total derivative and for any $z:=(\theta,m,X,V)$ and $t>0$. 

Moreover, let $\Phi_h$ be the one-step map described by the \textit{B,A} sections of the discretisation \eqref{eq:full_discretisation}. Then $\Phi_h$ is also volume preserving. In particular,
\begin{equation*}
    \mathrm{det} D\Phi_h(z)=1,
\end{equation*}
for any $z:=(\theta,m,X,V)$ and $h$ the step-size.
\end{lemma}

\begin{proof}
Using Liouville's formula, it is a classical result that the ODE, $\dot{x}=f(x)$, is volume preserving if and only if $\mathrm{div}(f)\equiv0$. For our system \eqref{eq:alg_underdamped}, the \textit{A,B} sub-system is given by,
\begin{equation}\label{eq:general_splitting_scheme_ode}
\begin{aligned}
    \md\theta_t &=m_t\md t,\\
        \md m_t &= -\frac{1}{N}\sum_{i=1}^N{F}\left(\theta_t, X^i_t\right)\md t
        ,\\
        \md X^i_t&=V_t^i \md t,\\
        \md V^i_t&=-\nabla_x E(\theta_t, X^i_t)\md t.
\end{aligned}    
\end{equation}
From this it is clear the divergence of the vector-field is zero as each derivative is zero and the system is volume preserving. We now show the discretisation of the \textit{A,B} subsystem in \eqref{eq:full_discretisation} is also volume preserving. Given a one-step scheme, $\Phi_h:\mathbb{R}^{d_\theta+ d_m}\times(\mathbb{R}^ {d_x+ d_v})^N\rightarrow\mathbb{R}^{d_\theta+ d_m}\times(\mathbb{R}^ {d_x+ d_v})^N$ approximating the true flow $\varphi_h$ of \eqref{eq:general_splitting_scheme_ode} over a time interval of length $h$, assume that \(\Phi_h\) is a \(C^1\)-diffeomorphism. It is volume preserving if and only if
\begin{equation*}
    \left|\det D\Phi_h(\theta,m,X,V)\right|=1
\end{equation*}
for every \((\theta,m,X,V)\). For the splitting maps considered below,
the determinant is in fact equal to one. We show this by splitting the main flow into the composition of multiple sub-flows. From the chain rule, it is enough to check that each sub-flow is volume preserving. Taking $\Phi_h$ to be the one-step map associated with the \textit{B,A,B} section of \eqref{eq:full_discretisation}, we consider the composition,
\begin{equation*}
    \Phi_h=\Phi_h^{(1)}\circ \Phi_h^{(2)}\circ\Phi_h^{(1)},
\end{equation*}
with the one-step sub-maps,
\begin{equation*}
    \Phi_h^{(1)}\begin{pmatrix}\theta\\m\\X^i\\V^i
    \end{pmatrix}=\begin{pmatrix}
        \theta \\
        m-\frac{h}{2}\bar{F}\left(\theta, X^{[N]}\right) \\
        X^i \\
        V^i-\frac{h}{2}\nabla_xE(\theta, X^i)
    \end{pmatrix},\qquad \Phi_h^{(2)}\begin{pmatrix}\theta\\m\\X^i\\V^i
        \end{pmatrix}=\begin{pmatrix}
            \theta + hm \\
            m \\
            X^i + hV^i \\
            V^i
        \end{pmatrix}.
\end{equation*}
Taking derivatives we see that $D\Phi_h^{(1)}$ and $D\Phi_h^{(2)}$ are both unit triangular matrices (up to row and column exchanges) and hence, $\mathrm{det}D\Phi_h^{(1)}=\mathrm{det}D\Phi_h^{(2)}=1$ so each component of the composition is volume preserving and hence the \textit{A,B} subsystem discretisation is also volume preserving.
\end{proof}

With a typical \gls*{uld} in which the \textit{A,B} parts describe a Hamiltonian system, whose generated flow is symplectic and hence integrators are designed to also be symplectic. In our case, the \textit{A,B} parts no longer form a single Hamiltonian system, but the generated flow is still volume preserving and finding a discretisation scheme that is also volume preserving helps respect the underlying structure of the ODE.

Other orderings of parameter and particle updates are also valid. For instance, one could alternate between a full step of the parameters and a full step of the particles. The reason we decide the interweaving updates in \eqref{eq:full_discretisation} is one of computational efficiency. The computational bottleneck in these schemes is usually the expensive computations and evaluations of the gradients of the energy and objective functions, with the cost of such gradient evaluations also scaling with the size of the model. We observe in \eqref{eq:full_discretisation} and other second order schemes that gradient evaluations are required twice in one iteration: once before the \textit{A} step and once after. However this issue may be circumvented by saving the second gradient evaluation at the end of one iteration and re-using it for the first gradient evaluation at the start of the next iteration, essentially only requiring one gradient evaluation per iteration. Note that this is true provided the $\textit{A1,A2}$ updates are sequential between the two sets of $\textit{B1,B2}$ updates. This would not be possible if we were to naively update particles and parameters in series. We would also like to mention that our scheme admits other formulations such as using the BAOAB scheme with the same benefits as our proposed OBABO-like scheme and may also be written in the same way.

Whilst other practical integrators for \gls*{uld} are known and widely used, such as the (Euler) Exponential Integrator used in the \gls*{mpgd} discretisation \citep{lim2024momentumparticlemaximumlikelihood}, we found empirically that the OBABO-like splitting scheme is the most stable and performative whilst admitting larger step-sizes. The \textit{BAB} sequence helps preserves the geometry and structure of the continuous-time dynamics and widely used to discretise Hamiltonian systems, in which we express a resemblance to. In the Hamiltonian setting the \textit{B1A1B1} system with fixed particles or the \textit{B2A2B2} system with fixed parameters is also more commonly known as the leap-frog algorithm and a key component of \gls*{hmc}. Even at the level of the full coupled system, the second-order full-system dynamics admits a natural decomposition into interpretable subcomponents that can be updated separately, which motivates the use of a splitting scheme. In particular the splitting of the OBABO scheme preserves this operator structure, respecting the position-momentum decomposition present.

\subsubsection{Hyperparameter Tuning}  \label{sec:algorithm_hp_tuning}
The choices of friction \(\gamma\) and step size \(h\) are central to the stability and convergence of the algorithm; principled tuning remains open. Because the parameter and particle updates play different roles, an implementation may use separate values \((\gamma_\theta,h_\theta)\) and \((\gamma_X,h_X)\) for the two subsystems.

In \eqref{eq:full_discretisation}, friction enters the split flows through the momentum coefficient \(\gamma h\). For fixed \(h>0\), the deterministic interwoven leapfrog scheme is obtained in the limit \(\gamma\downarrow0\). Conversely, as \(\gamma h\to\infty\), one has \(\eta\to0\): the particle momenta are replaced by independent standard Gaussian draws, whereas the parameter momentum is set to zero because the \textit{L} flow has no noise. In this limit,
\[
\widehat X_{n+1}^i
=
\widehat X_n^i-\frac{h^2}{2}\nabla_xE(\widehat\theta_n,\widehat X_n^i)
+h\xi_n^{X^i,1},
\qquad
\widehat\theta_{n+1}
=
\widehat\theta_n-\frac{h^2}{2}\bar F(\widehat\theta_n,\widehat X_n^{[N]}).
\]
Thus the position and parameter updates reduce respectively to an Euler--Maruyama step for overdamped Langevin dynamics and a gradient-descent step, both with step size \(h^2/2\).

\subsubsection{2D \textsc{\small EBM} Training Experimental Details}\label{app:training_ebm_2d}
\subfile{ebm_training}

%% file: sections/appendix/ebm_training.tex
Here we outline the details regarding the energy based model training. Given an EBM of the form,
\begin{equation}
    p_\theta(x)=\frac{1}{C_\theta}e^{-E(\theta,x)},\qquad C_\theta=\int e^{-E(\theta,x)}dx,
\end{equation}
we consider the diffusion limit of the PCD algorithm given in \cite{oliva2025uniformintimeconvergenceboundspersistent} and adapt the system of equations for under-damped Langevin dynamics. After approximating the expectation $\mathbb{E}_{X\sim p_\theta}[\nabla_\theta E(\theta,X)]$ in the exact maximum likelihood gradient, we have the particle approximation,
\begin{equation*}
    \nabla_\theta\ell\approx-\frac{1}{N}\sum_{n=1}^N\nabla_\theta E(\theta,X^n)+\frac{1}{M}\sum_{m=1}^M\nabla_\theta E(\theta, y^m),
\end{equation*}
with particles $\{X^i\}_{i=1}^N$ and fixed data $\{y^m\}_{m=1}^M$. In our framework, this may be written with
\begin{equation*}
    F(\theta,X^{i})=-\nabla_\theta E(\theta,X^i)+\frac{1}{M}\sum_{m=1}^M\nabla_\theta E(\theta, y^m).
\end{equation*}
We parametrise the energy function used for training in the 2D synthetic datasets experiments with a 3 hidden-layer, 128-unit \textit{SiLU} network. Throughout training we assume $C_\theta<\infty$ such that the model is a valid density for every parameter value we visit. We pick three different synthetic datasets to train on. Each has closed form density functions and are described below,

\textbf{Rings}. We consider two concentric rings with different radii but sharing the same total densities along their circumferences. In polar coordinates we define the independent distributions,
\begin{equation*}
    p_r(r)=\frac{1}{2}\mathcal{N}_{\mathrm{tr}}(r;r_1,\sigma^2,[0,\infty))+\frac{1}{2}\mathcal{N}_{\mathrm{tr}}(r;r_2,\sigma^2,[0,\infty)),\qquad r\ge 0.
\end{equation*}
where $\mathcal{N}_{\mathrm{tr}}(\cdot\,;\mu,\sigma,[0,\infty))$ is a truncated Gaussian with support on the non-negative reals $[0,\infty)$. Independently draw $A\sim\textrm{Uniform}(0,2\pi)$ and set $\mathbf{x}=r(\cos A,\sin A)$. The \textit{rings} distribution is exactly the Cartesian parametrisation,
\begin{equation*}
    p(\mathbf{x})=\frac{1}{4\pi\Vert\mathbf{x}\Vert}\left(\mathcal{N}_{\mathrm{tr}}(\|x\|;r_1,\sigma^2,[0,\infty))+\mathcal{N}_{\mathrm{tr}}(\|x\|;r_2,\sigma^2,[0,\infty))\right).
\end{equation*}
We note for large enough $r_1,r_2$ and small enough $\sigma$, the truncation effects are negligible and we may simply write the distributions with regular Gaussians, giving us closed-form expressions for the score. Here, we take $r_1=1,r_2=2.5$ and $\sigma=0.2$.

\textbf{Beads}. We consider a mixture of isotropic Gaussians centred equiradially from the origin. For $n$ Gaussians we define the density,
\begin{equation*}
    p(\mathbf{x})=\frac{1}{n}\sum_{k=1}^n\mathcal{N}(\mathbf{x}; l_k,\sigma^2)
\end{equation*}
with $l_k=(\sin(2\pi k/n),\cos(2\pi k/n))$. Here we take $n=5$ and $\sigma=0.1$.

\textbf{Lattice}. We consider a mixture of isotropic Gaussians centred on the points of a $n\times n$ lattice. Denoting these points $m_k=(\lfloor k/n\rfloor-(n-1)/2,k-n\lfloor k/n\rfloor-(n-1)/2)$ for $k=0,...,n^2-1$ we have the density,
\begin{equation*}
    p(\mathbf{x})=\frac{1}{n^2}\sum_{k=0}^{n^2-1}\mathcal{N}(\mathbf{x};\alpha m_k,\sigma^2),
\end{equation*}
where $\alpha$ is some scaling parameter. We take $n=3$ and $\alpha=\sqrt{2}$ and 
$\sigma=\sqrt{2}/10$. These distributions can be sampled from exactly with built-in functions to create training datasets. We generate 10,000 samples for each dataset.

To train the 2D synthetic datasets, we compare the different methods below,
\begin{itemize}
    \item \gls*{pcd} - We implement persistent contrastive divergence as a baseline. Particles are updated with a ULA kernel step, whilst parameters are updated with gradient descent.
    \item \gls*{mmfpo} - We accelerate by updating particles and parameters with under-damped Langevin using our proposed discretisation.
\end{itemize}
For each method we use a set of $10,000$ particles, matching the number of data-points, all initialised from i.i.d standard Gaussians. For \gls*{mmfpo}, we initialise particle momenta from i.i.d standard Gaussians and initialise parameter momenta with zeros.

For the performance metric, we estimate the Fisher divergence as it does not require computing the normalising constant of the learnt energy and we have access to the score of the data distribution. We compute the empirical Fisher divergence at each iteration with $M_{\mathrm{eval}}=10,000$ samples from the data distribution,
\begin{equation*}
    \widehat D_{\mathrm F}(\theta)
    =\frac1{2M_{\mathrm{eval}}}
    \sum_{m=1}^{M_{\mathrm{eval}}}
    \left\|
    \nabla_x\log p_{\mathrm{data}}(z^m)
    +\nabla_xE(\theta,z^m)
    \right\|^2,
    \qquad
    z^m\sim p_{\mathrm{data}}.
\end{equation*}
If the model family contains $p_{\mathrm{data}}$ and the usual score-identifiability and integrability conditions hold, the Fisher divergence shares the same zero at $p_\theta=p_{\mathrm{data}}$ with $\mathrm{KL}(p_{\mathrm{data}}\|p_\theta)$. Note that for the \textit{rings} dataset, whilst the score diverges at zero, the empirical calculation evaluates scores at samples from the data distribution which are sufficiently far from zero with high probability. We do not claim that their minimisers coincide under model misspecification or that the minimising neural-network parameter is unique.

We keep a fixed computational budget for each method, which for us is equivalent to a fixed number of training iterations since each iteration uses the same number of energy evaluations and gradient evaluations. We begin training with a shared step-size between parameter and particle updates for both \gls*{pcd} and \gls*{mmfpo}. In practice the oscillatory nature of the underdamped dynamics in \gls*{mmfpo} result in greater training fluctuations. So we also design learning rate decay schedules on the parameter step-sizes that decrease over iterations. In particular for the \textit{lattice} dataset, we use the linear decay,

\begin{equation}
    \text{lr}(t)=\begin{cases}
        h_0 & t\le L, \\
        h_0\left(\frac{T-t}{T-L}\right)+h_T\left(\frac{t-L}{T-L}\right) & L<t<T, \\
        h_T & t\ge T,
    \end{cases}
\end{equation}
where $h_0,h_T$ are hyperparameters starting and ending step-sizes, $L,T$ are times we pick to decay between. In general these are all hyperparameters but with the goal of damping the training fluctuations, extensive tuning is not required. 
For the \textit{lattice} dataset, we take $T=400,000$ to be the fixed number of training iterations, $L=6000$ and $h_T=h_0/10$. For the \textit{rings} and \textit{beads} datasets we do not use any scheduler and run for $T=150,000$ iterations. We run the training 25 times on different seeds and report the average Fisher divergence with standard-deviation error envelopes across iterations. All models were trained and evaluated on a high-performance computing cluster provided by the institution. Training was done on a NVIDIA RTX 6000 Ada Generation GPU with 18176 cores and 48GB of on-GPU memory. We plot the learnt probability densities and freshly generated samples in Figure~\ref{fig:models}.

\begin{figure*}[ht]
\begin{minipage}[t]{0.32\textwidth}
        \centering
        \includegraphics[width=\textwidth]{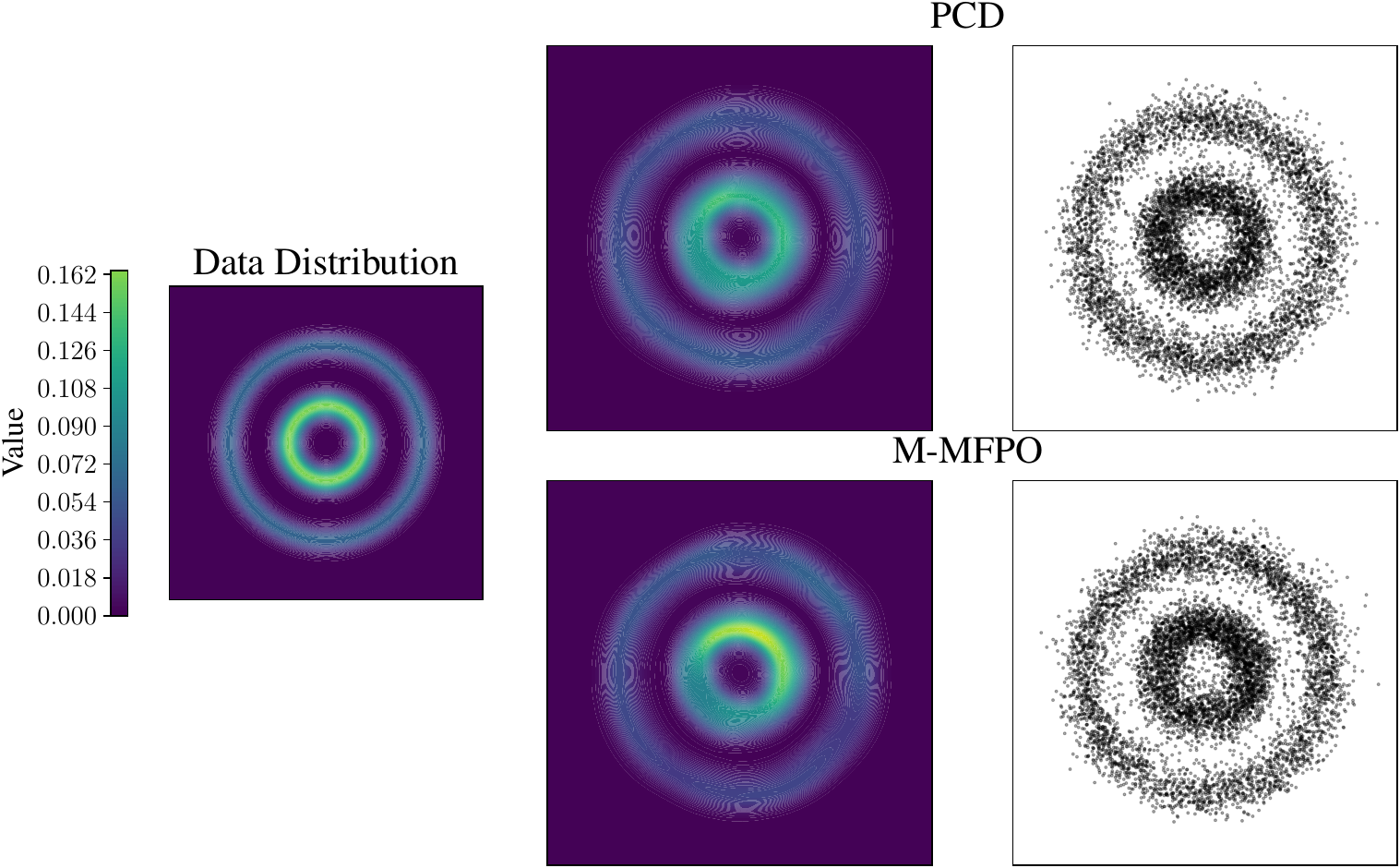}
    \end{minipage}
    \hfill
    \begin{minipage}[t]{0.32\textwidth}
        \centering
        \includegraphics[width=\textwidth]{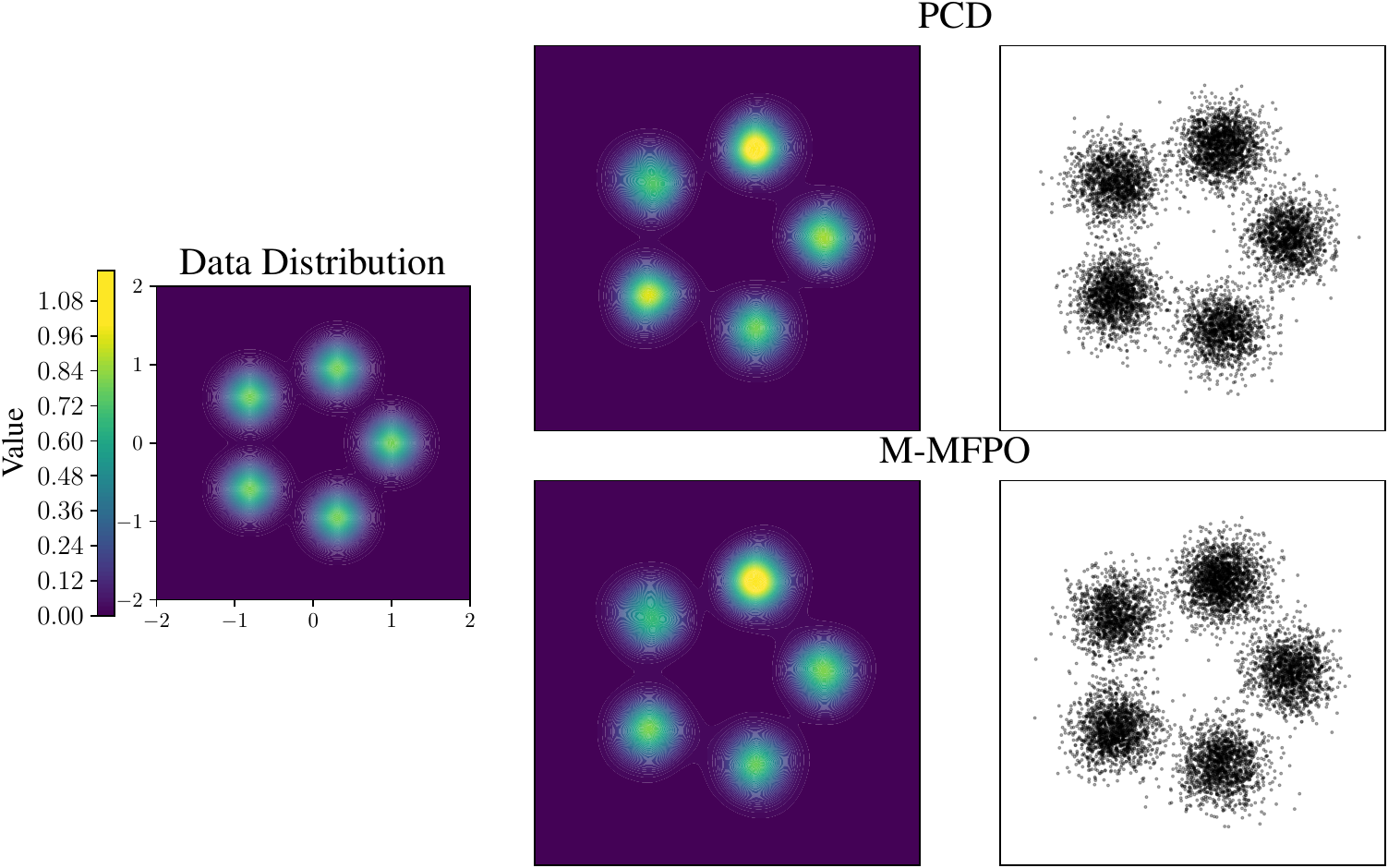}
    \end{minipage}
    \hfill
    \begin{minipage}[t]{0.32\textwidth}
        \centering
        \includegraphics[width=\textwidth]{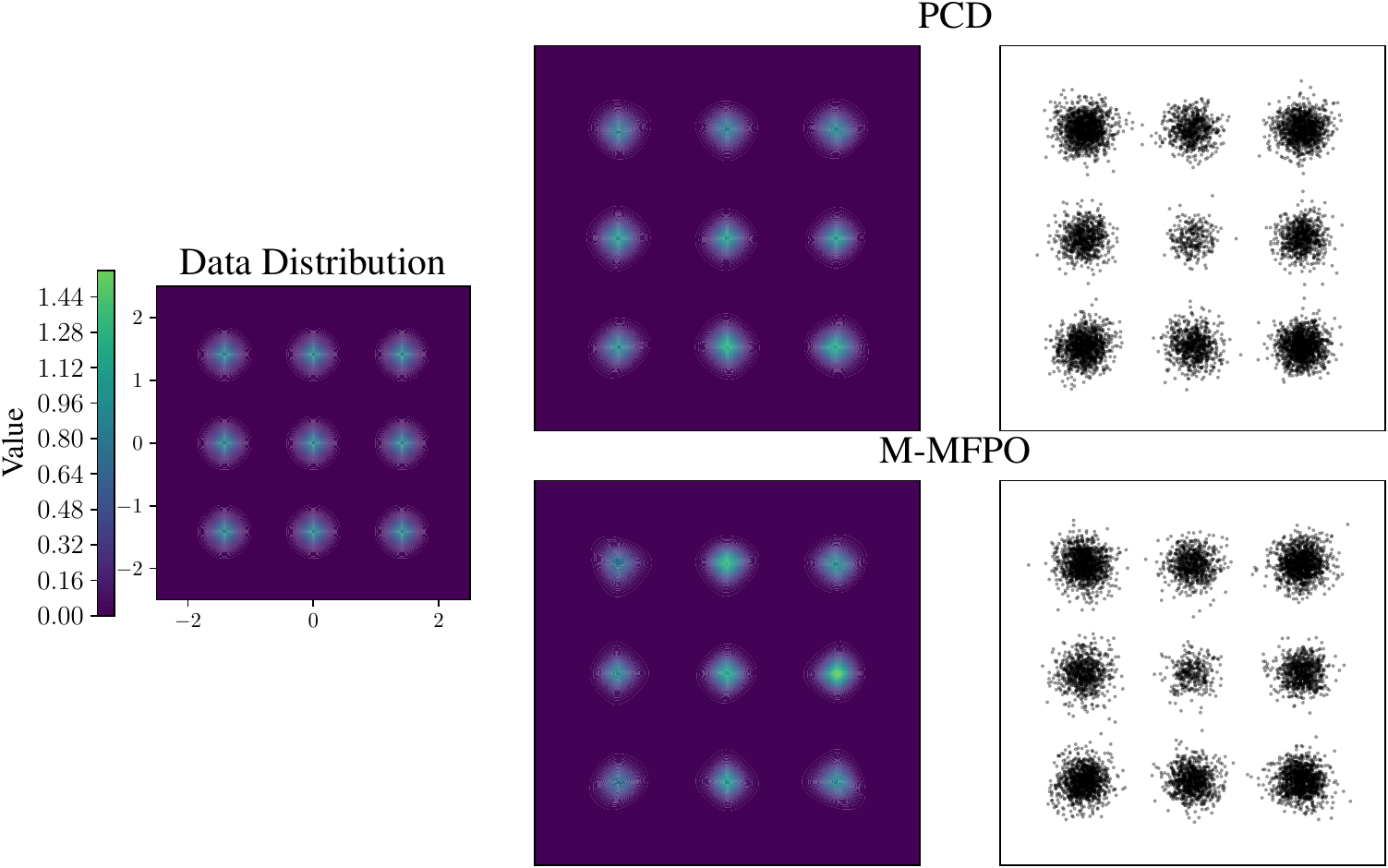}
    \end{minipage}
    \caption{Comparison of learnt densities and fresh samples from the model for the \textit{rings} dataset (left), \textit{beads} dataset (middle), and \textit{lattice} dataset (right).}
    \label{fig:models}
\end{figure*}

For the baseline \gls*{pcd} method, we only have step-size as a hyperparameter, whilst for \gls*{mmfpo} we have both step-size and friction. When tuning, there is a trade-off between speed and accuracy, with smaller step-sizes typically yielding a lower Fisher divergence at the cost of slower convergence. We plot in Figures \ref{fig:od_beads_hps} and \ref{fig:obabo_beads_hps} the effects of hyperparameters for training the \textit{beads} dataset on both the speed of convergence, which we measure with the area under the curve ($AOC$) and the accuracy at the end of training, which we measure by an averaged final Fisher divergence ($FD_\infty$). We record the Fisher divergence every 100 iterations over 100,000 iterations and evaluate the $AOC$ as the sum and $FD_\infty$ as the average of the last 20 recordings. We seek the hyperparameters producing both the lowest $AOC$ and $FD_\infty$.
\begin{figure}[h]
    \centering
    \begin{minipage}[b]{0.48\textwidth}
        \centering
        \includegraphics[width=0.75\textwidth]{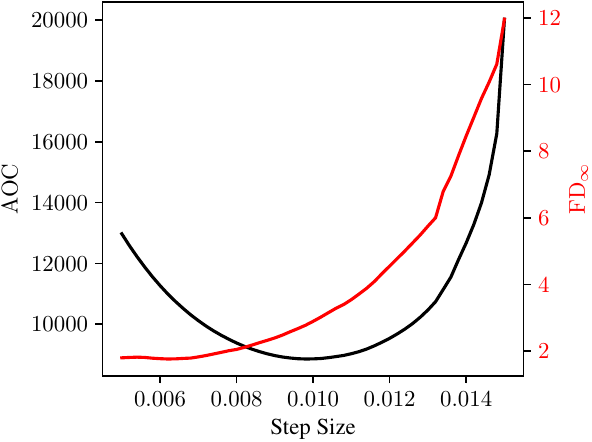}
    \end{minipage}
    \hfill
    \begin{minipage}[b]{0.48\textwidth}
        \centering
        \includegraphics[width=0.75\textwidth]{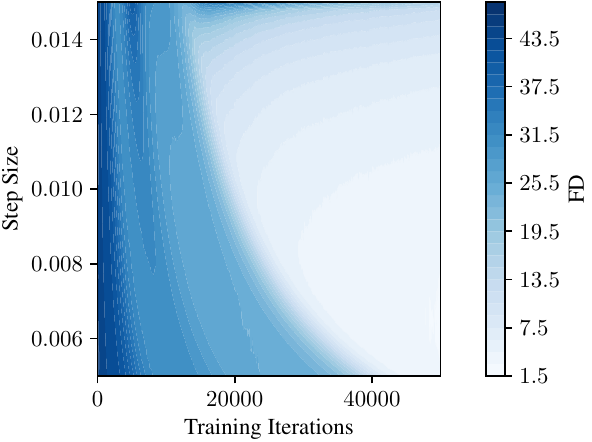}
    \end{minipage}
    \caption{Effect of step-size on convergence ($AOC$) and converged-value ($FD_\infty$) for PCD (left) and contour plot of Fisher divergence across iterations for a range of step-sizes (right)}
    \label{fig:od_beads_hps}
\end{figure}

We see in Figure~\ref{fig:od_beads_hps} increasing step-size leads to a worse convergence and decreasing step-size leads to much slower convergence. There is a minor trade off between the $AOC$ and $FD_\infty$ if both are to be minimised and we take the step-size $h=0.01$ in the final comparison.
\begin{figure}[htbp]
    \centering
    \begin{minipage}[b]{0.48\textwidth}
        \centering
        \includegraphics[width=0.75\textwidth]{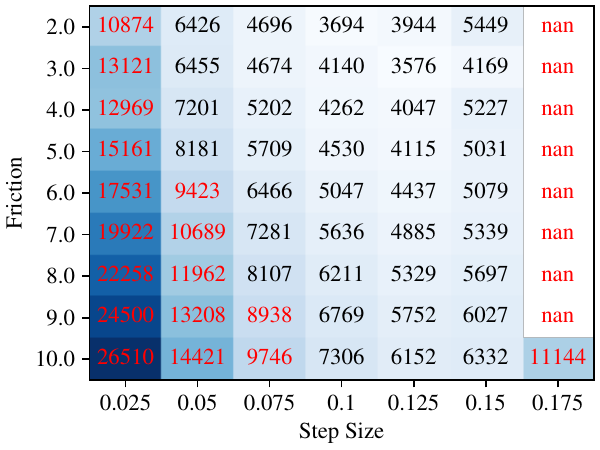}
    \end{minipage}
    \hfill
    \begin{minipage}[b]{0.48\textwidth}
        \centering
        \includegraphics[width=0.75\textwidth]{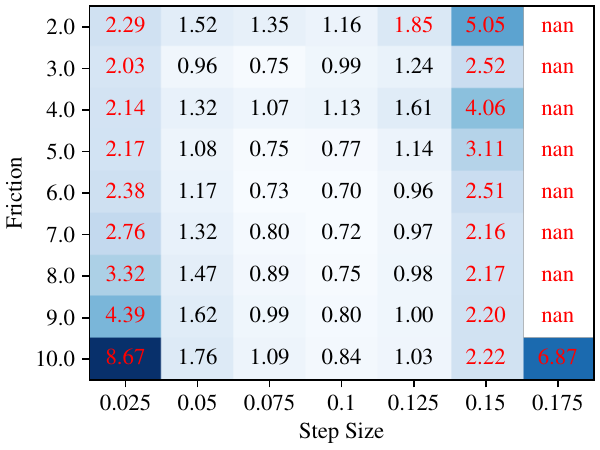}
    \end{minipage}
    \caption{Effect of step-size and friction on the $AOC$ (left) and $FD_\infty$ (right). Values in black are lower than the lowest observed corresponding value in the overdamped PCD across all step-sizes (lower is better).}
    \label{fig:obabo_beads_hps}
\end{figure}
We pick optimal hyperparameters for both methods from the sweeps whilst also ensuring long-term stability of the training. These are summarised in table~\ref{tab:ebm_training_hps}. Note that from Figure~\ref{fig:obabo_beads_hps}, values in black beat the best \gls*{pcd} hyperparameters and we see there are particular choices of friction and step-size that yield an $AOC$ and $FD_\infty$ that are both better than the best corresponding \gls*{pcd} hyperparameter.
\begin{table}[h]
    \caption{List of hyperparameters used in the 2D \gls*{ebm} training experiments for \gls*{pcd} and \gls*{mmfpo} methods across different datasets. Note that we use a linear step-size schedule for the \textit{lattice} dataset and the displayed step-size is the initial step-size.}
    \label{tab:ebm_training_hps}
    \centering
    \setlength{\tabcolsep}{2pt}
    {\begin{tabular}{l@{\hspace{24pt}}c@{\hspace{24pt}}cc}
        \toprule
        \multirow{2}{*}{Dataset} & \multicolumn{1}{c}{\gls*{pcd}} & \multicolumn{2}{c}{\gls*{mmfpo}} \\
        \cmidrule(lr){2-2}\cmidrule(lr){3-4}
        & \makecell{(initial) step-size} & \makecell{step-size} & \makecell{friction} \\
        \midrule
        \textbf{rings} & 
        0.016 & 
        0.19 & 
        5.5 \\
        \textbf{beads} & 
        0.01 & 
        0.1 & 
        6.0 \\
        \textbf{lattice} & 
        0.004 & 
        0.075 & 
        8.0 \\
        \bottomrule
    \end{tabular}}
\end{table}